\documentclass{article}

\usepackage{microtype}
\usepackage{graphicx}
\usepackage{subcaption}

\usepackage{booktabs} 
\usepackage{multirow}
\usepackage[flushleft]{threeparttable}
\usepackage{footnote}

\usepackage{amsmath}
\usepackage{amsthm}
\usepackage{commath}
\usepackage{amssymb}
\usepackage{textcomp}
\usepackage{gensymb}
\usepackage{algorithm}
\usepackage{graphicx}
\usepackage{enumitem}

\newcommand{\M}{\mathcal{M}}
\newcommand{\R}{\mathbb{R}}
\newcommand{\vect}[1]{{\overset{\rightharpoonup}{#1}}}
\newcommand{\bvect}[1]{\boldsymbol{#1}}

\newtheorem{thm}{Theorem}[section]
\newtheorem{prop}[thm]{Proposition}
\newtheorem{cor}[thm]{Corollary}
\newtheorem*{remark}{Remark}
\newtheorem{lemma}[thm]{Lemma}

\DeclareMathOperator*{\argmax}{arg\,max}
\DeclareMathOperator*{\argmin}{arg\,min}
\DeclareMathOperator{\arccosh}{arccosh}
\DeclareMathOperator{\rank}{rank}
\DeclareMathOperator*{\gyr}{gyr}
\DeclareMathOperator{\proj}{proj}
\newcommand{\Ll}{\mathcal{L}}

\newcommand{\inn}[1]{\left\langle#1\right\rangle}
\newcommand{\paren}[1]{\left(#1\right)}
\newcommand{\brac}[1]{\left\{#1\right\}}

\newcommand{\grad}{\nabla}
\newcommand{\jacob}{\widetilde{\nabla}}
\newcommand{\parderiv}[2]{\frac{\partial #1}{\partial #2}}
\newcommand{\snorm}[1]{\lVert#1\rVert}

\usepackage{hyperref}

\usepackage[accepted]{icml2020} 

\icmltitlerunning{Differentiating through the Fréchet Mean}

\begin{document}

\twocolumn[
\icmltitle{Differentiating through the Fr\'echet Mean}



\icmlsetsymbol{equal}{*}

\begin{icmlauthorlist}
\icmlauthor{Aaron Lou}{equal,co}
\icmlauthor{Isay Katsman}{equal,co}
\icmlauthor{Qingxuan Jiang}{equal,co}
\icmlauthor{Serge Belongie}{co}
\icmlauthor{Ser-Nam Lim}{fb}
\icmlauthor{Christopher De Sa}{co}
\end{icmlauthorlist}

\icmlaffiliation{co}{Department of Computer Science, Cornell University, NY, Ithaca, USA}
\icmlaffiliation{fb}{Facebook AI, NY, New York, USA}

\icmlcorrespondingauthor{Aaron Lou}{al968@cornell.edu}
\icmlcorrespondingauthor{Isay Katsman}{isk22@cornell.edu}
\icmlcorrespondingauthor{Qingxuan Jiang}{qj46@cornell.edu}

\icmlkeywords{hyperbolic neural networks, batch norm}

\vskip 0.3in
]



\printAffiliationsAndNotice{\icmlEqualContribution} 

\begin{abstract}
Recent advances in deep representation learning on Riemannian manifolds extend classical deep learning operations to better capture the geometry of the manifold. One possible extension is the Fr\'echet mean, the generalization of the Euclidean mean; however, it has been difficult to apply because it lacks a closed form with an easily computable derivative. In this paper, we show how to differentiate through the Fr\'echet mean for arbitrary Riemannian manifolds. Then, focusing on hyperbolic space, we derive explicit gradient expressions and a fast, accurate, and hyperparameter-free Fr\'echet mean solver. This fully integrates the Fr\'echet mean into the hyperbolic neural network pipeline. To demonstrate this integration, we present two case studies. First, we apply our Fr\'echet mean to the existing Hyperbolic Graph Convolutional Network, replacing its projected aggregation to obtain state-of-the-art results on datasets with high hyperbolicity. Second, to demonstrate the Fr\'echet mean's capacity to generalize Euclidean neural network operations, we develop a hyperbolic batch normalization method that gives an improvement parallel to the one observed in the Euclidean setting\footnote{Our PyTorch implementation of the differentiable Fr\'echet mean can be found at \href{https://github.com/CUAI/Differentiable-Frechet-Mean}{https://github.com/CUAI/Differentiable-Frechet-Mean}.}.
\end{abstract}

\section{Introduction}
\label{intro}

\begin{figure}[htb!]
    \centering
    \includegraphics[scale=0.5]{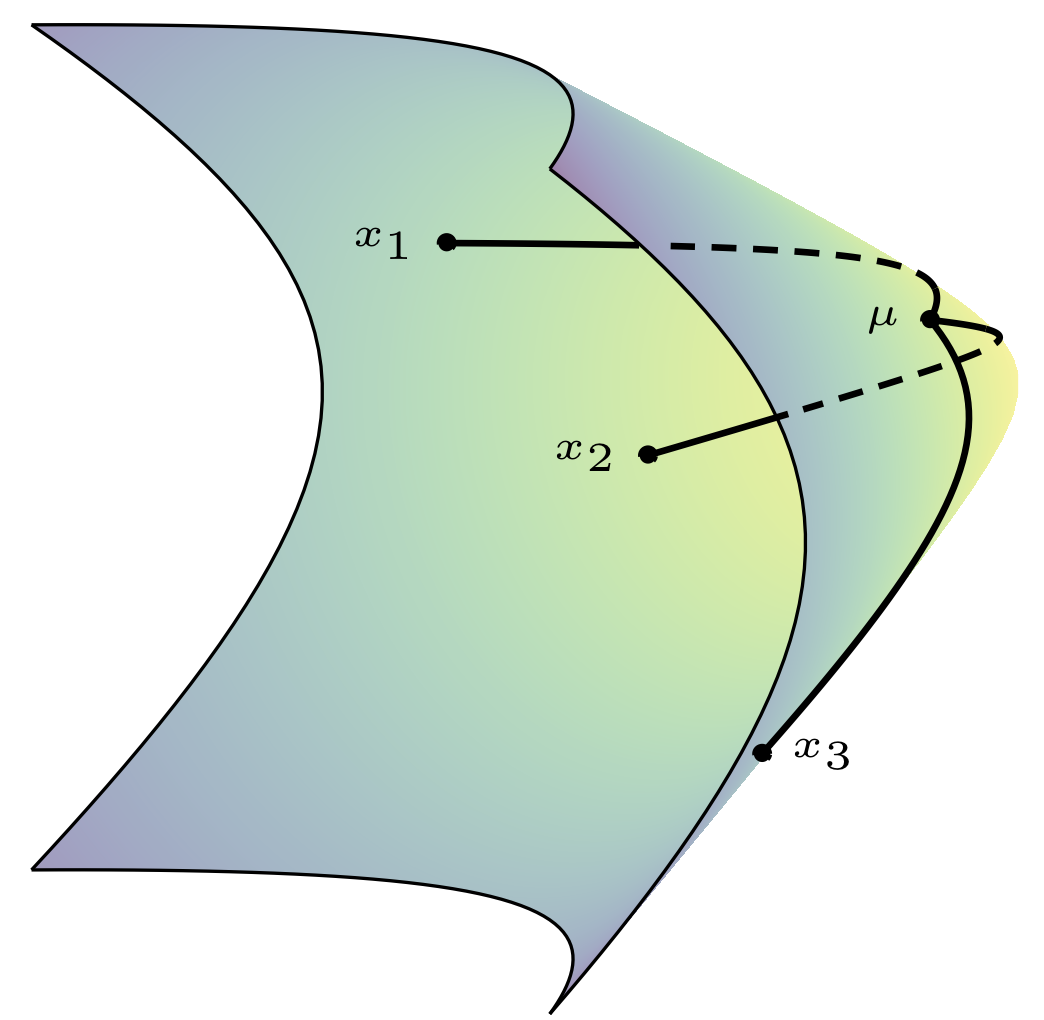}
    \caption{\label{fig:introfigure} Depicted above is the Fr\'echet mean, $\mu$, of three points, $x_1, x_2, x_3$ in the Lorentz model of hyperbolic space. As one can see, the Fr\'echet mean conforms with the geometry of the hyperboloid and is vastly different from the standard Euclidean mean.}
\end{figure}

Recent advancements in geometric representation learning have utilized hyperbolic space for tree embedding tasks \cite{Nickel2017PoincarEF, Nickel2018LearningCH, Yu2019NumericallyAH}. This is due to the natural non-Euclidean structure of hyperbolic space, in which distances grow exponentially as one moves away from the origin. Such a geometry is naturally equipped to embed trees, since if we embed the root of the tree near the origin and layers at successive radii, the geometry of hyperbolic space admits a natural hierarchical structure. More recent work has focused specifically on developing neural networks that exploit the structure of hyperbolic space \cite{Ganea2018HyperbolicNN,  Tifrea2018PoincarGH, Chami2019HyperbolicGC, liu2019hyperbolic}.



A useful structure that has thus far not been generalized to non-Euclidean neural networks is that of the Euclidean mean. The (trivially differentiable) Euclidean mean is necessary to perform aggregation operations such as attention \cite{Vaswani2017AttentionIA}, and stability-enhancing operations such as batch normalization \cite{Ioffe2015BatchNA}, in the context of Euclidean neural networks. The Euclidean mean extends naturally to the Fr\'echet mean in non-Euclidean geometries \cite{frechet1948elements}. However, unlike the Euclidean mean, the Fr\'echet mean does not have a closed-form solution, and its computation involves an argmin operation that cannot be easily differentiated. This makes the important operations we are able to perform in Euclidean space hard to generalize to their non-Euclidean counterparts. In this paper, we extend the methods in \citet{Gould2016OnDP} to differentiate through the Fr\'echet mean, and we apply our methods to downstream tasks. Concretely, our paper's contributions are that:



\begin{itemize}
    \item We derive closed-form gradient expressions for the Fr\'echet mean on Riemannian manifolds.

    \item For the case of hyperbolic space, we present a novel algorithm for quickly computing the Fr\'echet mean and a closed-form expression for its derivative.

    \item We use our Fr\'echet mean computation in place of the neighborhood aggregation step in Hyperbolic Graph Convolution Networks \cite{Chami2019HyperbolicGC} and achieve state-of-the-art results on graph datasets with high hyperbolicity.

    \item We introduce a fully differentiable Riemannian batch normalization method which mimics the procedure and benefit of standard Euclidean batch normalization.
\end{itemize}

\section{Related Work} 

\textbf{Uses of Hyperbolic Space in Machine Learning.} The usage of hyperbolic embeddings first appeared in \citet{Kleinberg2007GeographicRU}, in which the author uses them in a greedy embedding algorithm. Later analyses by \citet{SarkarLowDistortionDelauneyEmbeddings} and \citet{Sala2018RepresentationTF} demonstrate the empirical and theoretical improvement of this approach. However, only recently in \citet{Nickel2017PoincarEF, Nickel2018LearningCH} was this method extended to machine learning. Since then, models such as those in \citet{Ganea2018HyperbolicNN, Chami2019HyperbolicGC, liu2019hyperbolic} have leveraged hyperbolic space operations to obtain better embeddings using a hyperbolic version of deep neural networks.

\textbf{Fr\'echet Mean.} The Fr\'echet mean \cite{frechet1948elements}, as the generalization of the classical Euclidean mean, offers a plethora of applications in downstream tasks. As a mathematical construct, the Fr\'echet mean has been thoroughly studied in texts such as \citet{karcher1977riemannian, charlierbenjfrechet, bacak2014computing}.

However, the Fr\'echet mean is an operation not without complications; the general formulation requires an argmin operation and offers no closed-form solution. As a result, both computation and differentiation are problematic, although previous works have attempted to resolve such difficulties.

To address computation, \citet{gu2018learning} show that a Riemannian gradient descent algorithm recovers the Fr\'echet mean in linear time for products of Riemannian model spaces. However, without a tuned learning rate, it is too hard to ensure performance. \citet{Brooks2019RiemannianBN} instead use the Karcher Flow Algorithm \cite{karcher1977riemannian}; although this method is manifold-agnostic, it is slow in practice. We address such existing issues in the case of hyperbolic space by providing a fast, hyperparameter-free algorithm for computing the Fr\'echet mean.

Some works have addressed the differentiation issue by circumventing it, instead relying on pseudo-Fr\'echet means. In \citet{Law2019LorentzianDL}, the authors utilize a novel squared Lorentzian distance (as opposed to the canonical distance for hyperbolic space) to derive explicit formulas for the Fr\'echet mean in pseudo-hyperbolic space. In \citet{Chami2019HyperbolicGC}, the authors use an aggregation method in the tangent space as a substitute. Our work, to the best of our knowledge, is the first to provide explicit derivative expressions for the Fr\'echet mean on Riemannian manifolds.





\textbf{Differentiating through the argmin.} Theoretical foundations of differentiating through the argmin operator have been provided in \citet{Gould2016OnDP}. Similar methods have subsequently been used to develop differentiable optimization layers in neural networks \cite{amos2017optnet, agrawal2019differentiable}.

Given that the Fr\'echet mean is an argmin operation, one might consider utilizing the above differentiation techniques. However, a na\"{i}ve application fails, as the Fr\'echet mean's argmin domain is a manifold, and \citet{Gould2016OnDP} deals specifically with Euclidean space. Our paper extends this theory to the case of general Riemannian manifolds, thereby allowing the computation of derivatives for more general argmin problems, and, in particular, for computing the derivative of the Fr\'echet mean in hyperbolic space. 


\begin{table*}[!ht]
\centering
\caption{\label{tab:operations} Summary of operations in the Poincar\'e ball model and the hyperboloid model ($K<0$)}
\resizebox{\textwidth}{!}{%
\begin{tabular}{ccc}
\toprule
& \multicolumn{1}{c}{\textbf{Poincar\'e Ball}} & \multicolumn{1}{c}{\textbf{Hyperboloid}} \\
\rule{0pt}{2ex}
\textbf{Manifold} & 
    $\mathbb{D}_{K}^{n}=\{x\in \mathbb{R}^{n}: \langle x, x\rangle_{2}<-\frac{1}{K}\}$ & $\mathbb{H}_K^n = \{x \in \R^{n + 1} : \langle x, x\rangle_{\mathcal{L}}=\frac{1}{K}\}$ \\
\midrule
\textbf{Metric} & $g_{x}^{\mathbb{D}_{K}}=(\lambda_{x}^{K})^{2}g^{\mathbb{E}}$ where  $\lambda_{x}^{K}=\frac{2}{1+K\|x\|_{2}^{2}}$ and $g^{\mathbb{E}} = I$ & $g_x^{\mathbb{H}_K} = \eta$, where $\eta$ is $I$ except $\eta_{0,0} = -1$\\
\midrule
\textbf{Distance} & $d_{\mathbb{D}}^{K}(x, y)=\frac{1}{\sqrt{|K|}}\cosh^{-1}\left(1-\frac{2K\|x-y\|_{2}^{2}}{(1+K\|x\|_{2}^{2})(1+K\|y\|_{2}^{2})}\right)$ & $d_{\mathbb{H}}^{K}(x, y) =\frac{1}{\sqrt{|K|}} \cosh^{-1}(K\langle x, y\rangle_\Ll)$ \\
\midrule
\textbf{Exp map} & $\exp_{x}^{K}(v) = x\oplus_{K}\left(\tanh\left(\sqrt{|K|}\frac{\lambda_{x}^{K}\|v\|_{2}}{2}\right)\frac{v}{\sqrt{|K|}\|v\|_{2}}\right)$ & $\exp_{x}^{K}(v) = \cosh(\sqrt{|K|}||v||_{\Ll})x +  v \frac{\sinh(\sqrt{|K|}||v||_{\Ll})}{\sqrt{|K|}||v||_{\Ll}}$ \\
\midrule
\textbf{Log map} & $\log_{x}^{K}(y)=\frac{2}{\sqrt{|K|}\lambda_{x}^{K}}\tanh^{-1}(\sqrt{|K|}\|-x\oplus_{K} y\|_{2})\frac{-x\oplus_{K} y}{\|-x\oplus_{K} y\|_{2}}$ & $\log_{x}^{K}(y) = \frac{\cosh^{-1}(K\inn{x, y}_{\mathcal{L}})}{\sinh\left(\cosh^{-1}(K\langle x, y\rangle_{\Ll})\right)}(y-K\langle x, y\rangle_{\mathcal{L}}x)$ \\
\midrule
\textbf{Transport} & $PT_{x\rightarrow y}^{K}(v)=\frac{\lambda_{x}^{K}}{\lambda_{y}^{K}}\gyr[y, -x]v$  & $PT_{x\rightarrow y}^{K}(v)= v - \frac{K\langle y, v \rangle_{\Ll}}{1+K\langle x, y\rangle_{\Ll}} (x + y)$ \\
\bottomrule
\end{tabular}}
\end{table*}




\begin{table*}[!ht]
\centering
\caption{\label{table:hypnet} Summary of hyperbolic counterparts of Euclidean operations in neural networks}
\begin{tabular}{cc} 
    \toprule
    \textbf{Operation} & \textbf{Formula} \\
    \midrule
    Matrix-vector multiplication & $A \otimes^{K} x = \exp_{0}^{K}(A \log_{0}^{K}(x))$ \\
    Bias translation & $x\oplus^{K} b=\exp_{x}(PT_{0\rightarrow x}^{K}(b))$ \\
    Activation function & $\sigma^{K_{1}, K_{2}}(x) = \exp_{0}^{K_{1}}(\sigma(\log_{0}^{K_{2}}(x)))$ \\
    \bottomrule
\end{tabular}
\end{table*}

\section{Background}
\label{background}

In this section, we establish relevant definitions and formulas of Riemannian manifolds and hyperbolic spaces. We also briefly introduce neural network layers in hyperbolic space.

\subsection{Riemannian Geometry Background}

Here we provide some of the useful definitions from Riemannian geometry. For a more in-depth introduction, we refer the interested reader to our Appendix \ref{appendix:detailedDG} or texts such as \citet{lee2003introduction} and \citet{lee1997riemannian}.

\textbf{Manifold and tangent space:} An $n$-dimensional manifold $\mathcal{M}$ is a topological space that is locally homeomorphic to $\mathbb{R}^{n}$. The tangent space $T_x\M$ at $x$ is defined as the vector space of all tangent vectors at $x$ and is isomorphic to $\R^n$. We assume our manifolds are smooth, i.e. the maps are diffeomorphic. The manifold admits local coordinates $(x_1, \dots, x_n)$ which form a basis $(d x_1, \dots, d x_n)$ for the tangent space.

\textbf{Riemannian metric and Riemannian manifold:} For a manifold $\mathcal{M}$, a Riemannian metric $\rho=(\rho_{x})_{x\in \mathcal{M}}$ is a smooth collection of inner products $\rho_{x}:T_{x}\mathcal{M}\times T_{x}\mathcal{M} \to \R$ on the tangent space of every $x \in \M$. The resulting pair $(\mathcal{M}, \rho)$ is called a Riemannian manifold. Note that $\rho$ induces a norm in each tangent space $T_x\M$, given by $\norm{\vec{v}}_\rho = \sqrt{\rho_{x}(\vec{v}, \vec{v})}$ for any $\vec{v}\in T_{x}\mathcal{M}$. We oftentimes associate $\rho$ to its matrix form $(\rho_{ij})$ where $\rho_{ij}=\rho(dx_i, dx_j)$ when given local coordinates.

\textbf{Geodesics and induced distance function:} For a curve $\gamma: [a, b] \to \M$, we define the length of $\gamma$ to be $L(\gamma) = \int_a^b \norm{\gamma'(t)}_\rho dt$. For $x, y \in \M$, the distance $d(x, y) = \inf L(\gamma)$ where $\gamma$ is any curve such that $\gamma(a) = x, \gamma(b) = y$. A geodesic $\gamma_{xy}$ from $x$ to $y$, in our context, should be thought of as a curve that minimizes this length\footnote{Formally, geodesics are curves with 0 acceleration w.r.t. the Levi-Civita connection. There are geodesics which are not minimizing curves, such as the larger arc between two points on a great circle of a sphere; hence this clarification is important.}.

\textbf{Exponential and logarithmic map:} For each point $x\in \mathcal{M}$ and vector $\vec{v} \in T_x\M$, there exists a unique geodesic $\gamma: [0, 1] \to \M$ where $\gamma(0) = x, \gamma'(0) = \vec{v}$. The exponential map $\exp_x : T_x \M \to \M$ is defined as $\exp_x(\vec{v}) = \gamma(1)$. Note that this is an isometry, i.e. $\norm{\vec{v}}_\rho=d(x, \exp_x(\vec{v}))$. The logarithmic map $\log_{x}: \mathcal{M}\rightarrow T_{x}\mathcal{M}$ is defined as the inverse of $\exp_{x}$, although this can only be defined locally\footnote{Problems in definition arise in the case of conjugate points \cite{lee1997riemannian}. However, $\exp$ is a local diffeomorphism by the inverse function theorem.}. 

\textbf{Parallel transport:} For $x, y\in \mathcal{M}$, the parallel transport $PT_{x\rightarrow y}: T_{x}\mathcal{M}\rightarrow T_{y}\mathcal{M}$ defines a way of transporting the local geometry from $x$ to $y$ along the unique geodesic that preserves the metric tensors.
 
\subsection{Hyperbolic Geometry Background}

We now examine hyperbolic space, which has constant curvature $K < 0$, and provide concrete formulas for computation. The two equivalent models of hyperbolic space frequently used are the Poincar\'e ball model and the hyperboloid model. We denote $\mathbb{D}_{K}^{n}$ and $\mathbb{H}_{K}^{n}$ as the $n$-dimensional Poincar\'e ball and hyperboloid models with curvature $K<0$, respectively.

\subsubsection{Basic Operations}

\textbf{Inner products:} We define $\inn{x, y}_2$ to be the standard Euclidean inner product and $\inn{x, y}_\mathcal{L}$ to be the Lorentzian inner product $-x_0y_0 + x_1y_1 + \dots + x_ny_n$.

\textbf{Gyrovector operations:} For $x, y\in \mathbb{D}_K^n$, the M\"{o}bius addition \cite{Ungar2009AGS} is
\begin{equation}
    x\oplus_{K}y=\frac{(1-2K\langle x, y\rangle_{2}-K\norm{y}_{2}^{2})x+(1+K\norm{x}_{2}^{2})y}{1-2K\langle x, y\rangle_{2}+K^{2}\norm{x}_{2}^{2}\norm{y}_{2}^{2}}
\end{equation}

This induces M\"{o}bius subtraction $\ominus_K$ which is defined as $x \ominus_K y = x \oplus_K -y$. In the theory of gyrogroups, the notion of the gyration operator \cite{Ungar2009AGS} is given by
\begin{equation}
    \gyr[x, y]v = \ominus_K(x \oplus_K y) \oplus_K (x \oplus_K (y \oplus_K v))
\end{equation}

\noindent \textbf{Riemannian operations on hyperbolic space:} We summarize computations for the Poincar\'e ball model and the hyperboloid model in Table \ref{tab:operations}.




\subsection{Hyperbolic Neural Networks}

Introduced in \citet{Ganea2018HyperbolicNN}, hyperbolic neural networks provide a natural generalization of standard neural networks.

\noindent \textbf{Hyperbolic linear layer:} Recall that a Euclidean linear layer is defined as $f: \R^m \to \R^n$, $f = \sigma(Ax + b)$ where $A \in \R^{n \times m}$, $x \in \mathbb{R}^m$, $b \in \R^n$ and $\sigma$ is some activation function.

With analogy to Euclidean layers, a hyperbolic linear layer $g:\mathbb{H}^{m}\rightarrow\mathbb{H}^{n}$ is defined by  $g=\sigma^{K, K}(A\otimes^{K} x\oplus^{K}b)$, where $A\in \mathbb{R}^{n\times m}$,  $x \in \mathbb{H}^m$, $b\in \mathbb{H}^{n}$, and we replace the operations by hyperbolic counterparts outlined in Table \ref{table:hypnet}.

Hyperbolic neural networks are defined as compositions of these layers, similar to how conventional neural networks are defined as compositions of Euclidean layers.



\section{A Differentiable Fr\'echet Mean Operation for General Riemannian Manifolds}
\label{sec:generalriemann}

In this section, we provide a few theorems that summarize our method of differentiating through the Fr\'echet mean.

\subsection{Background on the Fr\'echet Mean}

\textbf{Fr\'echet mean and variance:} On a Riemannian manifold $(\M, \rho)$, the Fr\'echet mean $\mu_{fr}\in \mathcal{M}$ and Fr\'echet variance $\sigma_{fr}^{2}\in \mathbb{R}$ of a set of points $\mathcal{B}=\{x^{(1)}, \cdots, x^{(t)}\}$ with each $x^{(l)}\in \mathcal{M}$ are defined as the solution and optimal values of the following optimization problem \cite{bacak2014computing}:

\begin{equation}\label{frechetMean}
    \mu_{fr}=\argmin\limits_{\mu\in \mathcal{M}}\frac{1}{t}\sum\limits_{l=1}^{t}d(x^{(l)}, \mu)^{2}
\end{equation}
\begin{equation}\label{frechetVar}
    \sigma_{fr}^{2}=\min\limits_{\mu\in \mathcal{M}}\frac{1}{t}\sum\limits_{l=1}^{t}d(x^{(l)}, \mu)^{2}
\end{equation}


In Appendix \ref{appendix:Gen}, we provide proofs to illustrate that this definition is a natural generalization of Euclidean mean and variance.

The Fr\'echet mean can be further generalized with an arbitrary re-weighting. In particular, for positive weights $\brac{w_{l}}_{l \in [t]}$, we can define the weighted Fr\'echet mean as:
\begin{equation}\label{frechetMeanW}
    \mu_{fr}=\argmin\limits_{\mu\in \mathcal{M}}\sum\limits_{l=1}^{t}w_{l}\cdot d(x^{(l)}, \mu)^{2}
\end{equation}

This generalizes the weighted Euclidean mean to Riemannian manifolds.

\begin{figure}[htb!]
    \centering
    \includegraphics[scale=0.5]{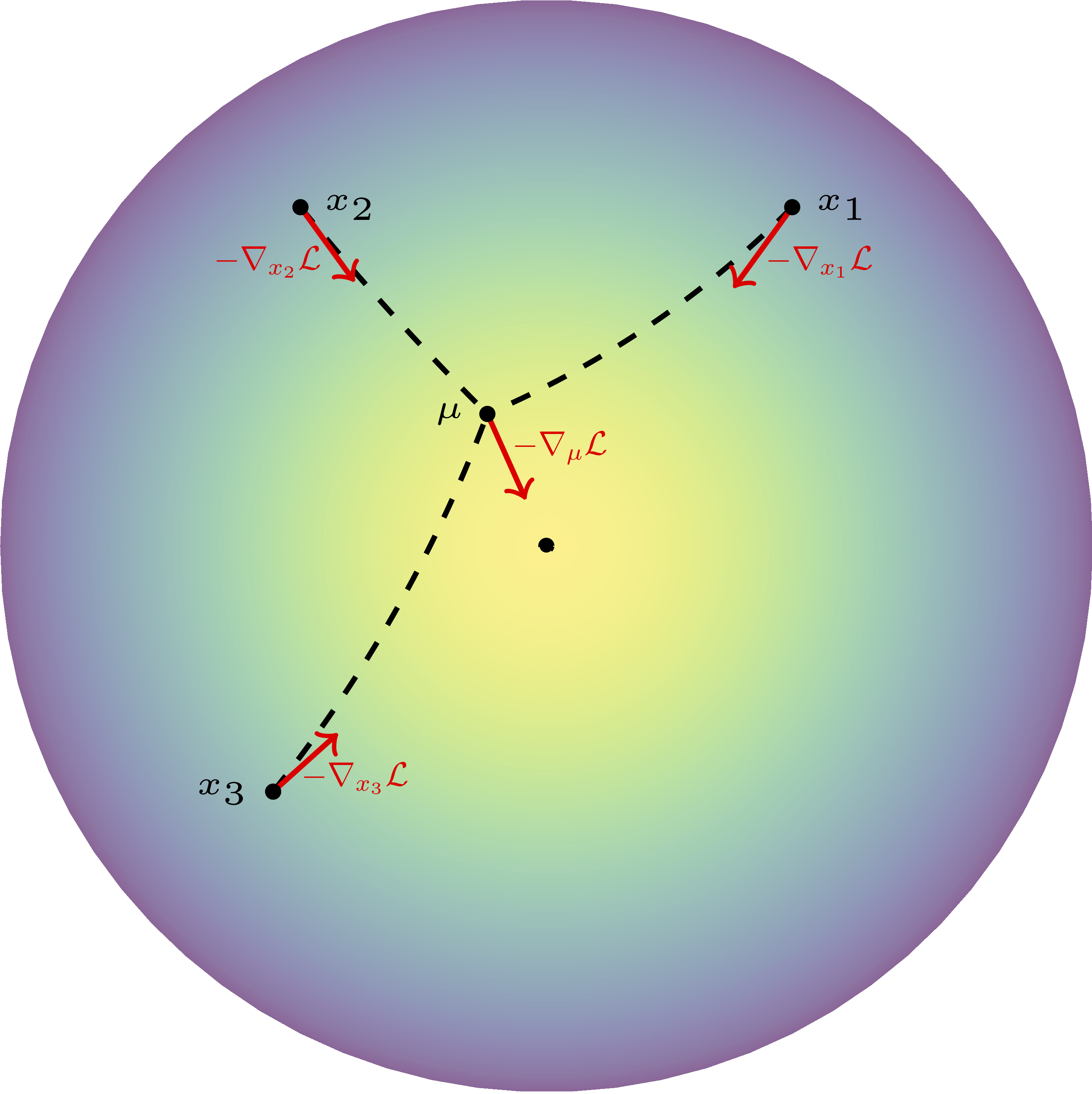}
    \caption{\label{fig:diffexample} Depicted above is the Fr\'echet mean, $\mu$, of three points, $x_1, x_2, x_3$ in the Poincar\'e ball model of hyperbolic space, $\mathbb{D}^2_{-1}$, as well as the negative gradients (shown in red) with respect to the loss function $\mathcal{L} = \norm{\mu}^2$.}
\end{figure}

\subsection{Differentiating Through the Fr\'echet Mean}


All known methods for computing the Fr\'echet mean rely on some sort of iterative solver \cite{gu2018learning}. While backpropagating through such a solver is possible, it is computationally inefficient and suffers from numerical instabilities akin to those found in RNNs \cite{Pascanu2012OnTD}. To circumvent these issues, recent works compute gradients at the solved value instead of differentiating directly, allowing for full neural network integration \cite{Chen2018NeuralOD, pogan2020differentiation}. However, to the best of our knowledge, no paper has investigated backpropagation on a manifold-based convex optimization solver.  Hence, in this section, we construct the gradient, relying on the fact that the Fr\'echet mean is an argmin operation.

\subsubsection{Differentiating through the argmin Operation}

Motivated by previous works on differentiating argmin problems \cite{Gould2016OnDP}, we propose a generalization which allows us to differentiate the argmin operation on the manifold. The full theory is presented in Appendix \ref{appendix:diffManifold}.

\subsubsection{Construction of the Fr\'echet mean derivative}

Since the Fr\'echet mean is an argmin operation, we can apply the theorems in Appendix \ref{appendix:diffManifold} to obtain gradients with respect to the input points. This operation (as well the resulting gradients) are visualized in Figure \ref{fig:diffexample}.

For the following theorems, we denote $\widetilde{\nabla}$ as the total derivative (or Jacobian) for notational convenience.

\begin{thm}\label{thm:diffFrechetMean}
     Let $\M$ be an $n$-dimensional Riemannian manifold, and let $\{x\} = (x^{(1)}, \dots, x^{(t)}) \in (\M)^t$ be a set of data points with weights $w_1, \dots, w_t \in \R^{+}$. Let $f : (\M)^t \times \M \to \M$ be given by $f(\brac{x}, y) = \sum\limits_{l = 1}^t w_{l}\cdot d(x^{(l)}, y)^2$ and $\overline{x} = \mu_{fr}(\brac{x}) = \argmin_{y \in \M} f(\{x\}, y)$ be the Fr\'echet mean. Then with respect to local coordinates we have
    \begin{equation}\label{eqn:manifoldgradient}
        \widetilde{\nabla}_{x^{(i)}}\mu_{fr}(\brac{x}) = - f_{YY}(\brac{x}, \overline{x})^{-1} f_{X^{(i)}Y}(\brac{x}, \overline{x})
    \end{equation}
    
    where the functions $f_{X^{(i)}Y}(\brac{x}, y) = \widetilde{\nabla}_{x^{(i)}} \nabla_y f(\brac{x}, y)$ and $f_{YY}(\brac{x}, y) = \nabla_{yy}^2 f(\brac{x}, \overline{x})$ are defined in terms of local coordinates.
\end{thm}
\begin{proof}
This is a special case of Theorem \ref{thm:diffthroughmanifoldgeneral} in the appendix. This is because the Fr\'echet objective function $f$ is a twice differentiable real-valued function for specific $x^{(i)}$ and $y$ (under our geodesic formulation); thus we obtain the desired formulation. The full explanation can be found in Remark \ref{remark:remarkfrechetmean}.
\end{proof}

While the above theorem gives a nice theoretical framework with minimal assumptions, it is in practice too unwieldy to apply. In particular, the requirement of local coordinates renders most computations difficult. We now present a version of the above theorem which assumes that the manifold is embedded in Euclidean space\footnote{We also present a more general way to take the derivative that drops this restriction via an exponential map-based parameterization in Appendix \ref{appendix:diffManifold}.}.

\begin{thm}\label{thm:diffFrechetMeanEmbedded}
    Assume the conditions and values in Theorem \ref{thm:diffFrechetMean}. Furthermore, assume $\M$ is embedded (as a Riemannian manifold) in $\R^m$ with $m \ge \dim \M$, then we can write
    \begin{equation}\label{eqn:embeddedmanifoldgrad}
        \begin{gathered}
        \widetilde{\nabla}_{x^{(i)}}\mu_{fr}(\brac{x}) = - f_{YY}^p(\brac{x}, \overline{x})^{-1} f_{X^{(i)}Y}^p(\brac{x}, \overline{x})
        \end{gathered}
    \end{equation}
    
    where $f_{YY}^p(\brac{x}, y) = \widetilde{\nabla}_y (\proj_{T_{\overline{x}}\M} \circ \nabla_y f)(\brac{x}, y)$, $f_{X^{(i)}Y}^p(\brac{x}, y) = \widetilde{\nabla}_{x^{(i)}}(\proj_{T_{\overline{x}}\M} \circ \nabla_y f)(\brac{x}, y)$, and $\proj_{T_{\overline{x}}\M} : \R^m \to T_{\overline{x}}\M \cong \R^n$ is the linear subspace projection operator. 
\end{thm}

\begin{proof}
    Similar to the relationship between Theorem \ref{thm:diffFrechetMean} and Theorem \ref{thm:diffthroughmanifoldgeneral}, this is a special case of Theorem \ref{thm:diffthroughmanifoldembedded} in the appendix.
\end{proof}

\section{Hyperbolic Fr\'echet Mean}

Although we have provided a formulation for differentiating through the Fr\'echet mean on general Riemannian manifolds, to properly integrate it in the hyperbolic setting we need to address two major difficulties:

\begin{enumerate}
    \item The lack of a fast forward computation.
    \item The lack of an explicit derivation of a backpropagation formula.
\end{enumerate}

Resolving these difficulties will allow us to define a Fr\'echet mean neural network layer for geometric, and specifically hyperbolic, machine learning tasks.

\subsection{Forward Computation of the Hyperbolic Fr\'echet Mean}

Previous forward computations fall into one of two categories: (1) fast, inaccurate computations which aim to approximate the true mean with a pseudo-Fr\'echet mean, or (2) slow, exact computations. In this section we focus on outperforming methods in the latter category, since we strive to compute the exact Fr\'echet mean (pseudo-means warp geometry).

\subsubsection{Former Attempts at Computing the Fr\'echet Mean}


The two existing algorithms for Fr\'echet mean computation are (1) Riemannian gradient-based optimization \cite{gu2018learning} and (2) iterative averaging \cite{karcher1977riemannian}. However, in practice both algorithms are slow to converge even for simple synthetic examples of points in hyperbolic space. To overcome this difficulty, which can cripple neural networks, we propose the following algorithm that is much faster in practice.


\begin{algorithm}[!htb]
\caption{\label{alg:poincareforward_mainpaper} Poincar\'e model Fr\'echet mean algorithm}
\textbf{Inputs}: $x^{(1)}, \cdots, x^{(t)}\in \mathbb{D}_{K}^{n}\subseteq \mathbb{R}^{n+1}$ and weights $w_1, \dots, w_t \in \R^{+}$.

\textbf{Algorithm}: 

$y_{0}=x^{(1)}$

Define $g(y)=\frac{2\arccosh(1+2y)}{\sqrt{y^{2}+y}}$

for $k=0, 1, \cdots, T$:

\hspace{20pt} for $l=1, 2, \cdots, t$: 

\hspace{40pt} $\alpha_{l}=w_{l} g\left(\frac{|K| \cdot \|x^{(l)}-y_{k}\|^{2}}{(1- |K| \cdot \|x^{(l)}\|^{2})(1-|K| \cdot \|y_{k}\|^{2})}\right) \frac{1}{1-|K| \cdot \|x^{(l)}\|^{2}}$

\hspace{20pt} $a=\sum\limits_{l=1}^{t}\alpha_{l} \text{  ,  } b=\sum\limits_{l=1}^{t}\alpha_{l}x^{(l)} \text{  ,  } c=\sum\limits_{l=1}^{t}\alpha_{l}\|x^{(l)}\|^{2}$

\hspace{20pt} $y_{k+1}=\left(\frac{(a+c|K|)-\sqrt{(a+c|K|)^{2}-4|K|\cdot\|b\|^{2}}}{2|K|\cdot\|b\|^{2}}\right)b$

return $y_{T}$
\end{algorithm}

\subsubsection{Algorithm for Fr\'echet Mean Computation via First-order Bound}

The core idea of our algorithm relies on the fact that the square of distance metric is a concave function for both the Poincar\'e ball and hyperboloid model. Intuitively, we select an initial ``guess" and use a first-order bound to minimize the Fr\'echet mean objective. The concrete algorithm for the Poincar\'e ball model is given as Algorithm \ref{alg:poincareforward_mainpaper} above. Note that the algorithm is entirely hyperparameter-free and does not require setting a step-size. Additionally we introduce three different initializations:
\begin{enumerate}
    \item Setting $y_0 = x^{(1)}$.
    \item Setting $y_0 = x^{(\argmax_i w_i)}$.
    \item Setting $y_0$ to be the output of the first step of the Karcher flow algorithm \cite{karcher1977riemannian}.
\end{enumerate}

We tried these initializations for our test tasks (in which weights were equal, tasks described in Section \ref{sec:casestudies}), and found little difference between them in terms of performance. Even for toy tasks with varying weights, these three methods produced nearly the same results. However, we give them here for completeness.

Moreover, we can prove that the algorithm is guaranteed to converge.

\begin{thm}
Let $x^{(1)}, \cdots, x^{(t)}\in \mathbb{D}_{K}^{n}$ be $t$ points\footnote{Here we present the version for $K=-1$ for cleaner presentation. The generalization to arbitrary $K<0$ is easy to compute, but clutters presentation.} in the Poincar\'e ball, $w_1, \dots, w_t \in \R^{+}$ be their weights, and let their weighted Fr\'echet mean be the solution to the following optimization problem.
\begin{equation}
\begin{gathered}
    \mu_{fr}=\argmin_{y \in \mathbb{D}^{n}_K}f(y)
\end{gathered}
\end{equation}
\begin{equation}
\begin{gathered}
    \text{where } f(y) =\sum_{l=1}^{t}w_{l}\cdot d_{\mathbb{D}^{n}_K}(x^{(l)}, y)^2 \\ 
    =\sum_{l=1}^{t} \frac{w_{l}}{|K|}  \arccosh^2\left(1-\frac{2K\|x^{(l)}-y\|^{2}}{(1+K \|x^{(l)}\|^{2})(1+K\|y\|^{2})}\right)
\end{gathered}
\end{equation}
Then Algorithm \ref{alg:poincareforward_mainpaper} gives a sequence of points $\{y_{k}\}$ such that their limit $\lim\limits_{k\rightarrow\infty}y_{k}=\mu_{fr}$ converges to the Fr\'echet mean solution.
\end{thm}
\begin{proof}
    See Theorem \ref{thm:poincareforwardconvergence} in the appendix.
\end{proof}

The algorithm and proof of convergence for the hyperboloid model are given in Appendix \ref{sec:hyperboloidforwardderiv} and are omitted here for brevity.

\subsubsection{Empirical Comparison to Previous Fr\'echet Mean Computation Algorithms}
To demonstrate the efficacy of our algorithm, we compare it to previous approaches on randomly generated data. Namely, we compare against a na\"{i}ve Riemannian Gradient Descent (RGD) approach \cite{Udriste1997} and against the Karcher Flow algorithm \cite{karcher1977riemannian}. We test our Fr\'echet mean algorithm against these methods on synthetic datasets of ten on-manifold randomly generated $16$-dimensional points. We run all algorithms until they are within $\epsilon=10^{-12}$ of the true Fr\'echet mean in norm, and report the number of iterations this takes in Table \ref{tab:empiricalfrechetmean} for both hyperboloid (H) and Poincar\'e (P) models of hyperbolic space. Note that we significantly outperform the other algorithms. We also observe that by allowing 200x more computation, a grid search on the learning hyperparameter\footnote{The grid search starts from $lr=0.2$ and goes to $lr=0.4$ in increments of $0.01$ for the Poincar\'e ball model, and from $lr=0.2$ to $0.28$ for the hyperboloid model (same increment).} in RGD obtains nearly comparable or better results (last row of Table \ref{tab:empiricalfrechetmean} for both models). However, we stress that this requires much more computation, and note that our algorithm produces nearly the same result while being \textbf{hyperparameter-free}.

\begin{table}[!htb]
\centering
\caption{\label{tab:empiricalfrechetmean} Empirical computation of the Fr\'echet mean; the average number of iterations, as well as runtime, required to become accurate within $\epsilon=10^{-12}$ of the true Fr\'echet mean are reported. $10$ trials are conducted, and standard deviation is reported. The primary baselines are the RGD \cite{Udriste1997} and Karcher Flow \cite{karcher1977riemannian} algorithms. (H) refers to hyperboloid and (P) refers to Poincar\'e.}. 
\vspace{0.2em}
\small
\begin{tabular}{cccc}
    \toprule
    & & Iterations & Time (ms)\footnotemark \\
    \midrule
    \parbox[t]{2mm}{\multirow{4}{*}{\rotatebox[origin=c]{90}{\small{H}}}} & RGD ($lr=0.01$) & $801.0$\tiny$\pm 21.0$ & $932.9$\tiny$\pm 130.0$ \\
    & Karcher Flow & $62.5$\tiny$\pm 6.0$ & $50.9$\tiny$\pm 8.9$ \\
    & Ours & $\mathbf{13.7}$\tiny$\pm 0.9$ & $\mathbf{6.1}$\tiny$\pm 1.9$ \\
    \cline{2-4}
    \rule{0pt}{3ex} & RGD + Grid Search on $lr$ & $27.7$\tiny$\pm 0.8$ & $5333.5$\tiny$\pm 770.7$ \\
    \midrule
    \parbox[t]{2mm}{\multirow{4}{*}{\rotatebox[origin=c]{90}{\small{P}}}} & RGD ($lr=0.01$) & $773.8$\tiny$\pm 22.1$ & $1157.3$\tiny$\pm 74.8$ \\
    & Karcher Flow & $57.5$\tiny$\pm 9.1$ & $59.8$\tiny$\pm 10.4$ \\
    & Ours & $\mathbf{13.4}$\tiny$\pm 0.5$ & $\mathbf{9.1}$\tiny$\pm 1.3$ \\
    \cline{2-4}
    \rule{0pt}{3ex} & RGD + Grid Search on $lr$  & $10.5$\tiny$\pm 0.5$ & $6050.6$\tiny$\pm 235.2$ \\
    \bottomrule
\end{tabular}
\end{table}
\footnotetext{Experiments were run with an Intel Skylake Core i7-6700HQ 2.6 GHz Quad core CPU.}

We also find that this convergence improvement translates to real world applications. Specifically, we find that for the graph link prediction experimental setting in Section \ref{sec:experimentalresults}, our forward pass takes anywhere from $\approx 15-25$ iterations, significantly outperforming the $1000+$ needed with RGD and $\approx 120$ needed with Karcher Flow.

\subsection{Backward Computation of the Hyperbolic Fr\'echet Mean}

For the backward computation, we re-apply the general Riemannian theory for differentiating through the Fr\'echet mean in Section \ref{sec:generalriemann} to hyperbolic space. Since most autodifferentiation packages do not support manifold-aware higher order differentiation, we derive the gradients explicitly. We begin with the Poincar\'e ball model by setting $\mathcal{M} = \mathbb{D}^n_K$ and applying Theorem \ref{thm:diffFrechetMeanEmbedded}.

\begin{thm}\label{thm:ballbackpropmain_mainpaper}
Let $x^{(1)}, \cdots, x^{(t)}\in \mathbb{D}_{K}^{n}\subseteq \mathbb{R}^{n}$ be $t$ points in the Poincar\'e ball and $w_1, \dots, w_t \in \R^{+}$ be the weights. Let their weighted Fr\'echet mean $\mu_{fr}$ be solution to the following optimization problem
\begin{equation}
    \mu_{fr}(x^{(1)}, \cdots, x^{(t)})=\argmin_{y \in \mathbb{D}_K^{n}}f(\{x\}, y)
\end{equation}
\begin{equation}
\begin{gathered}
    \text{where } f(\{x\}, y)= \sum_{l = 1}^t w_l \cdot d_{\mathbb{D}^{n}_K}(x^{(l)}, y)^2 = \\ \sum^{t}_{l=1} \frac{w_l}{|K|} \arccosh^2 \left(1 - \frac{2K ||x^{(l)}-y||^2_2}{(1+K||x^{(l)}||^2_2)(1+K||y||^2_2)} \right)
\end{gathered}
\end{equation}
Then the derivative of $\mu_{fr}$ with respect to $x^{(i)}$ is given by
\begin{equation}
 \widetilde{\nabla}_{x^{(i)}}\mu_{fr}(\{x\}) = -f_{YY}f(\brac{x}, \overline{x})^{-1} f_{X^{(i)}Y}(\brac{x}, \overline{x})
\end{equation}

where $\overline{x} = \mu_{fr}(\brac{x})$ and $f_{YY}$, $f_{X^{(i)} Y}$ are defined in Theorem \ref{thm:diffFrechetMeanEmbedded} \footnote{The projection operation is trivial since $\dim \mathbb{R}^n = \dim \mathbb{D}_K^n$.}.

The full concrete derivation of the above terms for the geometry induced by this manifold choice is given in Appendix Theorem \ref{thm:ballbackpropmain}.

\end{thm}
\begin{proof}
    This is a concrete application of Theorem \ref{thm:diffFrechetMeanEmbedded}. In particular since our manifold is embedded in $\mathbb{R}^n$ ($\mathbb{D}_{K}^{n}\subseteq \mathbb{R}^{n}$). Note that this is the total derivative in the ambient Euclidean space\footnote{To transform Euclidean gradients into Riemannian ones, simply multiply by inverse of the matrix of the metric.}. For the full proof see Theorem \ref{thm:ballbackpropmain} in the Appendix.
\end{proof}

The derivation for the hyperboloid model is given in Appendix \ref{thm:lorentzbackpropmain}.


\section{Case Studies}
\label{sec:casestudies}
To demonstrate the efficacy of our developed theory, we investigate the following test settings. In the first setting, we directly modify the hyperbolic aggregation strategy in Hyperbolic GCNs \cite{Chami2019HyperbolicGC} to use our differentiable Fr\'echet mean layer. This was the original intent\footnote{We quote directly from the paper \citet{Chami2019HyperbolicGC}: ``An analog of mean aggregation in hyperbolic space is the Fr\'echet mean, which, however, has no closed form solution. Instead, we propose to..."} but was not feasible without our formulation. In the second setting, we introduce Hyperbolic Batch Normalization (HBN) as an extension of the regular Euclidean Batch Normalization (EBN). When combined with hyperbolic neural networks \cite{Ganea2018HyperbolicNN}, HBN exhibits benefits similar to those of EBN with Euclidean networks.

\subsection{Hyperbolic Graph Convolutional Neural Networks (HGCNs)}

\subsubsection{Original Framework}
Introduced in \citet{Chami2019HyperbolicGC}, Hyperbolic Graph Convolutional Networks (GCNs) provide generalizations of Euclidean GCNs to hyperbolic space. The proposed network architecture is based on three different layer types: feature transformation, activation, and attention-based aggregation. 

\textbf{Feature transformation:} The hyperbolic feature transformation consists of a gyrovector matrix multiplication followed by a gyrovector addition. 
\begin{equation}
    h_{i}^{l}=(W^{l}\otimes^{K_{l-1}} x_{i}^{l-1})\oplus^{K_{l-1}} b^{l}
\end{equation}

\textbf{Attention-based aggregation:} Neighborhood aggregation combines local data at a node. It does so by projecting the neighbors using the logarithmic map at the node, averaging in the tangent space, and projecting back with the exponential map at the node. Note that the weights $w_{ij}$ are positive and can be trained or defined by the graph adjacency matrix.
\begin{equation}
    AGG^K(x_i)=\exp_{x_i}^K\paren{\sum_{j \in \mathcal{N}(i)} w_{ij} \log_{x_i}^K x_j}
\end{equation}


\textbf{Activation:} The activation layer applies a hyperbolic activation function.
\begin{equation}
    x_{i}^{l}=\sigma^{\otimes^{K_{l-1}, K_{l}}}(y_{i}^{l})
\end{equation}

\subsubsection{Proposed Changes}

The usage of tangent space aggregation in the HGCN framework stemmed from the lack of a differentiable Fr\'echet mean operation. As a natural extension, we substitute our Fr\'echet mean in place of the aggregation layer.

\subsubsection{Experimental Results}
\label{sec:experimentalresults}
We use precisely the same architecture as in \citet{Chami2019HyperbolicGC}, except we substitute all hyperbolic aggregation layers with our differentiable Fr\'echet mean layer. Furthermore, we test with precisely the same hyperparameters (learning rate, test/val split, and the like) as \citet{Chami2019HyperbolicGC} for a fair comparison. Our new aggregation allows us to achieve new state-of-the-art results on the Disease and Disease-M graph datasets \cite{Chami2019HyperbolicGC}. These datasets induce ideal test tasks for hyperbolic learning since they have very low Gromov $\delta$-hyperbolicity \cite{Adcock2013TreeLikeSI}, which indicates the structure is highly tree-like. Our results and comparison to the baseline are given in Table \ref{tab:hgcnmaintab}. We run experiments for 5 trials and report the mean and standard deviation. Due to practical considerations, we only test with the Poincar\'e model\footnote{The code for HGCN included only the Poincar\'e model implementation at the time this paper was submitted. Hence we use the Poincar\'e model for our experiments, although our contributions include derivations for both hyperboloid and Poincar\'e models.}. For reference, the strongest baseline results with the hyperboloid model are reported from \citet{Chami2019HyperbolicGC} (note that we outperform these results as well). On the rather non-hyperbolic CoRA \cite{Sen2008CollectiveCI} dataset, our performance is comparable to that of the best baseline. Note that this is similar to the performance exhibited by the vanilla HGCN. Hence we conjecture that when the underlying dataset is not hyperbolic in nature, we do not observe improvements over the best Euclidean baseline methods.

\begin{table}[!htb]
\centering
\caption{\label{tab:hgcnmaintab} ROC AUC results for Link Prediction (LP) on
various graph datasets, averaged over 5 trials (with standard deviations). Graph hyperbolicity values are also reported (lower $\delta$ is more hyperbolic). Results are given for models learning in Euclidean (E), Hyperboloid (H), and Poincar\'e (P) spaces. Note that the best Euclidean method is GAT \cite{velickovic2018graph} and is shown below for fair comparison on CoRA. We highlight the best result only if our result gives a p-value $< 0.01$ after running a paired-significance t-test.}
\begin{tabular}{ccccc}
    \toprule
    & & Disease & Disease-M & CoRA \\
    & & $\delta=0$ & $\delta=0$ & $\delta=11$ \\
    \midrule
    \parbox[t]{2mm}{\multirow{2}{*}{\rotatebox[origin=c]{90}{\large{E}}}} & MLP & $72.6$\tiny$\pm 0.6$  & $55.3$\tiny$\pm 0.5$ & $83.1$\tiny$\pm 0.5$ \\
    & GAT & $69.8$\tiny$\pm 0.3$ & $69.5$\tiny$\pm 0.4$ & $\mathbf{93.7}$\tiny$\pm 0.1$ \\
    \midrule
    \parbox[t]{2mm}{\multirow{2}{*}{\rotatebox[origin=c]{90}{\large{H}}}} & HNN & $75.1$\tiny$\pm 0.3$ & $60.9$\tiny$\pm 0.4$ & $89.0$\tiny$\pm 0.1$ \\
    & HGCN & $90.8$\tiny$\pm 0.3$ & $78.1$\tiny$\pm 0.4$ & $92.9$\tiny$\pm 0.1$ \\
    \midrule
    \parbox[t]{2mm}{\multirow{2}{*}{\rotatebox[origin=c]{90}{\large{P}}}} & HGCN & $76.4$\tiny$\pm 8.9$ & $81.4$\tiny$\pm 3.4$ & $93.4$\tiny$\pm 0.4$ \\
    & Ours & $\mathbf{93.7}$\tiny$\pm 0.4$ & $\mathbf{91.0}$\tiny$\pm 0.6$ & $92.9$\tiny$\pm 0.4$ \\
    \bottomrule
\end{tabular}
\end{table}


\begin{figure*}[htb!]
\centering

\begin{subfigure}{0.32\textwidth}
    \centering

    \begin{subfigure}{\textwidth}
        \centering
        \includegraphics[scale=0.40]{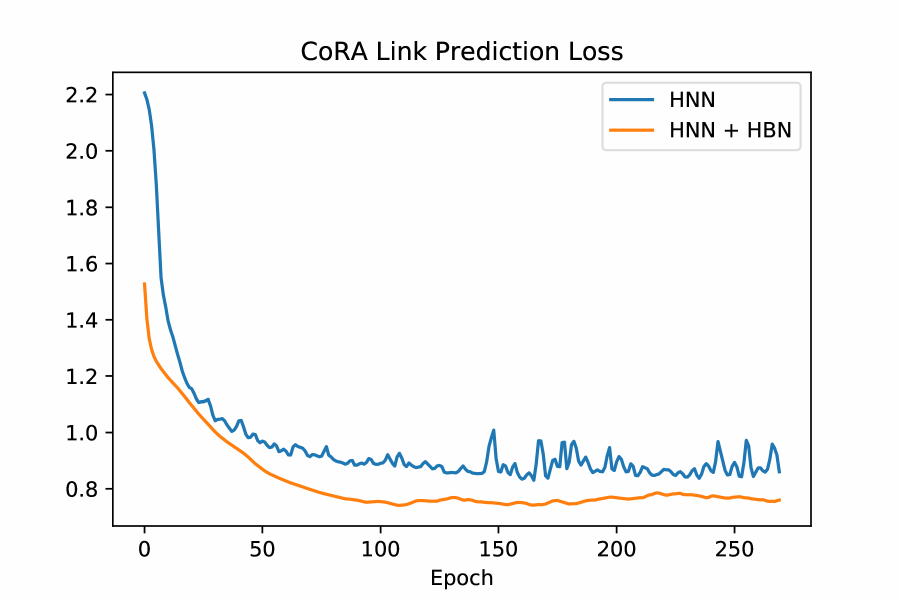}
    \end{subfigure}
    \begin{subfigure}{\textwidth}
        \centering
        \includegraphics[scale=0.40]{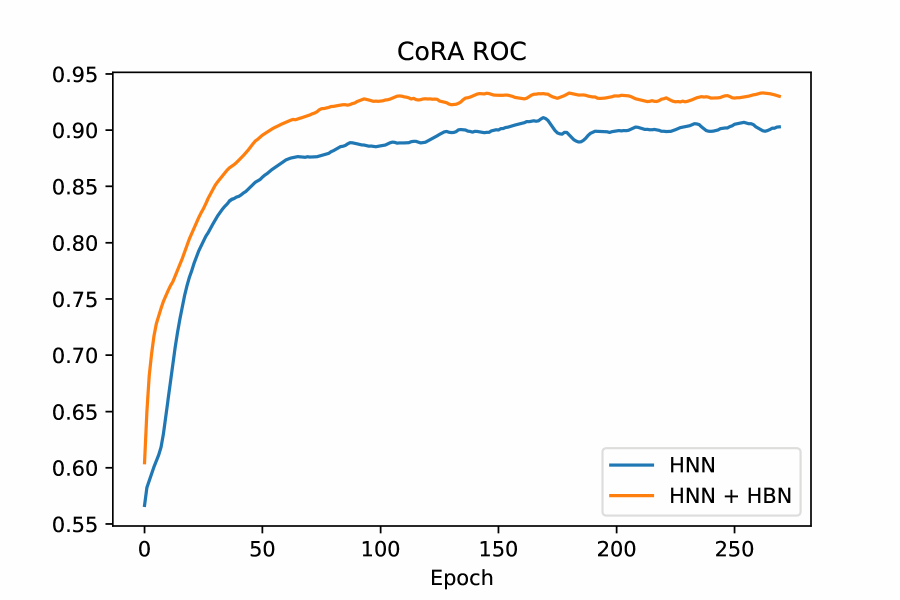}
    \end{subfigure}
\end{subfigure}%
\begin{subfigure}{0.32\textwidth}
    \centering

    \begin{subfigure}{\textwidth}
        \centering
        \includegraphics[scale=0.40]{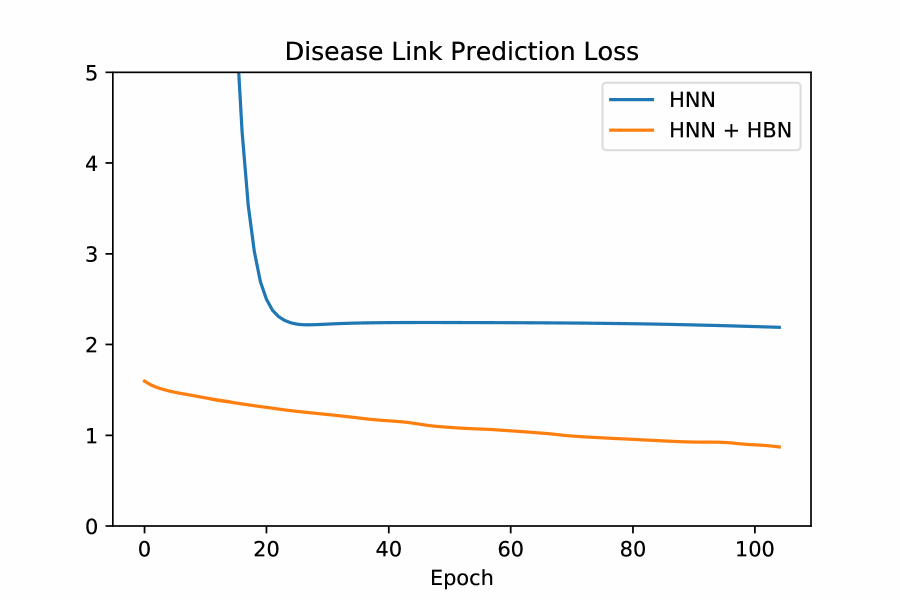}
    \end{subfigure}
    \begin{subfigure}{\textwidth}
        \centering
        \includegraphics[scale=0.40]{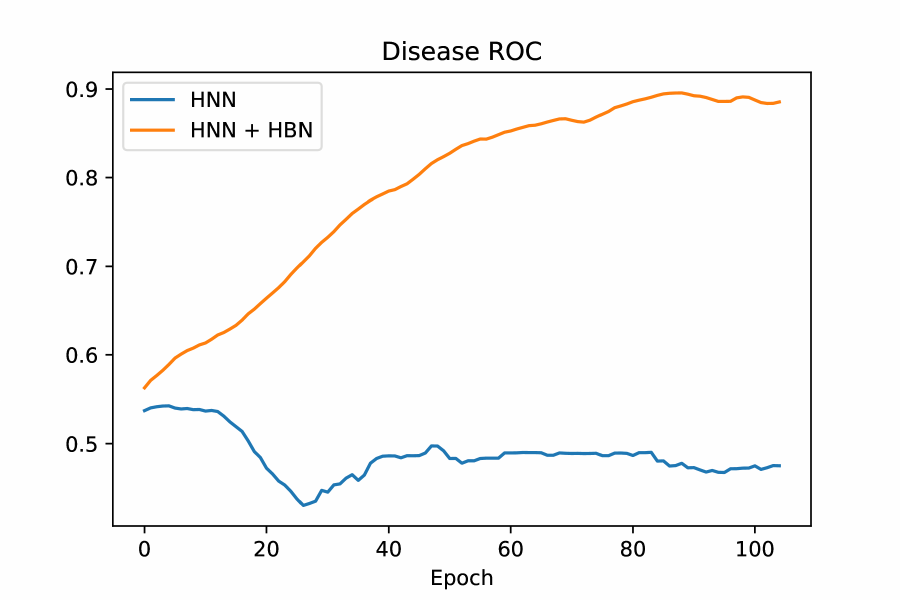}
    \end{subfigure}
\end{subfigure}%
\begin{subfigure}{0.32\textwidth}
    \centering
    
    \begin{subfigure}{\textwidth}
        \centering
        \includegraphics[scale=0.40]{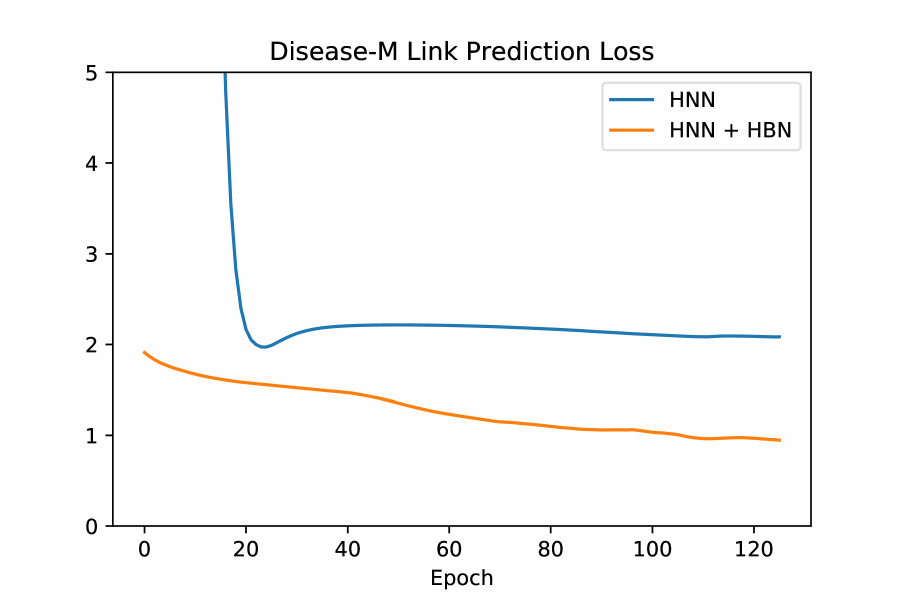}
    \end{subfigure}
    \begin{subfigure}{\textwidth}
        \centering
        \includegraphics[scale=0.40]{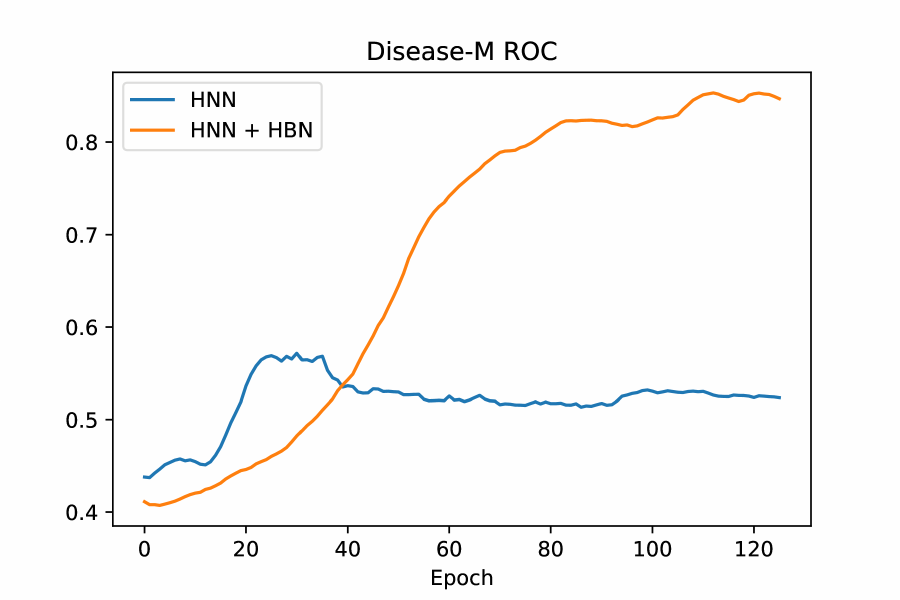}
    \end{subfigure}
\end{subfigure}
\caption{The graphs above correspond to a comparison of the HNN baseline, which uses a two-layer hyperbolic neural network encoder, and the baseline augmented with hyperbolic batch normalization after each layer. The columns correspond to the CoRA \cite{Sen2008CollectiveCI}, Disease \cite{Chami2019HyperbolicGC}, and Disease-M \cite{Chami2019HyperbolicGC} datasets, respectively. The top row shows the comparison in terms of validation loss, and the bottom row shows the comparison in terms of validation ROC AUC. The figures show that we converge faster and attain better performance in terms of both loss and ROC. Note that although CoRA is not hyperbolic (as previously mentioned), we find it encouraging that introducing hyperbolic batch normalization produces an improvement regardless of dataset hyperbolicity.}
\label{fig:allgraphs}
\end{figure*}

\subsection{Hyperbolic Batch Normalization}
Euclidean batch normalization \cite{Ioffe2015BatchNA} is one of the most widely used neural network operations that has, in many cases, obviated the need for explicit regularization such as dropout \cite{srivastava14a}. In particular, analysis demonstrates that batch normalization induces a smoother loss surface which facilitates convergence and yields better final results \cite{Santurkar2018HowDB}. Generalizing this for Riemannian manifolds is a natural extension, and such a computation would involve a differentiable Fr\'echet mean.

\subsubsection{Theoretical Formulation and Algorithm}

In this section we formulate Riemannian Batch Normalization as a natural extension of standard Euclidean Batch Normalization. This concept is, to the best of our knowledge, only touched upon by \citet{Brooks2019RiemannianBN} in the specific instance of the manifold of positive semidefinite matrices. However, we argue in Appendix \ref{appendix:rbn} that, unlike our method, their formulation is incomplete and lacks sufficient generality to be considered a true extension.

\begin{algorithm}[!htb]

\caption{\label{alg:RiemannianBN} Riemannian Batch Normalization}

\textbf{Training Input}: Batches of data points $\{x_1^{(t)}, \cdots, x_m^{(t)}\}\subseteq \mathcal{M}$ for $t \in [1, \dots, T]$, testing momentum $\eta \in [0, 1]$

\textbf{Learned Parameters}: Target mean $\mu'\in\mathcal{M}$, target variance $(\sigma')^{2}\in \mathbb{R}$ 

\textbf{Training Algorithm}:

$\mu_{test} \leftarrow \text{FrechetMean}(\{x_1^{(1)}, \dots, x_m^{(1)}\})$

$\sigma_{test} \leftarrow 0$

for $t = 1, \dots, T$:

\hspace{20pt} $\mu=\text{FrechetMean}(\{x_1^{(t)}, \dots, x_m^{(t)}\})$

\hspace{20pt} $\sigma^{2}=\frac{1}{m}\sum\limits_{i=1}^{m}d(x_i^{(t)},\mu)^{2}$

\hspace{20pt} $\mu_{test}=\text{FrechetMean}(\{\mu_{test}, \mu\}, \{\eta, 1-\eta\})$

\hspace{20pt} $\sigma_{test}= \frac{(t - 1)\sigma_{test}+ \sigma}{t}$

\hspace{20pt} for $i=1, \cdots, m$:

\hspace{40pt} $\tilde{x_i}^{(t)} \leftarrow \exp_{\mu'}\left(\frac{\sigma'}{\sigma} PT_{\mu \to \mu'} (\log_{\mu} x_i^{(t)})\right)$

\hspace{20pt} return normalized batch $\tilde{x_1}^{(t)}, \cdots, \tilde{x_m}^{(t)}$

\vspace{10pt}

\textbf{Testing Input}: Test data points $\{\overline{x_1}, \cdots, \overline{x_s}\}\subseteq \mathcal{M}$, final running mean $\mu_{test}$ and running variance $\sigma_{test}$

\textbf{Testing Algorithm}: 

$\overline{\mu} =\text{FrechetMean}(\{\overline{x_1}, \cdots, \overline{x_s}\})$

$\overline{\sigma}^{2}=\frac{1}{m}\sum\limits_{i=1}^{m}d(\overline{x_i}, \overline{\mu})^{2}$

for $i=1, \cdots, s$:

\hspace{20pt} $\tilde{\overline{x_i}} \leftarrow \exp_{\mu_{test}}\left(\frac{\sigma_{test}}{\overline{\sigma}} PT_{\overline{\mu} \to \mu_{test}} (\log_{\overline{\mu}} \overline{x_i})\right)$

return normalized batch $\tilde{\overline{x_1}}, \cdots, \tilde{\overline{x_s}}$

\end{algorithm}

Our full algorithm is given in Algorithm \ref{alg:RiemannianBN}. Note that in practice we use $\sqrt{\sigma^2 + \epsilon}$ in place of $\sigma$ as in the original formulation to avoid division by zero.



\subsubsection{Experimental Results}
We apply Riemannian Batch Normalization (specifically for hyperbolic space) to the encoding Hyperbolic Neural Network (HNN) \cite{Ganea2018HyperbolicNN} in the framework of \citet{Chami2019HyperbolicGC}. We run on the CoRA \cite{Sen2008CollectiveCI}, Disease \cite{Chami2019HyperbolicGC}, and Disease-M \cite{Chami2019HyperbolicGC} datasets and present the validation loss and ROC AUC diagrams in Figure \ref{fig:allgraphs}. 

In terms of both loss and ROC, our method results in both faster convergence and a better final result. These improvements are expected as they appear when applying standard batch normalization to Euclidean neural networks. So, our manifold generalization does seem to replicate the useful properties of standard batch normalization. Additionally, it is encouraging to see that, regardless of the hyperbolic nature of the underlying dataset, hyperbolic batch normalization produces an improvement when paired with a hyperbolic neural network.




\section{Conclusion and Future Work}

We have presented a fully differentiable Fr\'echet mean operation for use in any differentiable programming setting. Concretely, we introduced differentiation theory for the general Riemannian case, and for the demonstrably useful case of hyperbolic space, we provided a fast forward pass algorithm and explicit derivative computations. We demonstrated that using the Fr\'echet mean in place of tangent space aggregation yields state-of-the-art performance on link prediction tasks in graphs with tree-like structure. Additionally, we extended batch normalization (a standard Euclidean operation) to the realm of hyperbolic space. On a graph link prediction test task, we showed that hyperbolic batch normalization gives benefits similar to those experienced in the Euclidean setting.

We hope our work paves the way for future developments in geometric representation learning. Potential future work can focus on speeding up our computation of the Fr\'echet mean gradient, finding applications of our theory on manifolds beyond hyperbolic space, and applying the Fr\'echet mean to generalize more standard neural network operations.

\section{Acknowledgements}

We would like to acknowledge Horace He for his helpful comments regarding implementation. In addition, we would like to thank Facebook AI for funding equipment that made this work possible.

\bibliography{frechet_mean}

\begin{thebibliography}{37}
\providecommand{\natexlab}[1]{#1}
\providecommand{\url}[1]{\texttt{#1}}
\expandafter\ifx\csname urlstyle\endcsname\relax
  \providecommand{\doi}[1]{doi: #1}\else
  \providecommand{\doi}{doi: \begingroup \urlstyle{rm}\Url}\fi

\bibitem[Adcock et~al.(2013)Adcock, Sullivan, and
  Mahoney]{Adcock2013TreeLikeSI}
Adcock, A.~B., Sullivan, B.~D., and Mahoney, M.~W.
\newblock Tree-like structure in large social and information networks.
\newblock \emph{2013 IEEE 13th International Conference on Data Mining}, pp.\
  1--10, 2013.

\bibitem[Agrawal et~al.(2019)Agrawal, Amos, Barratt, Boyd, Diamond, and
  Kolter]{agrawal2019differentiable}
Agrawal, A., Amos, B., Barratt, S., Boyd, S., Diamond, S., and Kolter, J.~Z.
\newblock Differentiable convex optimization layers.
\newblock In \emph{Advances in Neural Information Processing Systems}, pp.\
  9558--9570, 2019.

\bibitem[Amos \& Kolter(2017)Amos and Kolter]{amos2017optnet}
Amos, B. and Kolter, J.~Z.
\newblock Optnet: Differentiable optimization as a layer in neural networks.
\newblock In \emph{Proceedings of the 34th International Conference on Machine
  Learning-Volume 70}, pp.\  136--145, 2017.

\bibitem[Bac{\'a}k(2014)]{bacak2014computing}
Bac{\'a}k, M.
\newblock Computing medians and means in hadamard spaces.
\newblock \emph{SIAM Journal on Optimization}, 24\penalty0 (3):\penalty0
  1542--1566, 2014.

\bibitem[Brooks et~al.(2019)Brooks, Schwander, Barbaresco, Schneider, and
  Cord]{Brooks2019RiemannianBN}
Brooks, D., Schwander, O., Barbaresco, F., Schneider, J.-Y., and Cord, M.
\newblock Riemannian batch normalization for spd neural networks.
\newblock In \emph{Advances in Neural Information Processing Systems}, pp.\
  15463--15474, 2019.

\bibitem[Casado(2019)]{LezcanoCasado2019TrivializationsFG}
Casado, M.~L.
\newblock Trivializations for gradient-based optimization on manifolds.
\newblock In \emph{Advances in Neural Information Processing Systems}, pp.\
  9154--9164, 2019.

\bibitem[Chami et~al.(2019)Chami, Ying, R\'e, and
  Leskovec]{Chami2019HyperbolicGC}
Chami, I., Ying, Z., R\'e, C., and Leskovec, J.
\newblock Hyperbolic graph convolutional neural networks.
\newblock In \emph{Advances in Neural Information Processing Systems}, pp.\
  4869--4880, 2019.

\bibitem[Charlier(2013)]{charlierbenjfrechet}
Charlier, B.
\newblock Necessary and sufficient condition for the existence of a fr{\'e}chet
  mean on the circle.
\newblock \emph{ESAIM: Probability and Statistics}, 17:\penalty0 635--649,
  2013.

\bibitem[Chen et~al.(2018)Chen, Rubanova, Bettencourt, and
  Duvenaud]{Chen2018NeuralOD}
Chen, T.~Q., Rubanova, Y., Bettencourt, J., and Duvenaud, D.~K.
\newblock Neural ordinary differential equations.
\newblock In \emph{Advances in Neural Information Processing Systems}, pp.\
  6571--6583, 2018.

\bibitem[Fr{\'e}chet(1948)]{frechet1948elements}
Fr{\'e}chet, M.
\newblock Les {\'e}l{\'e}ments al{\'e}atoires de nature quelconque dans un
  espace distanci{\'e}.
\newblock In \emph{Annales de l'institut Henri Poincar{\'e}}, volume~10, pp.\
  215--310, 1948.

\bibitem[Ganea et~al.(2018)Ganea, B{\'e}cigneul, and
  Hofmann]{Ganea2018HyperbolicNN}
Ganea, O., B{\'e}cigneul, G., and Hofmann, T.
\newblock Hyperbolic neural networks.
\newblock In \emph{Advances in Neural Information Processing Systems}, pp.\
  5345--5355, 2018.

\bibitem[Gould et~al.(2016)Gould, Fernando, Cherian, Anderson, Cruz, and
  Guo]{Gould2016OnDP}
Gould, S., Fernando, B., Cherian, A., Anderson, P., Cruz, R.~S., and Guo, E.
\newblock On differentiating parameterized argmin and argmax problems with
  application to bi-level optimization.
\newblock \emph{ArXiv}, abs/1607.05447, 2016.

\bibitem[Gu et~al.(2019)Gu, Sala, Gunel, and Ré]{gu2018learning}
Gu, A., Sala, F., Gunel, B., and Ré, C.
\newblock Learning mixed-curvature representations in product spaces.
\newblock In \emph{International Conference on Learning Representations}, 2019.

\bibitem[Gulcehre et~al.(2019)Gulcehre, Denil, Malinowski, Razavi, Pascanu,
  Hermann, Battaglia, Bapst, Raposo, Santoro, and
  de~Freitas]{Glehre2018HyperbolicAN}
Gulcehre, C., Denil, M., Malinowski, M., Razavi, A., Pascanu, R., Hermann,
  K.~M., Battaglia, P., Bapst, V., Raposo, D., Santoro, A., and de~Freitas, N.
\newblock Hyperbolic attention networks.
\newblock In \emph{International Conference on Learning Representations}, 2019.

\bibitem[Ioffe \& Szegedy(2015)Ioffe and Szegedy]{Ioffe2015BatchNA}
Ioffe, S. and Szegedy, C.
\newblock Batch normalization: Accelerating deep network training by reducing
  internal covariate shift.
\newblock In \emph{International Conference on Machine Learning}, pp.\
  448--456, 2015.

\bibitem[Karcher(1977)]{karcher1977riemannian}
Karcher, H.
\newblock Riemannian center of mass and mollifier smoothing.
\newblock \emph{Communications on Pure and Applied Mathematics}, 30\penalty0
  (5):\penalty0 509--541, 1977.

\bibitem[Kleinberg(2007)]{Kleinberg2007GeographicRU}
Kleinberg, R.~D.
\newblock Geographic routing using hyperbolic space.
\newblock \emph{IEEE INFOCOM 2007 - 26th IEEE International Conference on
  Computer Communications}, pp.\  1902--1909, 2007.

\bibitem[Law et~al.(2019)Law, Liao, Snell, and Zemel]{Law2019LorentzianDL}
Law, M., Liao, R., Snell, J., and Zemel, R.
\newblock Lorentzian distance learning for hyperbolic representations.
\newblock In \emph{International Conference on Machine Learning}, pp.\
  3672--3681, 2019.

\bibitem[Lee(1997)]{lee1997riemannian}
Lee, J.
\newblock \emph{Riemannian Manifolds: An Introduction to Curvature}.
\newblock Graduate Texts in Mathematics. Springer New York, 1997.

\bibitem[Lee(2003)]{lee2003introduction}
Lee, J.~M.
\newblock Introduction to smooth manifolds.
\newblock \emph{Graduate Texts in Mathematics}, 218, 2003.

\bibitem[Liu et~al.(2019)Liu, Nickel, and Kiela]{liu2019hyperbolic}
Liu, Q., Nickel, M., and Kiela, D.
\newblock Hyperbolic graph neural networks.
\newblock In \emph{Advances in Neural Information Processing Systems}, pp.\
  8228--8239, 2019.

\bibitem[Nickel \& Kiela(2017)Nickel and Kiela]{Nickel2017PoincarEF}
Nickel, M. and Kiela, D.
\newblock Poincar{\'e} embeddings for learning hierarchical representations.
\newblock In \emph{Advances in Neural Information Processing Systems}, pp.\
  6338--6347, 2017.

\bibitem[Nickel \& Kiela(2018)Nickel and Kiela]{Nickel2018LearningCH}
Nickel, M. and Kiela, D.
\newblock Learning continuous hierarchies in the {L}orentz model of hyperbolic
  geometry.
\newblock In \emph{Proceedings of the 35th International Conference on Machine
  Learning}, pp.\  3779--3788, 2018.

\bibitem[Pascanu et~al.(2013)Pascanu, Mikolov, and Bengio]{Pascanu2012OnTD}
Pascanu, R., Mikolov, T., and Bengio, Y.
\newblock On the difficulty of training recurrent neural networks.
\newblock In \emph{International Conference on Machine Learning}, pp.\
  1310--1318, 2013.

\bibitem[Pogan{\v{c}}i{\'c} et~al.(2020)Pogan{\v{c}}i{\'c}, Paulus, Musil,
  Martius, and Rolinek]{pogan2020differentiation}
Pogan{\v{c}}i{\'c}, M.~V., Paulus, A., Musil, V., Martius, G., and Rolinek, M.
\newblock Differentiation of blackbox combinatorial solvers.
\newblock In \emph{International Conference on Learning Representations}, 2020.

\bibitem[Sala et~al.(2018)Sala, Sa, Gu, and R{\'e}]{Sala2018RepresentationTF}
Sala, F., Sa, C.~D., Gu, A., and R{\'e}, C.
\newblock Representation tradeoffs for hyperbolic embeddings.
\newblock \emph{Proceedings of Machine Learning Research}, 80:\penalty0
  4460--4469, 2018.

\bibitem[Santurkar et~al.(2018)Santurkar, Tsipras, Ilyas, and
  Madry]{Santurkar2018HowDB}
Santurkar, S., Tsipras, D., Ilyas, A., and Madry, A.
\newblock How does batch normalization help optimization?
\newblock In \emph{Advances in Neural Information Processing Systems}, pp.\
  2483--2493, 2018.

\bibitem[Sarkar(2011)]{SarkarLowDistortionDelauneyEmbeddings}
Sarkar, R.
\newblock Low distortion delaunay embedding of trees in hyperbolic plane.
\newblock In \emph{Proceedings of the 19th International Conference on Graph
  Drawing}, GD'11, pp.\  355–366, Berlin, Heidelberg, 2011. Springer-Verlag.

\bibitem[Sen et~al.(2008)Sen, Namata, Bilgic, Getoor, Gallagher, and
  Eliassi-Rad]{Sen2008CollectiveCI}
Sen, P., Namata, G., Bilgic, M., Getoor, L., Gallagher, B., and Eliassi-Rad, T.
\newblock Collective classification in network data.
\newblock \emph{AI Magazine}, 29:\penalty0 93--106, 2008.

\bibitem[Srivastava et~al.(2014)Srivastava, Hinton, Krizhevsky, Sutskever, and
  Salakhutdinov]{srivastava14a}
Srivastava, N., Hinton, G., Krizhevsky, A., Sutskever, I., and Salakhutdinov,
  R.
\newblock Dropout: A simple way to prevent neural networks from overfitting.
\newblock \emph{Journal of Machine Learning Research}, 15:\penalty0 1929--1958,
  2014.

\bibitem[Tifrea et~al.(2019)Tifrea, Becigneul, and Ganea]{Tifrea2018PoincarGH}
Tifrea, A., Becigneul, G., and Ganea, O.-E.
\newblock Poincar{\'e} glove: Hyperbolic word embeddings.
\newblock In \emph{International Conference on Learning Representations}, 2019.

\bibitem[Udri\c{s}te(1994)]{Udriste1997}
Udri\c{s}te, C.
\newblock \emph{Convex functions and optimization methods on Riemannian
  manifolds}.
\newblock Mathematics and Its Applications. Springer, Dordrecht, 1994.
\newblock \doi{10.1007/978-94-015-8390-9}.

\bibitem[Ungar(2008)]{Ungar2009AGS}
Ungar, A.~A.
\newblock A gyrovector space approach to hyperbolic geometry.
\newblock \emph{Synthesis Lectures on Mathematics and Statistics}, 1\penalty0
  (1):\penalty0 1--194, 2008.

\bibitem[Vaswani et~al.(2017)Vaswani, Shazeer, Parmar, Uszkoreit, Jones, Gomez,
  Kaiser, and Polosukhin]{Vaswani2017AttentionIA}
Vaswani, A., Shazeer, N., Parmar, N., Uszkoreit, J., Jones, L., Gomez, A.~N.,
  Kaiser, L., and Polosukhin, I.
\newblock Attention is all you need.
\newblock \emph{ArXiv}, abs/1706.03762, 2017.

\bibitem[Veli\v{c}kovi\'{c} et~al.(2018)Veli\v{c}kovi\'{c}, Cucurull, Casanova,
  Romero, Li\`o, and Bengio]{velickovic2018graph}
Veli\v{c}kovi\'{c}, P., Cucurull, G., Casanova, A., Romero, A., Li\`o, P., and
  Bengio, Y.
\newblock Graph attention networks.
\newblock In \emph{International Conference on Learning Representations}, 2018.

\bibitem[Wolter(1979)]{wolter1979}
Wolter, F.-E.
\newblock Distance function and cut loci on a complete riemannian manifold.
\newblock \emph{Archiv der Mathematik}, 32\penalty0 (1):\penalty0 92--96, 1979.
\newblock \doi{10.1007/BF01238473}.

\bibitem[Yu \& De~Sa(2019)Yu and De~Sa]{Yu2019NumericallyAH}
Yu, T. and De~Sa, C.~M.
\newblock Numerically accurate hyperbolic embeddings using tiling-based models.
\newblock In \emph{Advances in Neural Information Processing Systems}, pp.\
  2021--2031, 2019.

\end{thebibliography}
\bibliographystyle{icml2020}
\newpage

\onecolumn
\appendix








\section{Proof of Correctness of Fr\'echet Mean as Generalization of Euclidean Mean}\label{appendix:Gen}

In this section, we show that the Fr\'echet mean and variance are natural generalizations of Euclidean mean and variance. 

\begin{prop}\label{prop:appendixFrechetEquiv}
    On the manifold $\mathcal{M}=\mathbb{R}^{n}$, equations (\ref{frechetMean}) and (\ref{frechetVar}) are equivalent to the Euclidean mean and variance.
\end{prop}
\begin{proof}
Expanding the optimization function gives:
\begin{align*}
    \frac{1}{t}\sum\limits_{l=1}^{t}d(x^{(l)},\mu)^{2} &= \frac{1}{t}\sum\limits_{l=1}^{t}\norm{\mu - x^{(l)}}_{2}^{2} = \frac{1}{t}\sum\limits_{l=1}^{t}\sum\limits_{i=1}^{n}(\mu_{i}-x_{i}^{(l)})^{2} \\
    &=\frac{1}{t}\sum\limits_{l=1}^{t}\left[\sum\limits_{i=1}^{n}\mu_{i}^{2}-2\sum\limits_{i=1}^{n}\mu_{i}x_{i}^{(l)}+\sum\limits_{i=1}^{n}(x_{i}^{(l)})^{2}\right] \\
    &=\sum\limits_{i=1}^{n}\mu_{i}^{2}-\sum\limits_{i=1}^{n}\frac{2}{t}\left(\sum\limits_{l=1}^{t}x_{i}^{(l)}\right)\mu_{i}+\frac{1}{t}\sum\limits_{i=1}^{n}\sum\limits_{l=1}^{t}(x_{i}^{(l)})^{2} \\
    &=\sum\limits_{i=1}^{n}\left(\mu_{i}-\left(\frac{1}{t}\sum\limits_{l=1}^{t}x_{i}^{(l)}\right)\right)^{2}+\sum\limits_{i=1}^{n}\left(\frac{1}{t}\sum\limits_{l=1}^{t}(x_{i}^{(l)})^{2}-\left(\frac{1}{t}\sum\limits_{l=1}^{t}x_{i}^{(l)}\right)^{2}\right)
\end{align*}
Thus, optimizing the above quadratic function in $\mu$ gives (by a simple gradient computation):
\begin{equation*}
    \argmin_{\mu\in \mathbb{R}^{n}}\frac{1}{t}\sum\limits_{l=1}^{t}d(x^{(l)},\mu)^{2}=\frac{1}{t}\sum\limits_{l=1}^{t}x^{(l)}
\end{equation*}
\begin{equation*}
    \min_{\mu\in \mathbb{R}^{n}}\frac{1}{t}\sum\limits_{l=1}^{t}d(x^{(l)},\mu)^{2}=\sum_{i = 1}^n \paren{\frac{1}{t}\sum\limits_{l=1}^{t} (x_{i}^{(l)})^{2}-\left(\frac{1}{t}\sum\limits_{l=1}^{t}x_{i}^{(l)}\right)^{2}}
\end{equation*}
We note that these are the mean and variance function in the standard Euclidean sense (where the total variance is the sum of variances on each coordinate).
\end{proof}

\section{General Theorems on Differentiating through the Argmin}\label{appendix:diffargmin}

In this section, we provide generalizations of theorems in \citet{Gould2016OnDP} that will be useful in our gradient derivations when differentiating through the Fr\'echet mean.

We first consider the case of differentiating through an unconstrained optimization problem. This is a generalization of Lemma 3.2 in \citet{Gould2016OnDP}. Note that again we use $\widetilde{\nabla}_{x^{(i)}}$ to represent the total derivative.

\begin{thm}\label{thm:diffthroughunconstrained}
    Let $f: \R^n \times \R^m \to \R$ be a twice differentiable function. Let $g : \R^n \to \R^m$ be given by $g(x) = \argmin_{y \in \R^m} f(x, y)$. Then
    
    \begin{equation}\label{eqn:diffthroughunconstrained}
        \widetilde{\nabla}_x g(x) = -\nabla_{yy}^2f(x, g(x))^{-1} \widetilde{\nabla}_x \nabla_y f(x, g(x))
    \end{equation}
\end{thm}

\begin{proof}
    From our definition of the optimization problem, we know that
    
    \begin{equation*}
        \nabla_y f(x, y)\Big|_{y=g(x)} = 0
    \end{equation*}
    
    Taking the derivative with respect to $x$ gives 
    \begin{align*}
        0 &= \jacob_{x}\left(\nabla_{y}f(x, g(x))\right)\nonumber \\
        &= \jacob_{x}\nabla_{y}f(x, g(x))\cdot \jacob_{x}(x)+ \nabla_{yy}^{2}f(x, g(x))\cdot \jacob_{x}g(x) \nonumber\\
        &=\jacob_{x}\nabla_{y}f(x, g(x))+ \nabla_{yy}^{2}f(x, g(x))\cdot \jacob_{x}g(x)
    \end{align*}
    
    and rearranging gives the desired result
    \begin{equation*}
        \jacob_x g(x) = -\nabla_{yy}^2f(x, g(x))^{-1} \jacob_x \grad_y f(x, g(x))
    \end{equation*}
\end{proof}

We then consider the case of differentiating through a constrained optimization problem. This is a generalization of Lemma 4.2 in \citet{Gould2016OnDP}.

\begin{thm}\label{thm:diffthroughconstrained}
Let $f:\mathbb{R}^{n}\times \mathbb{R}^{n}\rightarrow \mathbb{R}$ be continuous with second derivatives. Let $A\in \mathbb{R}^{m\times n}$, $b\in \mathbb{R}^{m}$ with $\rank(A)=m$. Let $g:\mathbb{R}^{n}\rightarrow \mathbb{R}^{n}$ be defined by $g(x)=\argmin\limits_{y\in \mathbb{R}^{n}:Ay=b}f(x, y)$, and let $H=\nabla_{YY}^{2}f(x, g(x))\in \mathbb{R}^{n\times n}$, then we have 
\begin{align}
    \jacob_{x} g(x)=\left(H^{-1}A^{\top}(AH^{-1}A^{\top})^{-1}AH^{-1}-H^{-1}\right)\nabla_{XY}^{2}f(x, g(x)) 
\end{align}
where $\nabla_{XY}^{2}f(x, y)=\jacob_{x}\nabla_{y}f(x, y)$, and $\nabla_{YY}^{2}f(x, y)=\nabla_{yy}^{2}f(x, y)$.
\end{thm}
\begin{proof}
The proof is essentially the same as Gould's proof, and the only thing we need to be careful about is to ensure that the dimension of all quantities still make sense when we pass the partial derivative with respect to $x\in \mathbb{R}$ into the gradient with respect to $x\in \mathbb{R}^{n}$. We will carefully reproduce the steps in Gould's proof and make modifications if needed:

\vspace{10pt}

(1) \textbf{Formulating the Lagrangian}: The optimization problem that we are trying to solve is $g(x)=\argmin\limits_{y\in \mathbb{R}^{n}: Ay=b}f(x, y)$. 
We formulate its Lagrangian to be $L(x, y, \lambda)=f(x, y)+\lambda^{\top}(Ay-b)$. 

\vspace{10pt}

Let $\Tilde{g}(x)=(y^{*}(x), \lambda^{*}(x))$ be the optimal primal-dual pair, and write $\jacob_{x}\Tilde{g}(x)=(\jacob_{x}y^{*}(x), \nabla_{x}\lambda^{*}(x))=(\Tilde{g}_{Y}(x), \Tilde{g}_{\Lambda}(x))$. Note that we have $\Tilde{g}_{Y}(x)\in \mathbb{R}^{n\times n}$ and $\Tilde{g}_{\Lambda}(x)\in \mathbb{R}^{m\times n}$.

\vspace{10pt}

(2) \textbf{Derivative conditions from the Lagrangian}: From choice of optimal points $(y^{*}(x), \lambda^{*}(x))$, we have 
\begin{equation*}
    \begin{cases}
        \nabla_{y}L(x, y, \lambda)=0 \\
        \nabla_{\lambda}L(x, y, \lambda)=0
    \end{cases}\Rightarrow  \begin{cases}
        \nabla_{Y}f(x, y^{*}(x))+A^{\top}\lambda^{*}(x)=0 \\
        Ay^{*}(x)-b=0
    \end{cases}
\end{equation*}
We note that the first equation has both sides in $\mathbb{R}^{n}$, and the second equation has both sides in $\mathbb{R}^{m}$. 

\vspace{10pt}

Now we take the Jacobian $\jacob_{x}$ for both equations\footnote{Note that $\nabla_{x}$ (taking gradient over variable $x$) is different from $\nabla_{X}$ (taking the gradient over the first variable of the function).}. For the first equation, applying the chain rule will result in 
\begin{align*}
    0 &= \jacob_{x}\left(\nabla_{Y}f(x, y^{*}(x))+A^{\top}\lambda^{*}(x)\right)\nonumber \\
    &= \nabla_{XY}^{2}f(x, y^{*}(x))\cdot \jacob_{x}(x)+\nabla_{YY}^{2}f(x, y^{*}(x))\cdot \jacob_{x}(y^{*}(x))+A^{\top}\jacob_{x}(\lambda^{*}(x))\nonumber \\
    &= \nabla_{XY}^{2}f(x, y^{*}(x))+\nabla_{YY}^{2}f(x, y^{*}(x))\cdot \Tilde{g}_{Y}(x)+A^{\top}\Tilde{g}_{\Lambda}(x)
\end{align*}
For the second equation, this will result in \begin{equation*}
    0=\jacob_{x}\left(Ay^{*}(x)-b\right)=A\cdot \jacob_{x}(y^{*}(x))=A\Tilde{g}_{Y}(x)
\end{equation*}
The above two equations give the following system:
\begin{equation*}
    \begin{cases}
        \nabla_{XY}^{2}f(x, g(x))+\nabla_{YY}^{2}f(x, g(x))\cdot \Tilde{g}_{Y}(x)+A^{\top}\Tilde{g}_{\Lambda}(x)=0 \\
        A\Tilde{g}_{Y}(x)=0
    \end{cases}
\end{equation*}
where the first equation has both sides in $\mathbb{R}^{n\times n}$, and the second equation has both sides in $\mathbb{R}^{m\times n}$.

\vspace{10pt}

(3) \textbf{Computing the Jacobian matrix}: Now we will solve for $\Tilde{g}_{Y}(x)$ based on the above equations. We will denote $H=\nabla_{YY}^{2}f(x, g(x))$.

\vspace{10pt}

We first solve for $\Tilde{g}_{Y}(x)$ in the first equation:
\begin{equation*}
    \Tilde{g}_{Y}(x)=-H^{-1}\left(\nabla_{XY}^{2}f(x, g(x))+A^{\top}\Tilde{g}_{\Lambda}(x)\right)
\end{equation*}
We then substitute this value in the second equation to get 
\begin{equation*}
    0=A\left(-H^{-1}\left(\nabla_{XY}^{2}f(x, g(x))+A^{\top}\Tilde{g}_{\Lambda}(x)\right)\right)=-AH^{-1}\nabla_{XY}^{2}f(x, g(x))-AH^{-1}A^{\top}\Tilde{g}_{\Lambda}(x)
\end{equation*}
So we can solve for $\Tilde{g}_{\Lambda}(x)$:
\begin{equation*}
    \Tilde{g}_{\Lambda}(x)=-(AH^{-1}A^{\top})^{-1}AH^{-1}\nabla_{XY}^{2}f(x, g(x))
\end{equation*}
We finally plug this into the first equation again:
\begin{align*}
    \Tilde{g}_{Y}(x) &= -H^{-1}\left(\nabla_{XY}^{2}f(x, g(x))+A^{\top}\left(-(AH^{-1}A^{\top})^{-1}AH^{-1}\nabla_{XY}^{2}f(x, g(x))\right)\right) \\ 
    &=-H^{-1}\nabla_{XY}^{2}f(x, g(x))+H^{-1}A^{\top}(AH^{-1}A^{\top})^{-1}AH^{-1}\nabla_{XY}^{2}f(x, g(x)) \\
    &=\left(H^{-1}A^{\top}(AH^{-1}A^{\top})^{-1}AH^{-1}-H^{-1}\right)\nabla_{XY}^{2}f(x, g(x))
\end{align*}
From our definition of $\Tilde{g}_{Y}$, we know that $\jacob_{x}y^{*}(x)=\tilde{g}_{Y}(x)$, where $y^{*}(x)=g(x)$ is the optimal solution for the original optimization problem. Thus, we have 
\begin{equation*}
    \jacob_{x}g(x)=\left(H^{-1}A^{\top}(AH^{-1}A^{\top})^{-1}AH^{-1}-H^{-1}\right)\nabla_{XY}^{2}f(x, g(x))
\end{equation*}
which is what we want to show.
\end{proof}

\section{In-depth Differential Geometry Background}\label{appendix:detailedDG}

In this section we give a more formal introduction to differential geometry, which will be critical in the proof of and understanding of our differentiation theorem. Most of these definitions originate from \cite{lee2003introduction} or \cite{lee1997riemannian}.

\subsection{Additional Manifold Definitions}

\textbf{Manifolds:} An $n$-dimensional manifold $\M$ is a second countable Hausdorff space that is locally homeomorphic to $\R^n$. On open subsets of $\M$, we define a coordinate chart $(U, \varphi)$ where $\varphi: U \to \tilde{U} \subseteq \R^n$ is the homeomorphism.

\textbf{Smooth manifolds:} A manifold $\M$ with dimension $n$ is smooth if it has a smooth atlas, i.e. a collection of charts $(U, \varphi)$ such that for any two charts $(U, \varphi)$ and $(V, \psi)$, either $\psi \circ \varphi^{-1}$ is a diffeomorphism or $U \cap V = \emptyset$.

\textbf{Tangent spaces:} A tangent vector $v$ at a point $p \in \M$ is a linear map $v: C^\infty(M) \to \R$ which satisfies $v(fg) = f(p)vg + g(p)vf$. This linear map is also commonly called a derivation at $p$. The tangent space $T_pM$ is the collection of these derivations. The tangent space is isomorphic to $\R^n$ as a vector space and there exist bases $\parderiv{}{x^i}\rvert_p$.

\textbf{Coordinate systems:} Instead of writing out charts $(U, \varphi)$, we use local coordinates $(x^i)$ (which constitute maps from $U \to \R^n$ but provide a cleaner notation). We associate local coordinates $x^i$ with their induced tangent vectors $\parderiv{}{x^i}\rvert_p$, which form a basis for the tangent space.

\textbf{Pushforward:} A smooth map between manifolds $F: \M \to \mathcal{N}$ admits a push-forward $F_*: T_p\M \to T_{F(p)}\mathcal{N}$ given by $(F_*v)(f) = v(f \circ F)$. When the bases for $T_p\M$ and $T_{f(p)}\mathcal{N}$ are constructed by local coordinates, we can represent $F$ by $\tilde{F}$, which computes in these coordinates. Then $F_*$ corresponds to $\jacob \tilde{F}$. Note that since the pushfoward is defined without respect to local coordinates, when given local coordinates, we oftentimes use the above two formulations interchangeably as $\jacob F$.

\textbf{Differential:} The exterior derivative of a function $f: \M \to \R$ in terms of local coordinates $(x_1, \dots, x_n)$ is given by $df = \begin{bmatrix} \partial_1 f \\ \vdots \\ \partial_n f \end{bmatrix}$ where $\partial_i f = \parderiv{f}{x^i}$. This can be generalized to differential forms (in this case we only consider $0$-forms), but that is outside the scope of our paper. We will oftentimes just write this as $\nabla f$ to stress the matrix value version.

\subsection{Additional Riemannian Geometry Definitions}

\textbf{Riemannian gradient:} On a Riemannian manifold $(\M, \rho)$, the Riemannian gradient $\nabla^r$ of a function $f: \M \to \R$ is defined as $\nabla^{r}f=\rho^{-1} \circ \nabla f$, where $\rho^{-1}$ is taken as a matrix and the gradient is taken with respect to local coordinates.

\textbf{Geodesics:} Geodesics are formally given as curves which have zero acceleration w.r.t. the Levi-Civita Connection. This material is outside the scope of this paper, but it is important to note that geodesics can be non-minimizing.

\textbf{Conjugate points:} Conjugate points $(p, q)$ can be thought of as places where the geodesic between $p$ and $q$ is non-minimizing. The more formal definition involves Jacobi fields, where conjugate points are points where the nontrivial Joacbia field vanishes. However, this is outside the scope of this paper.

\textbf{Hadamard manifolds:} A Hadamard manifold is a manifold with everywhere non-positive sectional curvature. This space enables convex analysis/optimization since it is topologically trivial (similar to $\R^n$). Some canonical examples of Hadamard manifolds include Euclidean space $\mathbb{R}^{n}$, hyperbolic space $\mathbb{H}^{n}$, and the space of symmetric positive definite (SPD) matrices $\mathbb{S}_{+}^{n}$. A Hadamard manifold is geodesically complete and has no conjugate points.

\subsection{A Few Useful Lemmas}

\begin{lemma}
    Around each point $x \in \M$, $\exp_x$ is a local diffeomorphism.
\end{lemma}

\begin{proof}
    This is given in Lemma 5.10 of \cite{lee1997riemannian} as a consequence of the inverse function theorem.
\end{proof}

\begin{lemma}[Chain Rule]\label{lem:chainrulejacobian}
    Suppose $g: \M \to \mathcal{N}$ and $f: \mathcal{N} \to \mathcal{L}$ be smooth maps between manifolds. The $\widetilde{\nabla}(f \circ g)(x) = \widetilde{\nabla}f(g(x)) \circ \widetilde{\nabla} g(x)$.
\end{lemma}

\begin{proof}
    This follows directly from application of the standard chain rule.
\end{proof}

\begin{cor}
    We note that if $g: \M \to \mathcal{N}$ and $f: \mathcal{N} \to \R$ be smooth maps between Riemannian manifolds, with $\M$ having metric $\rho$, then
    
    \begin{equation}\label{eqn:riemannianGradientChainRule}
       \nabla_x^r (f \circ g)(x) = \rho_x^{-1} \circ (\widetilde{\nabla} g(x))^\top \circ \nabla f(g(x))
    \end{equation}
\end{cor}

\begin{proof}
    This is a direct consequence of Lemma \ref{lem:chainrulejacobian} above and definition of Riemannian gradient.
\end{proof}

\section{Differentiating Argmin on Manifolds}\label{appendix:diffManifold}

We now extend the results in Theorem \ref{thm:diffthroughunconstrained} to product of Riemannian manifolds.

\begin{thm}\label{thm:diffthroughmanifoldgeneral}
    Let $(\M, \rho)$ be an $m$-dimensional Riemannian manifold, $(\mathcal{N}, \phi)$ be an $n$-dimensional Riemannian manifold, and $f : \M \times \mathcal{N} \to \R$ be twice differentiable. Let $g(x) = \argmin_{y \in \mathcal{N}} f(x, y)$. With respect to local coordinates on $\M$ and $\mathcal{N}$, we have 
    \begin{equation}\label{eqn:manifoldgradientgeneral}
        \widetilde{\nabla} g(x) = -\nabla_{yy} f(x, g(x))^{-1} \circ \widetilde{\nabla}_x \nabla_y f(x, g(x))
    \end{equation}
    In particular, if $\mathcal{L}: \M \to \R$ is some differentiable function, then we can calculate the Riemannian gradient as
    \begin{equation}\label{eqn:manifoldrgradient}
        \nabla_x^r (\mathcal{L} \circ g)(x) = - \rho_x^{-1} \circ (\widetilde{\nabla}_x \nabla_y f(x, g(x)))^\top \circ \nabla_{yy} f(x, g(x))^{-1} \circ \nabla \mathcal{L}(g(x))
    \end{equation}
\end{thm}


\begin{proof}
    (1) \textbf{We first show the validity of equation (\ref{eqn:manifoldgradientgeneral}).} Note that this is the same expression as equation (\ref{eqn:diffthroughunconstrained}), so all that remains to be proven is that we can apply Theorem \ref{thm:diffthroughunconstrained}. 
    
    The proof for Theorem \ref{thm:diffthroughunconstrained} only depends on the gradient of $f$ being $0$. We know that $g(x)$ is a global minimum, so for any chart on $\mathcal{N}$ around $g(x)$, $g(x)$ is a local minimum. Therefore, we can take local coordinates and apply Theorem \ref{thm:diffthroughunconstrained} on these coordinates to obtain our result.
    
    \vspace{10pt}
    
    (2) \textbf{We then show the validity of equation (\ref{eqn:manifoldrgradient}) from the definition of Riemannian gradient.} This follows immediately from the definition of Riemannian gradient:
    \begin{align*}
        \nabla_x^r (\mathcal{L} \circ g)(x) &= \rho_x^{-1} \circ \widetilde{\nabla}_x g(x)^\top \circ \nabla \mathcal{L}(g(x))\\
        &= \rho_{x}^{-1} \circ (-\nabla_{yy} f(x, g(x))^{-1} \circ \widetilde{\nabla}_x \nabla_y f(x, g(x)))^\top \circ \nabla \mathcal{L}(g(x))\\
        &= - \rho_{x}^{-1} \circ (\widetilde{\nabla}_x \nabla_y f(x, g(x)))^\top \circ (\nabla_{yy} f(x, g(x))^{-1})^\top \circ \nabla \mathcal{L}(g(x))\\
        &= -\rho_{x}^{-1} \circ (\widetilde{\nabla}_x \nabla_y f(x, g(x)))^\top \circ \nabla_{yy} f(x, g(x))^{-1} \circ \nabla \mathcal{L}(g(x))
    \end{align*}
    where we note that $\nabla_{yy}f(x, g(x))^{-1}$ is symmetric (since the Jacobian is symmetric).
\end{proof}

\begin{remark}
    While the above theorem is sufficient for establishing the existence of a local computation of gradient of an argmin on a manifold, it is not conducive for actual computation due to difficulty in obtaining a local coordinate chart. For practical purposes, we present two more computationally tractable versions below, one of which relies on an exponential map reparameterization inspired by \citet{LezcanoCasado2019TrivializationsFG} and one for when the manifold exists in ambient Euclidean space.
\end{remark}

\begin{thm}\label{thm:diffthroughmanifoldexp}
    Let $(\M, \rho)$ be an $m$-dimensional Riemannian manifold, $(\mathcal{N}, \phi)$ be an $n$-dimensional Riemannian manifold, and $f:\M \times \mathcal{N} \to \R$ be a twice differentiable function. Let $g(x) = \argmin_{y \in \mathcal{N}} f(x, y)$. Suppose we identify each tangent space $T_x\M$ with $\R^m$ and $T_y\mathcal{N}$ with $\R^n$ for all $x \in \M, y \in \mathcal{N}$. Fix $x', y'$ as inputs and let $y' = g(x')$. Construct $\widehat{f} : \R^m \times \R^n \to \R$ to be given by $\widehat{f}(x, y) = f(\exp_{x'} x, \exp_{y'} y)$. Then
    \begin{equation}
        \widetilde{\nabla} g(x') = - \widetilde{\nabla} \exp_{y'}(\bvect{0}) \circ \nabla_{yy}^2 \widehat{f}(\bvect{0}, \bvect{0})^{-1} \circ \widetilde{\nabla}_x \nabla_y \widehat{f}(\bvect{0}, \bvect{0}) \circ \widetilde{\nabla} \log_{x'}(x')
    \end{equation}
    where $\log_{x'}$ is a local inverse of $\exp_{x'}$ around $\bvect{0}$.
\end{thm}

\begin{proof}
    Define $\widehat{g}(x) = \argmin_{y \in \R^n} \widehat{f}(x, y)$. Then we see that, locally, $g = \exp_{y'} \circ \widehat{g} \circ \log_{x'}$. Applying chain rule gives us
    \begin{equation*}
        \widetilde{\nabla} g(x') = \widetilde{\nabla} \exp_{y'}(\bvect{0}) \circ \widetilde{\nabla} \widehat{g}(\bvect{0}) \circ \widetilde{\nabla} \log_{x'}(x')
    \end{equation*}
    where we plugged in $\log_{x'} x' = \bvect{0}$. We can apply Theorem \ref{thm:diffthroughunconstrained} because $\widehat{g} : \R^m \to \R^n$ and since $x'$ is an argmin then $0$ is a local argmin. We substitute
    \begin{equation*}
        \widetilde{\nabla} \widehat{g}(\bvect{0}) = - \nabla_{yy}^2 \widehat{f}(\bvect{0}, \bvect{0})^{-1} \circ \widetilde{\nabla}_x \nabla_y \widehat{f}(\bvect{0}, \bvect{0})
    \end{equation*}
    gives us the desired result.
\end{proof}

\begin{thm}\label{thm:diffthroughmanifoldembedded}
    Let $\M$ be an $m$-dimensional manifold and $\mathcal{N}$ be an $n$-dimensional manifold. Suppose $\M$ is embedded in $\R^M$ and $\mathcal{N}$ is embedded in $\R^N$ where $M > m$ and $N > n$. Let $f: \M \times \mathcal{N} \to \R$ be a twice differentiable function and $g(x) = \argmin_{y \in \mathcal{N}} f(x, y)$. We have 
    \begin{equation}\label{eqn:embeddedmanifoldgradgeneral}
        \widetilde{\nabla} g(x) = -\proj_{T_x\M}\paren{\paren{\widetilde{\nabla}_y^{euc}\circ \proj_{T_{g(x)}\mathcal{N}} \circ\nabla_y^{euc} f(x, g(x))}^{-1}\paren{\widetilde{\nabla}^{euc}_x\circ \proj_{T_{g(x)}\mathcal{N}}\circ \nabla_y^{euc} f(x, g(x))}}
    \end{equation}
    where $\widetilde{\nabla}^{euc}$ is the total derivative w.r.t. the ambient (Euclidean) space.
\end{thm}

\begin{proof}
    \textbf{We will reproduce the steps in the proof of Theorem \ref{thm:diffthroughunconstrained} for this general setting.} Note that
    \begin{equation*}
        \proj_{T_{g(x)} \mathcal{N}}\circ\nabla_y^{euc} f(x, g(x)) = 0
    \end{equation*}
    as the gradient in the tangent space $T_{g(x)} \mathcal{N}$ is $0$. By taking the total derivative we know that
    \begin{align*}
        0 &= \widetilde{\nabla}_x^{euc}\paren{\proj_{T_{g(x)} \mathcal{N}}\circ\nabla_y^{euc} f(x, g(x))}\\
        &= \widetilde{\nabla}_x^{euc}\circ \proj_{T_{g(x)} \mathcal{N}}\circ\nabla_y^{euc} f(x, g(x)) \cdot \widetilde{\nabla}_x^{euc}(x) + \widetilde{\nabla}_y^{euc}\circ \proj_{T_{g(x)} \mathcal{N}}\circ\nabla_y^{euc} f(x, g(x)) \cdot \widetilde{\nabla}_x^{euc} g(x)\\
        &= \widetilde{\nabla}_x^{euc}\circ \proj_{T_{g(x)} \mathcal{N}}\circ\nabla_y^{euc} f(x, g(x)) + \widetilde{\nabla}_y^{euc}\circ \proj_{T_{g(x)} \mathcal{N}}\circ\nabla_y^{euc} f(x, g(x)) \cdot \widetilde{\nabla}_x^{euc} g(x)
    \end{align*}
    Rearranging gives us
    \begin{equation*}
        \widetilde{\nabla}_{x}^{euc} g(x) = -\paren{\widetilde{\nabla}_y^{euc}\circ \proj_{T_{g(x)}\mathcal{N}}\circ \nabla_y^{euc} f(x, g(x))}^{-1}\paren{\widetilde{\nabla}^{euc}_x\circ \proj_{T_{g(x)}\mathcal{N}}\circ \nabla_y^{euc} f(x, g(x))} 
    \end{equation*}
    and to obtain the derivative in the tangent space, we simply project to $T_{g(x)}\M$.
\end{proof}

\begin{remark}[Differentiating through the Fr\'echet Mean]\label{remark:remarkfrechetmean}
     Here we examine the formulas given in Theorems \ref{thm:diffFrechetMean} and \ref{thm:diffFrechetMeanEmbedded}. The Fr\'echet mean objective function is not differentiable in general, as this requires the manifold to be diffeomorphic to $\R^n$ \cite{wolter1979}. In the main paper, this complex formulation is sidestepped (as we define geodesics to be the unique curves that minimize length). Here, we discuss the problem of conjugate points and potential difficulties.
     
     \vspace{10pt}
     
     We note that a sufficient condition for differentiability of the Fr\'echet mean is that its objective function is smooth, which occurs precisely when the squared Riemannian distance $d(x, y)^2$ is smooth. The conditions for this are that the manifold is geodesically complete and $x$ and $y$ are not conjugate points. This means that on Hadamard spaces, we can differentiate the Fr\'echet mean everywhere. Similarly, in cases when the manifold is not geodesically complete, such as when the manifold is not connected, we can still differentiate in places where the Fr\'echet mean is well defined. Finally, in the canonical case of spherical geometry, the set of conjugate points is a set of measure $0$, meaning we can differentiate almost everywhere. Hence we see that in most practical settings the difficulties of non-differentiability do not arise.
\end{remark}

\section{Derivations of Algorithms to Compute the Fr\'echet Mean in Hyperbolic Space}

\subsection{Derivation of Algorithm for Hyperboloid Model Fr\'echet Mean Computation}
\label{sec:hyperboloidforwardderiv}

In this section, we derive an algorithm that is guaranteed to converge to the Fr\'echet mean of points in the hyperboloid model.

\begin{algorithm}[!htb]
\caption{\label{alg:hyperboloidforward} Hyperboloid model Fr\'echet mean algorithm}
\textbf{Inputs}: Data points $x^{(1)}, \cdots, x^{(t)}\in \mathbb{H}_{K}^{n}\subseteq \mathbb{R}^{n+1}$, weights $w_{1},\cdots, w_{t}\in \mathbb{R}$.

\textbf{Algorithm}: 

$y_{0}=x^{(1)}$

for $k=0, 1, \cdots, T$:

\hspace{20pt} $u_{k+1} 
    = \sum\limits_{l=1}^{t} \left(w_{l}\cdot \frac{2 \arccosh\left(-|K|(x^{(l)})^\top M y_{k}\right)}{\sqrt{\left(-|K|(x^{(l)})^\top M y_{k}\right)^2 - 1}}\cdot x^{(l)} \right)$
    
\hspace{20pt} $y_{k+1} = \frac{u_{k+1}}{\sqrt{-|K|u_{k+1}^\top M u_{k+1}}}$

return $y_{T}$
\end{algorithm}

We now prove that this algorithm indeed converges to the Fr\'echet mean for points in the Lorentz model. For convenience of notation, let $M$ be the hyperbolic metric tensor, i.e. $M = \begin{bmatrix} -1 & 0 & \cdots & 0 \\ 0 & 1 & \cdots & 0 \\ 0 & 0 & \ddots & 0 \\ 0 & 0 & 0 & 1 \end{bmatrix}$.
\begin{thm}\label{thm:lorentzforward}
Let $x^{(1)}, \cdots, x^{(t)}\in \mathbb{H}_{K}^{n}\subseteq \mathbb{R}^{n+1}$ be $t$ points in the hyperboloid space, $w_1, \dots, w_t \in \R^{+}$ be their weights, and let their weighted Fr\'echet mean be the solution to the following optimization problem
\begin{equation}
    y^{*}=\argmin_{y \in \mathbb{H}_{K}^{n}}\sum_{l=1}^{t}w_{l}\cdot d_{\mathbb{H}_{K}^{n}}(x^{(l)}, y)^{2}=\argmin_{y \in \mathbb{H}_{K}^{n}}\sum_{l=1}^{t}\frac{1}{|K|}w_{l}\cdot \arccosh^{2}(-|K|(x^{(l)})^\top M y)
\end{equation}
Then Algorithm \ref{alg:hyperboloidforward} gives a sequence of points $\{y_{k}\}$ such that their limit $\lim\limits_{k\rightarrow\infty}y_{k}=y^{*}$ converges to the Fr\'echet mean solution.
\end{thm}

\begin{proof}
(1) \textbf{Apply concavity of $\arccosh$ to give an upper bound on the objective function}: Since $f(x)=\arccosh^{2}(x)$ is concave, its graph lies below the tangent, so we have
\begin{align*}
    \arccosh^{2}(x) &\leq  \arccosh^{2}(y)+(\arccosh^{2}(y))'(x-y) \nonumber \\
    &= \arccosh^{2}(y)+(x-y)\frac{2\arccosh(y)}{\sqrt{y^{2}-1}}
\end{align*}
Let us denote \begin{equation*}
    g(y)=\sum_{l=1}^{t}w_{l}\cdot \arccosh^{2}(-|K|(x^{(l)})^{\top} M y)
\end{equation*}
to be the objective function. Applying our concavity property to this objective function with respect to some fixed $y_{k}$ at iteration $k$ gives
\begin{align*}
    g(y) &= \sum_{l=1}^{t} w_{l}\cdot \arccosh^{2}(-|K|(x^{(l)})^\top M y)\nonumber \\
    &\le \sum_{l=1}^{t}w_{l}\cdot  \left( \arccosh^2\left(-|K|(x^{(l)})^\top M y_{k}\right) + |K|\left((-(x^{(l)})^\top M y) - (-(x^{(l)})^\top M y_{k})\right) \cdot \frac{2 \arccosh(-|K|(x^{(l)})^\top M y_{k})}{\sqrt{\left(-|K|(x^{(l)})^\top M y_{k}\right)^2 - 1}} \right)\nonumber \\
    &= g(y_{k}) + \sum_{l=1}^{t} w_{l}\cdot |K|(x^{(l)})^\top M (y_{k} - y) \cdot \frac{2 \arccosh(-|K|(x^{(l)})^\top M y_{k})}{\sqrt{\left(-|K|(x^{(l)})^\top M y_{k}\right)^2 - 1}}
\end{align*}

(2) \textbf{Finding solution to the minimization problem of upper bound}: Now consider the following minimization problem:
\begin{equation}\label{eqn:hyperboloidforwardintermediate}
y_{k+1}^{*}=\argmin_{y\in \mathbb{H}_{K}^{n}}\left(g(y_{k}) + \sum_{l=1}^{t} w_{l}\cdot |K|(x^{(l)})^\top M (y_{k} - y) \cdot \frac{2 \arccosh(-|K|(x^{(l)})^\top M y_{k})}{\sqrt{\left(-|K|(x^{(l)})^\top M y_{k}\right)^2 - 1}}\right)
\end{equation}
We will show that the solution of this optimization problem satisfies the computation in the algorithm:
\begin{equation*}
    y_{k+1}^{*}=\frac{u_{k+1}}{\sqrt{-u_{k+1}^\top M u_{k+1}}} \text{ with }
    u_{k+1} 
    = \sum\limits_{l=1}^{t} \left(w_{l}\cdot \frac{2 \arccosh(-|K|(x^{(l)})^\top M y_{k})}{\sqrt{\left(-|K|(x^{(l)})^\top M y_{k}\right)^2 - 1}}\cdot x^{(l)} \right)
\end{equation*}

(2-1) We first note that we can remove the terms that don't depend on $y$ in the optimization problem:
\begin{align}\label{eqn:SimplifiedOptHyperboloidForward}
    &\argmin_{y \in \mathbb{H}_{K}^{n}} \left(g(y_{k}) + \sum_{l=1}^{t} \left(w_{l}\cdot |K|(x^{(l)})^\top M (y_{k} - y) \cdot \frac{2 \arccosh(-|K|(x^{(l)})^\top M y_{k})}{\sqrt{\left(-|K|(x^{(l)})^\top M y_{k}\right)^2 - 1}}\right)\right) \nonumber\\
    &=
    \argmin_{y \in \mathbb{H}_{K}^{n}} \left( -y^\top M\cdot \sum_{l=1}^{t} \left(w_{l} \cdot \frac{2 \arccosh(-|K|(x^{(l)})^\top M y_{k})}{\sqrt{\left(-|K|(x^{(l)})^\top M y_{k}\right)^2 - 1}}\cdot x^{(l)}\right)\right)
\end{align}

(2-2) We now propose a general method of solving optimization problems in the form (2-1).

For any $u$ such that $u^\top M u < 0$, we have
\begin{align*}
    \argmin_{y \in \mathbb{H}_{K}^{n}}\left(-y^\top M u\right)
    &=\argmin_{y \in \mathbb{H}_{K}^{n}} \frac{-y^\top M u}{\sqrt{-|K|u^\top M u}}=\argmin_{y \in \mathbb{H}_{K}^{n}} \left(\arccosh\left(-y^\top M \frac{u}{\sqrt{-|K|u^\top M u}} \right)\right)\nonumber \\
    &=\argmin_{y \in \mathbb{H}_{K}^{n}} \left(d_{\mathbb{H}_{K}^{n}}\left(y, \frac{u}{\sqrt{-|K|u^\top M u}} \right)\right)=\frac{u}{\sqrt{-|K|u^\top M u}}
\end{align*}

(2-3) Specifically in the $k$-th iteration, we can define
\begin{equation}\label{eqn:UpdateRuleHyperboloidForward}
    u_{k+1}=\sum_{l=1}^{t} \left(w_{l}\cdot \frac{2 \arccosh(-|K|(x^{(l)})^\top M y_{k})}{\sqrt{\left(-|K|(x^{(l)})^\top M y_{k}\right)^2 - 1}}\cdot x^{(l)}\right)
\end{equation}
In order to apply the statement in (2-2), we now show that $u_{k+1}$ satisfies $u_{k+1}^{\top}Mu_{k+1}<0$. We can expand this as a sum: \begin{equation*}
    u_{k+1}^{\top}Mu_{k+1}=\sum\limits_{l=1}^{t}\sum\limits_{j=1}^{t}\left(w_{l}\cdot \frac{2 \arccosh(-|K|(x^{(l)})^\top M y_{k})}{\sqrt{\left(-|K|(x^{(l)})^\top M y_{k}\right)^2 - 1}}\right)\left(w_{j}\cdot \frac{2 \arccosh(-|K|(x^{(j)})^\top M y_{k})}{\sqrt{\left(-|K|(x^{(j)})^\top M y_{k}\right)^2 - 1}}\right)(x^{(l)})^{\top}Mx^{(j)}
\end{equation*}
We know that the two constant blocks are both greater than $0$. We also have $(x^{(l)})^{\top}Mx^{(l)}=-1$ strictly smaller than $0$. Thus, we only need to show that $(x^{(l)})^{\top}Mx^{(j)}\leq 0$ for $l \neq j$. 

We can expand \begin{equation*}
    \begin{cases}
        (x^{(l)})^{\top}Mx^{(l)}=-1 \\
        (x^{(j)})^{\top}Mx^{(j)}=-1
    \end{cases}\Rightarrow \begin{cases}
        (x_{0}^{(l)})^{2}=(x_{1}^{(l)})^{2}+\cdots +(x_{n}^{(l)})^{2}+1 \\
        (x_{0}^{(j)})^{2}=(x_{1}^{(j)})^{2}+\cdots +(x_{n}^{(j)})^{2}+1 \\
    \end{cases}
\end{equation*}
Multiplying them and applying Cauchy's inequality gives \begin{equation*}
    (x_{0}^{(l)}x_{0}^{(j)})^{2}=((x_{1}^{(l)})^{2}+\cdots +(x_{n}^{(l)})^{2}+1)((x_{1}^{(j)})^{2}+\cdots +(x_{n}^{(j)})^{2}+1)\geq (|x_{1}^{(l)}x_{1}^{(j)}|+\cdots +|x_{n}^{(l)}x_{n}^{(j)}|+1)^{2}
\end{equation*}
This implies \begin{equation*}
    x_{0}^{(l)}x_{0}^{(j)}=|x_{0}^{(l)}x_{0}^{(j)}|\geq |x_{1}^{(l)}x_{1}^{(j)}|+\cdots +|x_{n}^{(l)}x_{n}^{(j)}|+1\geq x_{1}^{(l)}x_{1}^{(j)}+\cdots +x_{n}^{(l)}x_{n}^{(j)}+1 
\end{equation*}
if we assume $x_{0}^{(l)}>0$ for all $1\leq l\leq n$ (i.e. we always pick points on the same connected component of the hyperboloid, either with $x_{0}$ always positive or always negative). Thus, we have $(x^{(l)})^\top Mx^{(j)}\leq -1<0$ for $l\neq j$ as well. This together our above results ensure that we have $u_{k+1}^{\top}Mu_{k+1}<0$.

\vspace{10pt}

(2-4) Since we have verified $u_{k+1}^{\top}Mu_{k+1}<0$ in (2-3), we know that we can apply the result in (2-2) to the optimization problem (\ref{eqn:SimplifiedOptHyperboloidForward}) to obtain
\begin{equation*}
    y_{k+1}^{*}=\argmin_{y \in \mathbb{H}_{K}^{n}} \left( -y^\top M\cdot \sum_{l=1}^{t} \left(w_{l} \cdot \frac{2 \arccosh(-|K|(x^{(l)})^\top M y_{k})}{\sqrt{\left(-|K|(x^{(l)})^\top M y_{k}\right)^2 - 1}}\cdot x^{(l)}
    \right)\right)
    =\frac{u_{k+1}}{\sqrt{-|K|u_{k+1}^\top M u_{k+1}}}
\end{equation*}
and this together with equation (\ref{eqn:UpdateRuleHyperboloidForward}) gives exactly the same process as in the algorithm. Thus, the algorithm generates the solution $y_{k+1}=y_{k+1}^{*}$ of this minimization problem. 

\vspace{10pt}

(3) \textbf{Proof of convergence}: We now show that the sequence $\{y_{k}\}$ converges to the Fr\'echet mean $y^{*}$ as $k\rightarrow\infty$.

To show this, we consider the objective function minimized in (\ref{eqn:hyperboloidforwardintermediate}): 
\begin{equation}
    y_{k+1}^{*}=\argmin_{y\in \mathbb{H}_{K}^{n}}h(y) \text{ , where } h(y)=g(y_{k}) + \sum_{l=1}^{t}w_{l}\cdot  |K|(x^{(l)})^\top M (y_{k} - y) \cdot \frac{2 \arccosh(-|K|(x^{(l)})^\top M y_{k})}{\sqrt{\left(-|K|(x^{(l)})^\top M y_{k}\right)^2 - 1}}
\end{equation}

We know that $h(y_{k})=g(y_{k})$ for $y_{k}\in \mathbb{H}_{K}^{n}$, so we must have $g(y_{k+1}^{*})\leq g(y_{k})$ with equality only if $y_{k}$ is already the minimizer of $g$ (in which case, by definition of $g$, we have already reached a Fr\'echet mean). Thus, if we are not at the Fr\'echet mean, we must have $g(y_{k+1})<g(y_{k})$ strictly decreasing. 

This means the sequence $\{y_{k}\}$ must converge to the Fr\'echet mean, as this is the only fixed point that we can converge to.

\end{proof}

\subsection{Derivation of Algorithm for Poincar\'e Model Fr\'echet Mean Computation}
\label{sec:poincareforwardderiv}

In this section, we derive an algorithm that is guaranteed to converge to the Fr\'echet mean of points in the Poincar\'e ball model. This is the same algorithm as Algorithm \ref{alg:poincareforward_mainpaper} in the main paper.


\begin{algorithm}[!htb]
\caption{\label{alg:poincareforward} Poincar\'e model Fr\'echet mean algorithm}
\textbf{Inputs}: $x^{(1)}, \cdots, x^{(t)}\in \mathbb{D}_{K}^{n}\subseteq \mathbb{R}^{n+1}$.

\textbf{Algorithm}: 

$y_{0}=x^{(1)}$

Define $g(y)=\frac{2\arccosh(1+2y)}{\sqrt{y^{2}+y}}$

for $k=0, 1, \cdots, T$:

\hspace{20pt} for $l=1, 2, \cdots, t$: 

\hspace{40pt} $\alpha_{l}=w_{l}\cdot g\left(\frac{|K| \cdot \|x^{(l)}-y_{k}\|^{2}}{(1-|K|\cdot \|x^{(l)}\|^{2})(1-|K|\cdot \|y_{k}\|^{2})}\right)\cdot \frac{1}{1-|K|\cdot\|x^{(l)}\|^{2}}$

\hspace{20pt} $a=\sum\limits_{l=1}^{t}\alpha_{l} \text{  ,  } b=\sum\limits_{l=1}^{t}\alpha_{l}x^{(l)} \text{  ,  } c=\sum\limits_{l=1}^{t}\alpha_{l}\|x^{(l)}\|^{2}$

\hspace{20pt} $y_{k+1}=\left(\frac{(a+c|K|)-\sqrt{(a+c|K|)^{2}-4|K|\cdot\|b\|^{2}}}{2|K|\cdot\|b\|^{2}}\right)b$

return $y_{T}$
\end{algorithm}

\vspace{50pt}

We now prove that this algorithm indeed converges to the Fr\'echet mean for points in the Poincar\'e ball model.
\begin{thm}\label{thm:poincareforwardconvergence}
Let $x^{(1)}, \cdots, x^{(t)}\in \mathbb{D}_{K}^{n}$ be $t$ points in the Poincar\'e ball, $w_1, \dots, w_t \in \R^{+}$ be their weights, and let their weighted Fr\'echet mean be the solution to the following optimization problem
\begin{equation}
    y^{*}=\argmin_{y \in \mathbb{D}_{K}^{n}}\sum_{l=1}^{t}w_{l}\cdot d_{\mathbb{D}_{K}^{n}}(x^{(l)}, y)^{2}=\argmin_{y \in \mathbb{D}_{K}^{n}}\sum_{l=1}^{t}\frac{1}{|K|}w_{l}\cdot \arccosh^2\left(1+2\cdot \frac{|K|\cdot\|x^{(l)}-y\|^{2}}{(1-|K|\cdot\|x^{(l)}\|^{2})(1-|K|\cdot\|y\|^{2})}\right)
\end{equation}
Then Algorithm \ref{alg:poincareforward} gives a sequence of points $\{y_{k}\}$ such that their limit $\lim\limits_{k\rightarrow\infty}y_{k}=y^{*}$ converges to the Fr\'echet mean solution.
\end{thm}
\begin{proof}
(1) \textbf{Apply concavity of $\arccosh$ to give an upper bound on the objective function}: Let us denote $g(y)=\arccosh(1+2y)^{2}$, and also denote $h(y)$ to be the objective function:
\begin{equation*}
    h(y)=\sum\limits_{l=1}^{t}w_{l}\cdot g\left(\frac{|K|\cdot\|x^{(l)}-y\|^{2}}{(1-|K|\cdot\|x^{(l)}\|^{2})(1-|K|\cdot\|y\|^{2})}\right)
\end{equation*}
We know that $g(y)$ is concave, and this means
\begin{equation*}
    g(x)\leq g(y)+g'(y)(x-y) 
\end{equation*}
Applying this to our objective function with respect to some fixed $y_{k}$ at iteration $k$ gives:
\begin{align*}
    &g\left(\frac{\|x^{(l)}-y\|^{2}}{(1-\|x^{(l)}\|^{2})(1-\|y\|^{2})}\right)\leq g\left(\frac{|K|\cdot\|x^{(l)}-y_{k}\|^{2}}{(1-|K|\cdot\|x^{(l)}\|^{2})(1-|K|\cdot\|y_{k}\|^{2})}\right)\nonumber \\
    &\hspace{10pt}+ g'\left(\frac{|K|\cdot\|x^{(l)}-y_{k}\|^{2}}{(1-|K|\cdot\|x^{(l)}\|^{2})(1-|K|\cdot\|y_{k}\|^{2})}\right)\cdot\left(\frac{|K|\cdot\|x^{(l)}-y\|^{2}}{(1-|K|\cdot\|x^{(l)}\|^{2})(1-|K|\cdot\|y\|^{2})}-\frac{|K|\cdot\|x^{(l)}-y_{k}\|^{2}}{(1-|K|\cdot\|x^{(l)}\|^{2})(1-|K|\cdot\|y_{k}\|^{2})}\right)
\end{align*}
Summing this up over all $0\leq l\leq t$ gives us 
\begin{align}\label{poincareupperbound}
    &h(y)\leq h(y_{k})+ \sum\limits_{l=0}^{t}w_{l}\cdot g'\left(\frac{|K|\cdot\|x^{(l)}-y_{k}\|^{2}}{(1-|K|\cdot\|x^{(l)}\|^{2})(1-|K|\cdot\|y_{k}\|^{2})}\right)\nonumber\\ &\hspace{50pt}\cdot \left(\frac{|K|\cdot\|x^{(l)}-y\|^{2}}{(1-|K|\cdot\|x^{(l)}\|^{2})(1-|K|\cdot\|y\|^{2})}-\frac{|K|\cdot\|x^{(l)}-y_{k}\|^{2}}{(1-|K|\cdot\|x^{(l)}\|^{2})(1-|K|\cdot\|y_{k}\|^{2})}\right)
\end{align}

(2) \textbf{Finding solution to the minimization problem of upper bound}: Our goal is to solve the optimization problem of the RHS of equation (\ref{poincareupperbound}). Following similar ideas as in Theorem \ref{thm:lorentzforward}, we can remove terms unrelated to the variable $y$ and simplify this optimization problem:
\begin{equation*}
    y_{k+1}^{*}=\argmin\limits_{y\in \mathbb{D}_{K}^{n}}\sum\limits_{l=1}^{t}w_{l}\cdot g'\left(\frac{\|x^{(l)}-y_{k}\|^{2}}{(1-|K|\cdot\|x^{(l)}\|^{2})(1-|K|\cdot\|y_{k}\|^{2})}\right)\cdot \left(\frac{\|x^{(l)}-y\|^{2}}{(1-|K|\cdot\|x^{(l)}\|^{2})(1-|K|\cdot\|y\|^{2})}\right)
\end{equation*}
Let us denote \begin{equation*}
    \alpha_{l}=w_{l}\cdot g'\left(\frac{\|x^{(l)}-y_{k}\|^{2}}{(1-|K|\cdot\|x^{(l)}\|^{2})(1-|K|\cdot\|y_{k}\|^{2})}\right)\cdot \frac{1}{1-|K|\cdot\|x^{(l)}\|^{2}}
\end{equation*}
Then we can simplify the optimization problem to 
\begin{equation*}
    \argmin\limits_{y\in \mathbb{D}_{K}^{n}}\left(\sum\limits_{l=1}^{t}\alpha_{l}\cdot \frac{\|x^{(l)}-y\|^{2}}{1-|K|\cdot\|y\|^{2}}\right)=\argmin\limits_{y\in \mathbb{D}_{K}^{n}}\left(\sum\limits_{l=1}^{t}\alpha_{l}\cdot \frac{\|x^{(l)}\|^{2}-2(x^{(l)})^{\top}y+\|y\|^{2}}{1-|K|\cdot\|y\|^{2}}\right)
\end{equation*}
Now let \begin{equation*}
    a=\sum\limits_{l=1}^{t}\alpha_{l} \text{  ,  } b=\sum\limits_{l=1}^{t}\alpha_{l}x^{(l)} \text{  ,  } c=\sum\limits_{l=1}^{t}\alpha_{l}\|x^{(l)}\|^{2}
\end{equation*}
Then the above optimization problem further simplifies to 
\begin{equation*}
    \argmin\limits_{y\in \mathbb{D}_{K}^{n}}\left(\frac{a\|y\|^{2}-2b^{T}y+c}{1-|K|\cdot\|y\|^{2}}\right)
\end{equation*}
Say we fix some length $\|y\|$, then to minimize the above expression, we must choose $y=\eta b$ (so that $b^{\top}y$ achieves maximum). Plugging this into the above expression gives
\begin{equation*}
    \argmin\limits_{\eta b\in \mathbb{D}_{K}^{n}}\left(\frac{a\|b\|^{2}\eta^{2}-2\|b\|^{2}\eta+c}{1-|K|\cdot\|b\|^{2}\eta^{2}}\right)
\end{equation*}
Now we'll consider \begin{equation*}
    f(\eta)=\frac{a\|b\|^{2}\eta^{2}-2\|b\|^{2}\eta+c}{1-|K|\cdot\|b\|^{2}\eta^{2}}
\end{equation*}
\begin{align*}
    \Rightarrow f'(\eta) &= \frac{(1-|K|\cdot\|b\|^{2}\eta^{2})(a\|b\|^{2}\eta^{2}-2\|b\|^{2}\eta+c)'-(a\|b\|^{2}\eta^{2}-2\|b\|^{2}\eta+c)(1-|K|\cdot\|b\|^{2}\eta^{2})'}{(1-|K|\cdot\|b\|^{2}\eta^{2})^{2}} \nonumber\\
    &=\frac{(1-|K|\cdot\|b\|^{2}\eta^{2})(2a\|b\|^{2}\eta-2\|b\|^{2})-(a\|b\|^{2}\eta^{2}-2\|b\|^{2}\eta+c)(-2|K|\cdot\|b\|^{2}\eta)}{(1-|K|\cdot\|b\|^{2}\eta^{2})^{2}} \nonumber\\
    &= \frac{-2\|b\|^{2}+2a\|b\|^{2}\eta+2|K|\cdot\|b\|^{4}\eta^{2}-2|K|\cdot a\|b\|^{4}\eta^{3}+2|K|\cdot\|b\|^{2}c\eta-4|K|\cdot\|b\|^{4}\eta^{2}+2|K|\cdot a\|b\|^{4}\eta^{3}}{(1-|K|\cdot\|b\|^{2}\eta^{2})^{2}} \nonumber \\
    &=\frac{-2\|b\|^{2}+2(a+c|K|)\|b\|^{2}\eta-2|K|\cdot\|b\|^{4}\eta^{2}}{(1-|K|\cdot\|b\|^{2}\eta^{2})^{2}}=\frac{-2\|b\|^{2}}{(1-|K|\cdot\|b\|^{2}\eta^{2})^{2}}\left(|K|\cdot\|b\|^{2}\eta^{2}-(a+c|K|)\eta+1\right)
\end{align*}
Thus, to achieve minimum, we will have $f'(\eta)=0$, so
\begin{equation*}
    f'(\eta)=0\Rightarrow |K|\cdot\|b\|^{2}\eta^{2}-(a+c|K|)\eta+1=0\Rightarrow \eta=\frac{(a+c|K|)\pm\sqrt{(a+c|K|)^{2}-4|K|\cdot\|b\|^{2}}}{2|K|\cdot\|b\|^{2}}
\end{equation*}
Moreover, we know that \begin{equation*}
    f'(\eta)<0 \text{ for } \eta<\frac{(a+c|K|)-\sqrt{(a+c|K|)^{2}-4|K|\cdot\|b\|^{2}}}{2|K|\cdot\|b\|^{2}}
\end{equation*} 
and \begin{equation*}
    f'(\eta)>0 \text{ for } \frac{(a+c|K|)-\sqrt{(a+c|K|)^{2}-4|K|\cdot\|b\|^{2}}}{2|K|\cdot\|b\|^{2}}<\eta<\frac{(a+c|K|)+\sqrt{(a+c|K|)^{2}-4|K|\cdot\|b\|^{2}}}{2|K|\cdot\|b\|^{2}}
\end{equation*}
This means the actual $\eta$ that achieves the minimum is
\begin{equation*}
    \eta=\frac{(a+c|K|)-\sqrt{(a+c|K|)^{2}-4|K|\cdot\|b\|^{2}}}{2|K|\cdot\|b\|^{2}}
\end{equation*}
and thus we have
\begin{equation*}
    y_{k+1}^{*}=\left(\frac{(a+c|K|)-\sqrt{(a+c|K|)^{2}-4|K|\cdot\|b\|^{2}}}{2|K|\cdot\|b\|^{2}}\right)b
\end{equation*}
and this is exactly what we computed for $y_{k+1}$ in Algorithm \ref{alg:poincareforward}. 

(3) \textbf{Proof of convergence}: We now show that the sequence $\{y_{k}\}$ converges to the Fr\'echet mean $y^{*}$ as $k\rightarrow\infty$.

We know that if we pick the minimizing $y$ for the RHS of equation (\ref{poincareupperbound}), the RHS is smaller than or equal to the case where we pick $y=y_{k}$ (in which case the RHS becomes $g(y_{k})$). 

Thus, we know that $g(y_{k+1}^{*})\leq g(y_{k})$. Similar to the case in Theorem \ref{thm:lorentzforward}, we know that equality only holds when we're already at the Fr\'echet mean, so if we are not at the Fr\'echet mean, we must have $g(y_{k+1})<g(y_{k})$ strictly decreasing. This means the sequence $\{y_{k}\}$ must converge to the Fr\'echet mean, as this is the only fixed point that we can converge to.

\end{proof}

\section{Explicit Derivations of Backpropagation of Fr\'echet Mean for Hyperbolic Space}

\subsection{Differentiating through all Parameters of the Fr\'echet Mean}

With our constructions from Appendix \ref{appendix:diffManifold}, we derive gradient expressions for the Fr\'echet mean in hyperbolic space. In particular, we wish to differentiate with respect to input points, weights, and curvature. To see that these are all differentiable values, we note that the squared distance function in the Fr\'echet mean objective function admits derivatives for all of these variables. In particular, we see that the Fr\'echet mean for hyperbolic space w.r.t. these input parameters is effectively a smooth function from $(\mathbb{H}^n)^t \times \R^t \times \R \to \mathbb{H}^n$, where the input variables are points, weights, and curvature respectively.

\subsection{Derivation of Formulas for Hyperboloid Model Backpropagation}

In this section, we derive specific gradient computations for the Fr\'echet mean in hyperboloid model; we assume the Fr\'echet mean is already provided from the forward pass. This gradient computation is necessary since, at the time of writing, all machine learning auto-differentiation packages do not support manifold-aware higher-order differentiation.

\vspace{10pt}

Our first goal is to recast our original problem into an equaivalent optimization problem that can be easily differentiated.
\begin{thm}\label{thm:lorentzbackpropequiv}
Let $x^{(1)}, \cdots, x^{(t)}\in \mathbb{H}_{K}^{n}\subseteq \mathbb{R}^{n+1}$ be $t$ points in the hyperboloid space, $w_1, \dots, w_t \in \R^{+}$ be their weights, and let their weighted Fr\'echet mean $y^{*}$ be the solution to the following optimization problem.
\begin{equation}\label{eqn:frechetoptproblorentz}
    y^{*}(x^{(1)}, \cdots, x^{(t)})=\argmin_{y \in \mathbb{H}_{K}^{n}}\sum_{l=1}^{t}\frac{w_{l}}{|K|} \arccosh^{2}(K(x^{(l)})^\top M y)
\end{equation}

Let $\overline{x} = y^{*}(x^{(1)}, \cdots, x^{(t)})$ and $M$ be the identity matrix with top-left coordinate set to be $-1$. We can recast the above optimization problem with a reparametrization map $h: T_{\overline{x}} \mathbb{H}_K^n \to \mathbb{H}_K^n$ by $h(u)=u+\overline{x}\cdot \sqrt{1-Ku^{\top}Mu}$. We obtain the equivalent problem 
\begin{equation}\label{eqn:lorentzbackpropopt}
    u^{*}(x^{(1)}, \cdots, x^{(t)})=\argmin\limits_{u\in \mathbb{R}^{n}, \bar x^{\top}Mu=0}F(x^{(1)}, \cdots, x^{(t)}, u)
\end{equation}
\begin{equation}
    \text{where } F(x^{(1)}, \cdots, x^{(t)}, u)=\sum\limits_{l=1}^{t} \frac{w_{l}}{|K|} g\left((x^{(l)})^{\top}Mu+(x^{(l)})^{\top}M\bar x \cdot \sqrt{1 - Ku^\top M u}\right)
\end{equation}
\begin{equation}
    g(x)=\arccosh^{2}(Ky)
\end{equation}

Note that $u^{*}(x^{(1)}, \cdots, x^{(t)}) = 0$.
\end{thm}
\begin{proof}
(1) \textbf{We first show that $h:\mathbb{H}^{n}\rightarrow\mathbb{H}^{n}$ is a bijection.} 

For completeness, we derive the re-parameterization together with intuition. We let $\overline{x}$ be any point in $\mathbb{H}_{K}^n$ and $u$ a point in the tangent space satisfying $\overline{x}^{\top}Mu=0$. We wish to solve for a constant $c$ so that $h(\overline{x}) = u + \overline{x} \cdot c$ lies on the manifold. Note that this corresponds to re-scaling $\overline{x}$ so that the induced shift by the vector $u$ does not carry $\overline{x}$ off the manifold. Algebraically, we require $h(\overline{x})^\top M h(\overline{x}) = \frac{1}{K}$:
\begin{equation}
(u + \overline{x} \cdot c)^\top M (u + \overline{x} \cdot c) = \frac{1}{K}
\end{equation}
\begin{equation}
u^\top M u + u^\top M \overline{x} \cdot c + c \overline{x}^\top M u + c^2 \cdot \overline{x}^\top M \overline{x} = \frac{1}{K}
\end{equation}
\begin{equation}
u^\top M u + \frac{c^{2}}{K} = \frac{1}{K} \Rightarrow c = \sqrt{1-K u^\top M u}
\end{equation}

Note that since we are on the positive sheet of the hyperboloid, we take the positive root at the final step. Since the map is non-degenerate, i.e. $\sqrt{1-K u^\top M u} \neq 0$ since $u$ is in the tangent space, observe that for any point $h(\overline{x})$ on the hyperboloid we can solve for $\overline{x} = \frac{h(\overline{x}) - u}{\sqrt{1-Ku^\top M u}}$. Hence the map is surjective. Moreover, note that the value of $\overline{x}$ is unique; hence we have injectivity and conclude that $h$ is a bijection.

\vspace{10pt}

(2) The rest of the proof follows from plugging the reparametrization map into equation (\ref{eqn:frechetoptproblorentz}) and simplifying the result. 
\end{proof}
We then present how to differentiate through this equivalent problem. 
\begin{thm}\label{thm:lorentzbackpropmain}
For each $1\leq i\leq n$, consider the function $u_{i}^{*}:\mathbb{R}^{n}\rightarrow\mathbb{R}^{n}$ defined by
\begin{equation}
    u_{i}^{*}(x^{(i)})=u^{*}(\tilde{x}^{(1)}, \cdots, \tilde{x}^{(i-1)},x^{(i)},\tilde{x}^{(i+1)}, \cdots \tilde{x}^{(n)})
\end{equation}
which only varies $x^{(i)}$ and fixing all other input variables in $u^{*}$.
Then the Jacobian matrix of $u_{i}^{*}$ can be computed in the following way: \begin{equation}\label{eqn:hyperboloidbackwardJacobian}
    \widetilde{\nabla}_{x^{(i)}} u^{*}(x^{(i)})=\left(H^{-1}A^{\top}(AH^{-1}A^{\top})^{-1}AH^{-1}-H^{-1}\right)\widetilde{\nabla}_{x^{(i)}}\nabla_{u}F(x^{(1)}, \cdots, x^{(t)}, u)\Big|_{u=0}
\end{equation}
where we plug in $A=\bar{x}^{\top}M$, 
the Hessian evaluated at $u=0$: \begin{align}\label{eqn:hyperboloidbackwardhessian}
    H&=\nabla_{uu}^2 F(x^{(1)}, \ldots, x^{(t)}; u)\Big|_{u=0}\nonumber\\
    &=\sum_{l=1}^{t} \frac{w_{l}}{|K|} \cdot \left( g''\left((x^{(l)})^\top M \bar x \right) \cdot \left( M x^{(l)}\right) \left( M x^{(l)}\right)^\top - K \cdot g  '\left((x^{(l)})^\top M \bar x\right) \cdot \left((x^{(l)})^\top M \bar x\right) \cdot M\right)
\end{align}
and the mixed gradient evaluated at $u=0$:
\begin{equation}\label{eqn:hyperboloidbackwardmixed}
    \widetilde{\nabla}_{x^{(i)}}\nabla_{u}F(x^{(1)}, \cdots, x^{(t)}, u)\Big|_{u=0}=\frac{w_{i}}{|K|}\cdot\left(g''\left((x^{(i)})^\top M \bar x\right)\cdot\left( M x^{(i)}\right)\left(M\bar x\right)^{\top}+g'\left((x^{(i)})^\top M \bar x \right)\cdot M\right)
\end{equation}
\end{thm}
\begin{proof}
(1) \textbf{Application of Gould's theorem}: Recall from Theorem \ref{thm:lorentzbackpropequiv} that our minimization problem has the form
\begin{equation*}
    u^{*}(x^{(1)}, \cdots, x^{(t)})=\argmin\limits_{u\in \mathbb{R}^{n}, \bar x^{\top}Mu=0}F(x^{(1)}, \cdots, x^{(t)}, u)
\end{equation*}
where \begin{equation*}
    F(x^{(1)}, \cdots, x^{(t)}, u)=\sum\limits_{l=1}^{t}\frac{w_{l}}{|K|}\cdot g\left((x^{(l)})^{\top}Mu+(x^{(l)})^{\top}M\bar x \cdot \sqrt{1 - Ku^\top M u}\right)
\end{equation*}
\begin{equation*}
    g(x)=\arccosh(Kx)^{2}
\end{equation*}
We now apply Theorem \ref{thm:diffthroughconstrained} on the constrained optimization problem (\ref{eqn:lorentzbackpropopt}), noting that we can write $x^{\top}Mx=0$ in the form $Au=b$ for $A=x^{\top}M\in \mathbb{R}^{1\times n}$ and $b=0\in \mathbb{R}$. This gives us 
\begin{equation*}
    \widetilde{\nabla} u^{*}(x^{(i)})=\left(H^{-1}A^{\top}(AH^{-1}A^{\top})^{-1}AH^{-1}-H^{-1}\right)\widetilde{\nabla}_{x^{(i)}}\nabla_{u}F(x^{(1)}, \cdots, x^{(t)}, u)
\end{equation*}
where $A=x^{\top}M$ and $H=\nabla_{uu}^{2}F(x^{(1)}, \cdots , x^{(t)}, u)$, so we recover equation (\ref{eqn:hyperboloidbackwardJacobian}). We also note that after we compute the Fr\'echet mean in the forward pass, we can set $u=0$ and this will simplify the expressions of both partial derivatives.

\vspace{10pt}

(2) \textbf{Computing the gradient $\nabla_{u}F$}: We first compute $\nabla_{u}F(x^{(1)}, \cdots, x^{(n)}, u)\in \mathbb{R}^{n}$ using the chain rule.
\begin{equation*}
    \nabla_u F(x^{(1)}, \cdots, x^{(t)}; u) = \sum_{l=1}^{t}\frac{w_{l}}{|K|}\cdot g'\left((x^{(l)})^\top M u + (x^{(l)})^\top M \bar x \cdot \sqrt{1 -Ku^\top M u} \right) \cdot \left( M x^{(l)} + (x^{(l)})^\top M \bar x \cdot \frac{-KM u}{\sqrt{1 -K u^\top M u}} \right) 
\end{equation*}

(3) \textbf{Computing the Hessian $\nabla_{uu}^{2}F$}: We then evaluate the Hessian $\nabla_{uu}^{2}F(x^{(1)}, \cdots, x^{(t)}, u)\in \mathbb{R}^{n\times n}$.
\begin{align*}
&H=\nabla_{uu}^2 F(x^{(1)}, \ldots, x^{(t)}; u) \nonumber\\
&= 
\sum_{l=1}^{t} \frac{w_{l}}{|K|}\cdot g''\left((x^{(l)})^\top M u + (x^{(l)})^\top M \bar x \cdot \sqrt{1 -K u^\top M u} \right) \cdot \left( M x^{(l)} + (x^{(l)})^\top M \bar x \cdot \frac{-KM u}{\sqrt{1 -K u^\top M u}} \right)\nonumber\\
&\hspace{100pt}\cdot\left( M x^{(l)} + (x^{(l)})^\top M \bar x \cdot \frac{-KM u}{\sqrt{1 -K u^\top M u}} \right)^\top\nonumber \\
&\hspace{2em}+
\sum_{l=1}^t \frac{w_{l}}{|K|}\cdot g'\left((x^{(l)})^\top M u + (x^{(l)})^\top M \bar x \cdot \sqrt{1 -K u^\top M u} \right) \cdot (x^{(l)})^\top M \bar x \cdot \left( \frac{-KM}{\sqrt{1 -K u^\top M u}} - \frac{K^{2}M u u^\top M}{(1 -K u^\top M u)^{3/2}} \right)
\end{align*}

(4) \textbf{Computing $\widetilde{\nabla}_{x^{(i)}}\nabla_{u}F$}: We then evaluate $\widetilde{\nabla}_{x^{(i)}}\nabla_{u}F(x^{(1)}, \cdots, x^{(t)}, u)\in \mathbb{R}^{n\times n}$.

We first note that we can rewrite $(x^{(l)})^\top M \bar x \cdot \frac{M u}{\sqrt{1 -K u^\top M u}}=\frac{Mu \bar x^{\top}Mx^{(l)}}{\sqrt{1-Ku^{\top}Mu}}$ where the numerator is a matrix multiplication. We then take the derivative of $\nabla_{u}$ computed above:
\begin{align*}
    &\widetilde{\nabla}_{x^{(i)}}\nabla_{u}F(x^{(1)}, \cdots, x^{(t)}, u) \nonumber\\
    &= \frac{w_{i}}{|K|}\cdot g''\left((x^{(i)})^\top M u + (x^{(i)})^\top M \bar x \cdot \sqrt{1 -K u^\top M u} \right)\cdot\left( M x^{(i)} + (x^{(i)})^\top M \bar x \cdot \frac{-KM u}{\sqrt{1 -K u^\top M u}} \right)\left(Mu+M\bar x\sqrt{1-Ku^{\top}Mu}\right)^{\top} \nonumber\\
    &\hspace{50pt}+\frac{w_{i}}{|K|}\cdot g'\left((x^{(i)})^\top M u + (x^{(i)})^\top M \bar x \cdot \sqrt{1 -K u^\top M u} \right)\cdot \left(M-K\frac{Mu\bar x^{\top}M}{\sqrt{1-Ku^{\top}Mu}}\right)  
\end{align*}

(5) \textbf{Evaluating the above functions at $u=0$}: In our above computations, the parameter $u$ would be set to $0$ after our forward pass finds the Fr\'echet mean. Thus, we will evaluate our results in (3), (4) at $u=0$. 

The Hessian evaluated at $u=0$ gives \begin{align*}
    H&=\nabla_{uu}^2 F(x^{(1)}, \ldots, x^{(t)}; u)\Big|_{u=0}\nonumber\\
    &=\sum_{l=1}^{t}\frac{w_{l}}{|K|}\cdot\left(g''\left((x^{(l)})^\top M \bar x \right) \cdot \left( M x^{(l)}\right) \left( M x^{(l)}\right)^\top - K \cdot g '\left((x^{(l)})^\top M \bar x\right) \cdot \left((x^{(l)})^\top M \bar x\right) \cdot M\right)
\end{align*}
The mixed gradient evaluated at $u=0$ gives
\begin{equation*}
    \widetilde{\nabla}_{x^{(i)}}\nabla_{u}F(x^{(1)}, \cdots, x^{(t)}, u)\Big|_{u=0}=\frac{w_{i}}{|K|}\cdot\left(g''\left((x^{(i)})^\top M \bar x\right)\cdot\left(M x^{(i)}\right)\left(M\bar x\right)^{\top}+g'\left((x^{(i)})^\top M \bar x \right)\cdot M\right)
\end{equation*}
and the above two equations give exactly equations (\ref{eqn:hyperboloidbackwardhessian}) and (\ref{eqn:hyperboloidbackwardmixed}).
\end{proof}

\begin{remark}
    We omit derivations for curvature and weights, as this process follows similarly.
\end{remark}

\subsection{Derivation of Formulas for Poincar\'e Ball Model Backpropagation}

In this section, we derive specific gradient computations for the Poincar\'e ball. This is necessary since many machine learning auto-differentiation do not support higher order differentiation.

\begin{thm}\label{thm:ballbackpropmain}
Let $x^{(1)}, \cdots, x^{(t)}\in \mathbb{D}_{K}^{n}\subseteq \mathbb{R}^{n}$ be $t$ points in the Poincar\'e ball, $w_1, \dots, w_t \in \R$ be their weights, and let their weighted Fr\'echet mean $y^{*}$ be the solution to the following optimization problem.
\begin{equation}\label{eqn:poincarebackwardmain}
    y^{*}(x^{(1)}, \cdots, x^{(t)})=\argmin_{y \in \mathbb{D}^{n}_K}f(\{x\}, y)
\end{equation}
\begin{equation}\label{eqn:poincarebackwardobjective}
    \text{where } f(\{x\}, y)=\frac{1}{|K|}\sum^{t}_{l=1} w_{l}\cdot \arccosh \left(1 - \frac{2K ||x^{(l)}-y||^2_2}{(1+K||x^{(l)}||^2_2)(1+K||y||^2_2)} \right)^2
\end{equation}
Then the gradient of $y^{*}$ with respect to $x^{(i)}$ is given by \begin{equation}\label{eqn:poincarebackwardgradient}
    \widetilde{\nabla}_{x^{(i)}}y^{*}(\{x\})=-\nabla_{yy}^{2}f(\{x\}, y^{*}(\{x\}))^{-1}\widetilde{\nabla}_{x^{(i)}}\nabla_{y}f(\{x\}, y^{*}(\{x\}))
\end{equation}

\vspace{10pt}

(1) Terms of $\widetilde{\nabla}_{x^{(i)}}\nabla_{y}f(\{x\}, y^{*}(\{x\}))$ are given by \begin{equation}\label{eqn:poincarebackwarddxikdy}
    \frac{\partial}{\partial x_{ik}}\frac{\partial}{\partial y_{j}}f(\{x\}, y^{*}(\{x\}))=\frac{1}{|K|}w_{i}\cdot \paren{\left(\frac{2}{(v^{(i)})^{2}-1}-\frac{2(v^{(i)})\arccosh(v^{(i)})}{((v^{(i)})^{2}-1)^{\frac{3}{2}}}\right)M_{ik}T_{ij}+\frac{2\arccosh(v^{(i)})}{\sqrt{(v^{(i)})^2 - 1}}\frac{\partial T_{ij}}{\partial x_{ik}}}
\end{equation}
where \begin{equation}\label{eqn:poincarebackwardTij}
    T_{ij}=\frac{4K}{D^{(i)}}(x_{j}^{(i)}-y_{j})+\frac{4K^{2}}{(D^{(i)})^{2}}y_{j}\|x^{(i)}-y\|_{2}^{2}\cdot (1+K\|x^{(i)}\|_{2}^{2})
\end{equation}
\begin{equation}\label{eqn:poincarebackwardmik}
    M_{ik}=\frac{4K}{D^{(i)}}(y_{k}-x_{k}^{(i)})+\frac{4K^{2}}{(D^{(i)})^{2}}\cdot x_{k}^{(i)}\|x^{(i)}-y\|_{2}^{2}\cdot (1+K\|y\|_{2}^{2})
\end{equation}
\begin{equation}\label{eqn:poincarebackwarddTijdxij}
    \frac{\partial T_{ij}}{\partial x_{ij}}=\frac{4K}{D^{(i)}}-\frac{8K^{2}}{(D^{(i)})^{2}}x_{j}^{(i)}(x_{j}^{(i)}-y_{j})(1+K\|y\|_{2}^{2})+\frac{8K^{2}y_{j}}{(D^{(i)})^{2}}(x_{j}^{(i)}-y_{j})(1+K\|x^{(i)}\|_{2}^{2})-\frac{8K^{3}y_{j}}{(D^{(i)})^{2}}x_{j}^{(i)}\|x^{(i)}-y\|_{2}^{2}
\end{equation}
and for $k\neq j$,
\begin{equation}\label{eqn:poincarebackwarddTijdxik}
    \frac{\partial T_{ij}}{\partial x_{ik}}=-\frac{8K^{2}}{(D^{(i)})^{2}}x_{k}^{(i)}(x_{j}^{(i)}-y_{j})(1+K\|y\|_{2}^{2})+\frac{8K^{2}y_{j}}{(D^{(i)})^{2}}(x_{k}^{(i)}-y_{k})(1+K\|x^{(i)}\|_{2}^{2})-\frac{8K^{3}y_{j}}{(D^{(i)})^{2}}x_{k}^{(i)}\|x^{(i)}-y\|_{2}^{2}
\end{equation}

\vspace{10pt}

(2) Terms of $\nabla_{yy}^{2}f(\{x\}, y^{*}(\{x\}))$ is given by
\begin{equation}\label{eqn:poincarebackwarddyidyj}
    \frac{\partial}{\partial y_{i}}\frac{\partial}{\partial y_{j}}f(\{x\}, y^{*}(\{x\}))=\frac{1}{|K|}\sum\limits_{l=1}^{t} w_l\cdot\paren{\left(\frac{2}{(v^{(l)})^{2}-1}-\frac{2(v^{(l)})\arccosh(v^{(l)})}{((v^{(l)})^{2}-1)^{\frac{3}{2}}}\right)T_{li}T_{lj}+\frac{2\arccosh(v^{(l)})}{\sqrt{(v^{(l)})^2 - 1}}\frac{\partial T_{lj}}{\partial y_{i}}}
\end{equation}
where \begin{equation}\label{eqn:poincarebackwardvlDl}
    v^{(l)} = 1 - \frac{2K ||x^{(l)}-y||^2_2}{(1+K||x^{(l)}||^2_2)(1+K||y||^2_2)} \text{ , } D^{(l)} = (1 + K ||x^{(l)}||^2_2)(1 + K||y||^2_2)
\end{equation}
\begin{equation}\label{eqn:poincarebackwarddTlidyi}
    \frac{\partial T_{li}}{\partial y_{i}}=-\frac{4K}{D^{(l)}}-\frac{16K^{2}}{(D^{(l)})^{2}}y_{i}(x_{i}^{(l)}-y_{i})(1+K\|x^{(l)}\|_{2}^{2})+\frac{4K^{2}(1+K\|x^{(l)}\|_{2}^{2})}{(D^{(l)})^{2}}\|x^{(l)}-y\|_{2}^{2}-\frac{16K^{3}(1+K\|x^{(l)}\|_{2}^{2})^{2}}{(D^{(l)})^{3}}y_{i}^{2}\|x^{(l)}-y\|_{2}^{2}
\end{equation}
and for $i\neq j$
\begin{equation}\label{eqn:poincarebackwarddTljdyi}
    \frac{\partial T_{lj}}{\partial y_{i}}-\frac{8K^{2}(1+K\|x^{(l)}\|_{2}^{2})}{(D^{(l)})^{2}} \left[ y_{j} (x_{i}^{(l)} - y_{i}) + y_{i}(x_{j}^{(l)}-y_{j}) \right]-\frac{16K^{3}(1+K\|x^{(l)}\|_{2}^{2})^{2}}{(D^{(l)})^{3}}y_{i}y_{j}\|x^{(l)}-y\|_{2}^{2}
\end{equation}
\end{thm}
\begin{remark}
    While the above computation is challenging to verify theoretically, to ensure the correctness of our computation, we implemented a gradient check on test cases for which it was simple to generate the correct gradient with a simple $\epsilon$-perturbation, and ensured that this value coincided with the gradient provided by our formulas above.
\end{remark}
\begin{proof}
(1) \textbf{Application of Gould's theorem}: For the minimization problem 
\begin{equation*}
    y^{*}(x^{(1)}, \cdots, x^{(t)})=\argmin_{y \in \mathbb{D}^{n}_K}f(\{x\}, y)
\end{equation*}
\begin{equation*}
    \text{where } f(\{x\}, y)=\frac{1}{|K|}\sum^{t}_{l=1} w_l\cdot \arccosh \left(1 - \frac{2K ||x^{(l)}-y||^2_2}{(1+K||x^{(l)}||^2_2)(1+K||y||^2_2)} \right)^2
\end{equation*}
we can apply Theorem \ref{thm:diffthroughconstrained} to find the the gradient of $y^{*}$ with respect to each $x^{(i)}$: \begin{equation*}
    \widetilde{\nabla}_{x^{(i)}}y^{*}(\{x\})=-\nabla^{2}_{yy}f(\{x\}, y^{*}(\{x\}))^{-1}\widetilde{\nabla}_{x^{(i)}}\nabla_{y}f(\{x\}, y^{*}(\{x\}))
\end{equation*}
and this is exactly in the form of equation (\ref{eqn:poincarebackwardgradient}). Thus, our goal is to compute the Hessian $\nabla^{2}_{yy}f(\{x\}, y^{*}(\{x\}))$ and the mixed gradient $\widetilde{\nabla}_{x^{(i)}}\nabla_{y}f(\{x\}, y^{*}(\{x\}))$.

\vspace{10pt}

(2) \textbf{Computing gradient $\nabla_y f$}: Denote
\begin{equation*}
    v^{(l)} = 1 - \frac{2K ||x^{(l)}-y||^2_2}{(1+K||x^{(l)}||^2_2)(1+K||y||^2_2)} \text{ , } D^{(l)} = (1 + K ||x^{(l)}||^2_2)(1 + K||y||^2_2)
\end{equation*}
Then we have:
\begin{equation}\label{eqn:poincarebackwardpartialy}
    \nabla_y f(\{x\}, y) = \frac{1}{|K|}\left[\sum^{t}_{l=1} w_l \cdot \frac{2\arccosh(v^{(l)})}{\sqrt{(v^{(l)})^2 - 1}} \cdot \frac{\partial v^{(l)}}{\partial y_1}, \sum^{t}_{l=1} w_l \cdot \frac{2\arccosh(v^{(l)})}{\sqrt{(v^{(l)})^2 - 1}} \cdot \frac{\partial v^{(l)}}{\partial y_2}, \cdots \right]
\end{equation}
where we replace $\frac{\partial v^{(i)}}{\partial y_j}$ with the following expression $T_{ij}$:
\begin{align*}
    T_{ij} &= \frac{\partial v^{(i)}}{\partial y_j}=-2K\frac{\partial}{\partial y_{j}}\left(\frac{\|x^{(i)}-y\|_{2}^{2}}{D^{(l)}}\right)=-\frac{2K}{(D^{(l)})^{2}}\left(D^{(i)}\frac{\partial}{\partial y_{j}}\left(\|x^{(i)}-y\|_{2}^{2}\right)-\|x^{(i)}-y\|_{2}^{2}\frac{\partial}{\partial y_{j}}(D^{(i)})\right) \nonumber\\
    &=-\frac{2K}{(D^{(i)})^{2}} \left(D^{(i)}\cdot 2(y_{j}-x_{j}^{(i)})-\|x^{(i)}-y\|_{2}^{2}\cdot (1+K\|x^{(i)}\|_{2}^{2})\cdot 2Ky_{j}\right)\nonumber \\
    &=\frac{4K}{D^{(i)}}(x_{j}^{(i)}-y_{j})+\frac{4K^{2}}{(D^{(i)})^{2}}y_{j}\|x^{(i)}-y\|_{2}^{2}\cdot (1+K\|x^{(i)}\|_{2}^{2})
\end{align*}
which is exactly the formula of $T_{ij}$ in equation (\ref{eqn:poincarebackwardTij}). \textbf{For the following derivations, we will omit the initial $\frac{1}{|K|}$ term for brevity.} 

\vspace{10pt}

(3) \textbf{Now we derive the formula for $\frac{\partial}{\partial x_{ik}} \nabla_y f$.} 

\vspace{10pt}

(3-1) From our previous equation (\ref{eqn:poincarebackwardpartialy}) to evaluate $\nabla_{y}f$, our goal is to evaluate all terms of the form
\begin{align*}
    &\frac{\partial}{\partial x_{ik}}\left(\sum\limits_{l=1}^{t}w_l \frac{2\arccosh(v^{(l)})}{\sqrt{(v^{(l)})^2 - 1}}\frac{\partial v^{(l)}}{\partial y_j}\right) = \sum\limits_{l=1}^{t}w_l \cdot \frac{\partial}{\partial x_{ik}}\left(\frac{2\arccosh(v^{(l)})}{\sqrt{(v^{(l)})^2 - 1}}\right)\cdot \frac{\partial v^{(l)}}{\partial y_j}+\sum\limits_{l=1}^{t}w_l \cdot \frac{2\arccosh(v^{(l)})}{\sqrt{(v^{(l)})^2 - 1}}\cdot \frac{\partial}{\partial x_{ik}}\left(\frac{\partial v^{(l)}}{\partial y_j}\right) \\
    &= \sum\limits_{l=1}^{t} w_l \cdot  \left(\frac{2}{(v^{(l)})^{2}-1}-\frac{2(v^{(l)})\arccosh(v^{(l)})}{((v^{(l)})^{2}-1)^{\frac{3}{2}}}\right)\cdot \frac{\partial v^{(l)}}{\partial x_{ik}}\cdot\frac{\partial v^{(l)}}{\partial y_j}+\sum\limits_{l=1}^{t}w_l \cdot \frac{2\arccosh(v^{(l)})}{\sqrt{(v^{(l)})^2 - 1}}\cdot \frac{\partial}{\partial x_{ik}}\left(\frac{\partial v^{(l)}}{\partial y_j}\right)
\end{align*}
We already have the expression for $T_{lj}=\frac{\partial v^{(l)}}{\partial y_{j}}$ above, so we just need the expression for $\frac{\partial v^{(l)}}{\partial x_{ik}}$ and $\frac{\partial}{\partial x_{ik}}\left(\frac{\partial v^{(l)}}{\partial y_{j}}\right)$.

\vspace{10pt}

(3-2) We first evaluate $\frac{\partial}{\partial x_{ik}}\left(\frac{\partial v^{(l)}}{\partial y_{j}}\right)=\frac{\partial T_{lj}}{\partial x_{ik}}$. 
\begin{align*}
    \frac{\partial T_{lj}}{\partial x_{ik}} &= \frac{\partial}{\partial x_{ik}}\left(\frac{4K(x_{j}^{(l)}-y_{j})}{D^{(l)}}\right)+\frac{\partial}{\partial x_{ik}}\left(\frac{4K^{2}y_{j}\|x^{(l)}-y\|_{2}^{2}\cdot (1+K\|x^{(l)}\|_{2}^{2})}{(D^{(l)})^{2}}\right) \nonumber\\
    &= 4K\frac{\partial}{\partial x_{ik}}\left(\frac{x_{j}^{(l)}-y_{j}}{D^{(l)}}\right)+4K^{2}y_{j}\frac{\partial}{\partial x_{ik}}\left(\frac{\|x^{(l)}-y\|_{2}^{2}\cdot (1+K\|x^{(l)}\|_{2}^{2})}{(D^{(l)})^{2}}\right)
\end{align*}
This expression would become zero if $l\neq i$, so we only need to consider the case $l=i$.

For the first expression, when $k=j$, we have 
\begin{align*}
    4K\frac{\partial}{\partial x_{ij}}\left(\frac{x_{j}^{(i)}-y_{j}}{D^{(i)}}\right) &= \frac{4K}{(D^{(i)})^{2}}\left(D^{(i)}\frac{\partial}{\partial x_{ij}}(x_{j}^{(i)}-y_{j})-(x_{j}^{(i)}-y_{j})\frac{\partial}{\partial x_{ij}}D^{(i)}\right) \nonumber\\
    &=\frac{4K}{(D^{(i)})^{2}}\left(D^{(i)}-(x_{j}^{(i)}-y_{j})\cdot 2Kx_{j}^{(i)}(1+K\|y\|_{2}^{2})\right) \nonumber\\
    &= \frac{4K}{D^{(i)}}-\frac{8K^{2}}{(D^{(i)})^{2}}x_{j}^{(i)}(x_{j}^{(i)}-y_{j})(1+K\|y\|_{2}^{2})
\end{align*}

For the first expression, when $k\neq j$, we have 
\begin{align*}
    4K\frac{\partial}{\partial x_{ik}}\left(\frac{x_{j}^{(i)}-y_{j}}{D^{(i)}}\right) &= \frac{4K}{(D^{(i)})^{2}}\left(-(x_{j}^{(i)}-y_{j})\frac{\partial}{\partial x_{ik}}D^{(i)}\right)=\frac{4K}{(D^{(i)})^{2}}\left(-(x_{j}^{(i)}-y_{j})\cdot 2Kx_{k}^{(i)}(1+K\|y\|_{2}^{2})\right) \nonumber \\
    &= -\frac{8K^{2}}{(D^{(i)})^{2}}x_{k}^{(i)}(x_{j}^{(i)}-y_{j})(1+K\|y\|_{2}^{2})
\end{align*}

For the second expression, we have 
\begin{align*}
    &4K^{2}y_{j}\frac{\partial}{\partial x_{ik}}\left(\frac{\|x^{(i)}-y\|_{2}^{2}\cdot (1+K\|x^{(i)}\|_{2}^{2})}{(D^{(i)})^{2}}\right) \nonumber\\
    &= \frac{4K^{2}y_{j}}{(D^{(i)})^{4}}((D^{(i)})^{2}\|x^{(i)}-y\|_{2}^{2}\frac{\partial}{\partial x_{ik}}(1+K\|x^{(i)}\|_{2}^{2})+(D^{(i)})^{2}(1+K\|x^{(i)}\|_{2}^{2})\frac{\partial}{\partial x_{ik}}(\|x^{(i)}-y\|_{2}^{2}) \nonumber\\
    &\hspace{40pt} -\|x^{(i)}-y\|_{2}^{2}\cdot (1+K\|x^{(i)}\|_{2}^{2})\frac{\partial}{\partial x_{ik}}(D^{(i)})^{2}) \nonumber\\
    &= \frac{4K^{2}y_{j}}{(D^{(i)})^{4}}\cdot (D^{(i)})^{2}\|x^{(i)}-y\|_{2}^{2}\cdot 2Kx_{k}^{(i)}+\frac{4K^{2}y_{j}}{(D^{(i)})^{4}}\cdot (D^{(i)})^{2}(1+K\|x^{(i)}\|_{2}^{2})\cdot 2(x_{k}^{(i)}-y_{k}) \nonumber\\
    &\hspace{40pt} -\frac{4K^{2}y_{j}}{(D^{(i)})^{4}}\cdot \|x^{(i)}-y\|_{2}^{2}\cdot (1+K\|x^{(i)}\|_{2}^{2})\cdot 2D^{(i)}\cdot 2Kx_{k}^{(i)}(1+K\|y\|_{2}^{2}) \nonumber\\
    &=\frac{8K^{3}y_{j}}{(D^{(i)})^{2}}x_{k}^{(i)}\|x^{(i)}-y\|_{2}^{2}+\frac{8K^{2}y_{j}}{(D^{(i)})^{2}}(x_{k}^{(i)}-y_{k})(1+K\|x^{(i)}\|_{2}^{2})-\frac{16K^{3}y_{j}}{(D^{(i)})^{3}}x_{k}^{(i)}\|x^{(i)}-y\|_{2}^{2}(1+K\|x^{(i)}\|_{2}^{2})(1+K\|y\|_{2}^{2}) \nonumber\\
    &=\frac{8K^{3}y_{j}}{(D^{(i)})^{2}}x_{k}^{(i)}\|x^{(i)}-y\|_{2}^{2}+\frac{8K^{2}y_{j}}{(D^{(i)})^{2}}(x_{k}^{(i)}-y_{k})(1+K\|x^{(i)}\|_{2}^{2})-\frac{16K^{3}y_{j}}{(D^{(i)})^{2}}x_{k}^{(i)}\|x^{(i)}-y\|_{2}^{2} \nonumber\\
    &=\frac{8K^{2}y_{j}}{(D^{(i)})^{2}}(x_{k}^{(i)}-y_{k})(1+K\|x^{(i)}\|_{2}^{2})-\frac{8K^{3}y_{j}}{(D^{(i)})^{2}}x_{k}^{(i)}\|x^{(i)}-y\|_{2}^{2}
\end{align*}

Thus, we have for $k=j$, 
\begin{equation*}\label{eqn:poincarebackwardTijxij}
    \frac{\partial T_{ij}}{\partial x_{ij}}=\frac{4K}{D^{(i)}}-\frac{8K^{2}}{(D^{(i)})^{2}}x_{j}^{(i)}(x_{j}^{(i)}-y_{j})(1+K\|y\|_{2}^{2})+\frac{8K^{2}y_{j}}{(D^{(i)})^{2}}(x_{j}^{(i)}-y_{j})(1+K\|x^{(i)}\|_{2}^{2})-\frac{8K^{3}y_{j}}{(D^{(i)})^{2}}x_{j}^{(i)}\|x^{(i)}-y\|_{2}^{2}
\end{equation*}
which matches equation (\ref{eqn:poincarebackwarddTijdxij}). Also, for $k\neq j$, we have
\begin{equation*}\label{eqn:poincarebackwardTijxik}
     \frac{\partial T_{ij}}{\partial x_{ik}}=-\frac{8K^{2}}{(D^{(i)})^{2}}x_{k}^{(i)}(x_{j}^{(i)}-y_{j})(1+K\|y\|_{2}^{2})+\frac{8K^{2}y_{j}}{(D^{(i)})^{2}}(x_{k}^{(i)}-y_{k})(1+K\|x^{(i)}\|_{2}^{2})-\frac{8K^{3}y_{j}}{(D^{(i)})^{2}}x_{k}^{(i)}\|x^{(i)}-y\|_{2}^{2}
\end{equation*}
which matches equation (\ref{eqn:poincarebackwarddTijdxik}).

\vspace{10pt}

(3-3) We now evaluate $\frac{\partial v^{(l)}}{\partial x_{ik}}$. We have
\begin{equation*}
    \frac{\partial v^{(l)}}{\partial x_{ik}}=-2K\frac{\partial}{\partial x_{ik}}\left(\frac{\|x^{(l)}-y\|_{2}^{2}}{D^{(l)}}\right)=-\frac{2K}{(D^{(i)})^{2}}\left(D^{(l)}\frac{\partial}{\partial x_{ik}}\left(\|x^{(l)}-y\|_{2}^{2}\right)-\|x^{(l)}-y\|_{2}^{2}\frac{\partial}{\partial x_{ik}}(D^{(l)})\right)
\end{equation*}
Note that is also zero when $l\neq i$, so we'll assume $l=i$. This will give
\begin{align*}
    M_{ik}=\frac{\partial v^{(i)}}{\partial x_{ik}} &= -2K\frac{\partial}{\partial x_{ik}}\left(\frac{\|x^{(i)}-y\|_{2}^{2}}{D^{(i)}}\right)=-\frac{2K}{(D^{(i)})^{2}}\left(D^{(i)}\frac{\partial}{\partial x_{ik}}\left(\|x^{(i)}-y\|_{2}^{2}\right)-\|x^{(i)}-y\|_{2}^{2}\frac{\partial}{\partial x_{ik}}(D^{(i)})\right)\nonumber \\
    &=-\frac{2K}{(D^{(i)})^{2}}\left(D^{(i)}2(x_{k}^{(i)}-y_{k})-\|x^{(i)}-y\|_{2}^{2}\cdot 2Kx_{k}^{(i)}(1+K\|y\|_{2}^{2})\right) \nonumber\\
    &=\frac{4K}{D^{(i)}}(y_{k}-x_{k}^{(i)})+\frac{4K^{2}}{(D^{(i)})^{2}}\cdot x_{k}^{(i)}\|x^{(i)}-y\|_{2}^{2}\cdot (1+K\|y\|_{2}^{2})
\end{align*}
which matches equation (\ref{eqn:poincarebackwardmik}).

(3-4) Now we put everything together: 
\begin{align*}
    &\frac{\partial}{\partial x_{ik}}\left(\sum\limits_{l=1}^{t}w_{l}\cdot \frac{2\arccosh(v^{(l)})}{\sqrt{(v^{(l)})^2 - 1}}\frac{\partial v^{(l)}}{\partial y_j}\right)\nonumber\\
    &= \sum\limits_{l=1}^{t} w_{l}\cdot \left(\frac{2}{(v^{(l)})^{2}-1}-\frac{2(v^{(l)})\arccosh(v^{(l)})}{((v^{(l)})^{2}-1)^{\frac{3}{2}}}\right)\cdot \frac{\partial v^{(l)}}{\partial x_{ik}}\cdot\frac{\partial v^{(l)}}{\partial y_j}+\sum\limits_{l=1}^{t} w_{l}\cdot \frac{2\arccosh(v^{(l)})}{\sqrt{(v^{(l)})^2 - 1}}\cdot \frac{\partial}{\partial x_{ik}}\left(\frac{\partial v^{(l)}}{\partial y_j}\right) \nonumber\\
    &=w_{l}\cdot \paren{\left(\frac{2}{(v^{(i)})^{2}-1}-\frac{2(v^{(i)})\arccosh(v^{(i)})}{((v^{(i)})^{2}-1)^{\frac{3}{2}}}\right)\cdot \frac{\partial v^{(i)}}{\partial x_{ik}}\cdot\frac{\partial v^{(i)}}{\partial y_j}+\frac{2\arccosh(v^{(i)})}{\sqrt{(v^{(i)})^2 - 1}}\frac{\partial}{\partial x_{ik}}\left(\frac{\partial v^{(i)}}{\partial y_j}\right)} \nonumber\\
    &=w_{l}\cdot \paren{\left(\frac{2}{(v^{(i)})^{2}-1}-\frac{2(v^{(i)})\arccosh(v^{(i)})}{((v^{(i)})^{2}-1)^{\frac{3}{2}}}\right)M_{ik}T_{ij}+\frac{2\arccosh(v^{(i)})}{\sqrt{(v^{(i)})^2 - 1}}\frac{\partial T_{ij}}{\partial x_{ik}}}
\end{align*}
where $T_{ij}$ is given by equation (\ref{eqn:poincarebackwardTij}), $\frac{\partial T_{ij}}{\partial x_{ik}}$ is given by equations (\ref{eqn:poincarebackwarddTijdxij}) and (\ref{eqn:poincarebackwarddTijdxik}), and $M_{ik}$ is given by equation (\ref{eqn:poincarebackwardmik}). 

Now from equation (\ref{eqn:poincarebackwardpartialy}), to obtain the gradient for the original function, we need to multiply the above expression by an additional factor of $\frac{1}{K}$, and then the RHS would match what we want in equation (\ref{eqn:poincarebackwarddxikdy}).

\vspace{10pt}

(4) \textbf{Finally, we derive the formula for Hessian $\nabla_{yy}^{2}f$}: 

(4-1) From our formula for $\nabla_{y}f$ in (\ref{eqn:poincarebackwardpartialy}), our goal is to evaluate all terms of the form
\begin{align*}
    &\frac{\partial}{\partial y_{i}}\left(\sum\limits_{l=1}^{t} w_{l} \cdot  \frac{2\arccosh(v^{(l)})}{\sqrt{(v^{(l)})^2 - 1}}\frac{\partial v^{(l)}}{\partial y_j}\right) = \sum\limits_{l=1}^{t} w_{l}\cdot \frac{\partial}{\partial y_{i}}\left(\frac{2\arccosh(v^{(l)})}{\sqrt{(v^{(l)})^2 - 1}}\right)\cdot \frac{\partial v^{(l)}}{\partial y_j}+\sum\limits_{l=1}^{t} w_{l} \cdot \frac{2\arccosh(v^{(l)})}{\sqrt{(v^{(l)})^2 - 1}}\cdot \frac{\partial}{\partial y_{i}}\left(\frac{\partial v^{(l)}}{\partial y_j}\right) \nonumber\\
    &= \sum\limits_{l=1}^{t} w_{l}\cdot \left(\frac{2}{(v^{(l)})^{2}-1}-\frac{2(v^{(l)})\arccosh(v^{(l)})}{((v^{(l)})^{2}-1)^{\frac{3}{2}}}\right)\cdot \frac{\partial v^{(l)}}{\partial y_{i}}\cdot\frac{\partial v^{(l)}}{\partial y_j}+\sum\limits_{l=1}^{t}w_{l}\cdot \frac{2\arccosh(v^{(l)})}{\sqrt{(v^{(l)})^2 - 1}}\cdot \frac{\partial}{\partial y_{i}}\left(\frac{\partial v^{(l)}}{\partial y_j}\right)
\end{align*}
(4-2) We now consider evaluating $\frac{\partial}{\partial y_{i}}\left(\frac{\partial v^{(l)}}{\partial y_j}\right)=\frac{\partial T_{lj}}{\partial y_{i}}$.
\begin{align*}
    \frac{\partial T_{lj}}{\partial y_{i}} &= \frac{\partial}{\partial y_{i}}\left(\frac{4K(x_{j}^{(l)}-y_{j})}{D^{(l)}}\right)+\frac{\partial}{\partial y_{i}}\left(\frac{4K^{2}y_{j}\|x^{(l)}-y\|_{2}^{2}\cdot (1+K\|x^{(l)}\|_{2}^{2})}{(D^{(l)})^{2}}\right) \nonumber\\
    &= 4K\frac{\partial}{\partial y_{i}}\left(\frac{x_{j}^{(l)}-y_{j}}{D^{(l)}}\right)+4K^{2}(1+K\|x^{(l)}\|_{2}^{2})\frac{\partial}{\partial y_{i}}\left(\frac{y_{j}\|x^{(l)}-y\|_{2}^{2}\cdot }{(D^{(l)})^{2}}\right)
\end{align*}

For the first expression, when $i=j$, we have 
\begin{align*}
    4K\frac{\partial}{\partial y_{i}}\left(\frac{x_{i}^{(l)}-y_{i}}{D^{(l)}}\right) &=\frac{4K}{(D^{(l)})^{2}}\left(D^{(l)}\frac{\partial}{\partial y_{i}}(x_{i}^{(l)}-y_{i})-(x_{i}^{(l)}-y_{i})\frac{\partial}{\partial y_{i}}D^{(l)}\right) \nonumber\\
    &=\frac{4K}{(D^{(l)})^{2}}\left(-D^{(l)}-(x_{i}^{(l)}-y_{i})\cdot 2Ky_{i}(1+K\|x^{(l)}\|_{2}^{2})\right)\nonumber \\
    &= -\frac{4K}{D^{(l)}}-\frac{8K^{2}}{(D^{(l)})^{2}}y_{i}(x_{i}^{(l)}-y_{i})(1+K\|x^{(l)}\|_{2}^{2})
\end{align*}
For the first expression, when $i\neq j$, we have 
\begin{align*}
    4K\frac{\partial}{\partial y_{i}}\left(\frac{x_{j}^{(l)}-y_{j}}{D^{(l)}}\right) &=\frac{4K}{(D^{(l)})^{2}}\left(-(x_{j}^{(l)}-y_{j})\frac{\partial}{\partial y_{i}}D^{(l)}\right)=\frac{4K}{(D^{(l)})^{2}}\left(-(x_{j}^{(l)}-y_{j})\cdot 2Ky_{i}(1+K\|x^{(l)}\|_{2}^{2})\right) \nonumber\\
    &= -\frac{8K^{2}}{(D^{(l)})^{2}}y_{i}(x_{j}^{(l)}-y_{j})(1+K\|x^{(l)}\|_{2}^{2})
\end{align*}
For the second expression, when $i=j$, we have 
\begin{align*}
    &4K^{2}(1+K\|x^{(l)}\|_{2}^{2})\frac{\partial}{\partial y_{i}}\left(\frac{y_{i}\|x^{(l)}-y\|_{2}^{2}}{(D^{(l)})^{2}}\right) \nonumber\\
    &=\frac{4K^{2}(1+K\|x^{(l)}\|_{2}^{2})}{(D^{(l)})^{4}}\left((D^{(l)})^{2}y_{i}\frac{\partial}{\partial y_{i}}(\|x^{(l)}-y\|_{2}^{2})+(D^{(l)})^{2}\|x^{(l)}-y\|_{2}^{2}-y_{i}\|x^{(l)}-y\|_{2}^{2}\frac{\partial}{\partial y_{i}}(D^{(l)})^{2}\right) \nonumber\\
    &=\frac{4K^{2}(1+K\|x^{(l)}\|_{2}^{2})}{(D^{(l)})^{4}}\left((D^{(l)})^{2}y_{i}\cdot 2(y_{i}-x_{i}^{(l)})+(D^{(l)})^{2}\|x^{(l)}-y\|_{2}^{2}-y_{i}\|x^{(l)}-y\|_{2}^{2}\cdot 2D^{(l)}\cdot 2Ky_{i}(1+K\|x^{(l)}\|_{2}^{2})\right) \nonumber\\
    &=\frac{8K^{2}(1+K\|x^{(l)}\|_{2}^{2})}{(D^{(l)})^{2}}y_{i} (y_{i}-x_{i}^{(l)})+\frac{4K^{2}(1+K\|x^{(l)}\|_{2}^{2})}{(D^{(l)})^{2}}\|x^{(l)}-y\|_{2}^{2}-\frac{16K^{3}(1+K\|x^{(l)}\|_{2}^{2})^{2}}{(D^{(l)})^{3}}y_{i}^{2}\|x^{(l)}-y\|_{2}^{2}
\end{align*}
For the second expression, when $i\neq j$, we have \begin{align*}
    &4K^{2}(1+K\|x^{(l)}\|_{2}^{2})\frac{\partial}{\partial y_{i}}\left(\frac{y_{j}\|x^{(l)}-y\|_{2}^{2}\cdot }{(D^{(l)})^{2}}\right)\nonumber \\
    &=\frac{4K^{2}(1+K\|x^{(l)}\|_{2}^{2})}{(D^{(l)})^{4}}\left((D^{(l)})^{2}y_{j}\frac{\partial}{\partial y_{i}}(\|x^{(l)}-y\|_{2}^{2})-y_{j}\|x^{(l)}-y\|_{2}^{2}\frac{\partial}{\partial y_{i}}(D^{(l)})^{2}\right)\nonumber \\
    &=\frac{4K^{2}(1+K\|x^{(l)}\|_{2}^{2})}{(D^{(l)})^{4}}\left((D^{(l)})^{2}y_{j}\cdot 2(y_{i}-x_{i}^{(l)})-y_{j}\|x^{(l)}-y\|_{2}^{2}\cdot 2D^{(l)}\cdot 2Ky_{i}(1+K\|x^{(l)}\|_{2}^{2})\right)\nonumber \\
    &=\frac{8K^{2}(1+K\|x^{(l)}\|_{2}^{2})}{(D^{(l)})^{2}}y_{j}\cdot (y_{i}-x_{i}^{(l)})-\frac{16K^{3}(1+K\|x^{(l)}\|_{2}^{2})^{2}}{(D^{(l)})^{3}}y_{i}y_{j}\|x^{(l)}-y\|_{2}^{2}
\end{align*}
Thus, for $i=j$, we have 
\begin{equation*}
    \frac{\partial T_{li}}{\partial y_{i}}=-\frac{4K}{D^{(l)}}-\frac{16K^{2}}{(D^{(l)})^{2}}y_{i}(x_{i}^{(l)}-y_{i})(1+K\|x^{(l)}\|_{2}^{2})+\frac{4K^{2}(1+K\|x^{(l)}\|_{2}^{2})}{(D^{(l)})^{2}}\|x^{(l)}-y\|_{2}^{2}-\frac{16K^{3}(1+K\|x^{(l)}\|_{2}^{2})^{2}}{(D^{(l)})^{3}}y_{i}^{2}\|x^{(l)}-y\|_{2}^{2}
\end{equation*}
For $i\neq j$, we have
\begin{equation*}
    \frac{\partial T_{lj}}{\partial y_{i}}=-\frac{8K^{2}(1+K\|x^{(l)}\|_{2}^{2})}{(D^{(l)})^{2}} \left[ y_{j} (x_{i}^{(l)} - y_{i}) + y_{i}(x_{j}^{(l)}-y_{j}) \right]-\frac{16K^{3}(1+K\|x^{(l)}\|_{2}^{2})^{2}}{(D^{(l)})^{3}}y_{i}y_{j}\|x^{(l)}-y\|_{2}^{2}
\end{equation*}
which are exactly equations (\ref{eqn:poincarebackwarddTlidyi}) and (\ref{eqn:poincarebackwarddTljdyi}).

\vspace{10pt}

(4-3) Finally, we combine the results together: 
\begin{align*}
    &\frac{\partial}{\partial y_{i}}\left(\sum\limits_{l=1}^{t}w_{l}\cdot\frac{2\arccosh(v^{(l)})}{\sqrt{(v^{(l)})^2 - 1}}\frac{\partial v^{(l)}}{\partial y_j}\right) \nonumber\\
    &= \sum\limits_{l=1}^{t} w_{l}\cdot \paren{\left(\frac{2}{(v^{(l)})^{2}-1}-\frac{2(v^{(l)})\arccosh(v^{(l)})}{((v^{(l)})^{2}-1)^{\frac{3}{2}}}\right)\cdot \frac{\partial v^{(l)}}{\partial y_{i}}\cdot\frac{\partial v^{(l)}}{\partial y_j}+\frac{2\arccosh(v^{(l)})}{\sqrt{(v^{(l)})^2 - 1}}\cdot \frac{\partial}{\partial y_{i}}\left(\frac{\partial v^{(l)}}{\partial y_j}\right)} \nonumber\\
    &= \sum\limits_{l=1}^{t}w_{l}\cdot \paren{\left(\frac{2}{(v^{(l)})^{2}-1}-\frac{2(v^{(l)})\arccosh(v^{(l)})}{((v^{(l)})^{2}-1)^{\frac{3}{2}}}\right)T_{li}T_{lj}+\frac{2\arccosh(v^{(l)})}{\sqrt{(v^{(l)})^2 - 1}}\frac{\partial T_{lj}}{\partial y_{i}}}
\end{align*}
where $T_{li}, T_{lj}$ are given in equation (\ref{eqn:poincarebackwardTij}) and $\frac{\partial T_{lj}}{\partial y_{i}}$ is given in equation (\ref{eqn:poincarebackwarddTlidyi}) and (\ref{eqn:poincarebackwarddTljdyi}). 

Similarly, to obtain the gradient for the original function, we need to multiply the above expresssion by an additional factor of $\frac{1}{K}$, and the resulting RHS would match what we want in equation (\ref{eqn:poincarebackwarddyidyj}).

\end{proof}

\begin{thm}\label{thm:ballbackpropweights}
    Assume the same construction as the above theorem. However, consider $f$ as a function of the weights $w$. Then the gradient of $y^{*}$ with respect to the weights $w$ is given by 
    \begin{equation}\label{eqn:poincarebackwardgradient}
        \widetilde{\nabla}_wy^{*}(\{x\})=-\nabla_{yy}^{2}f(\{x\}, y^{*}(\{x\}))^{-1}\widetilde{\nabla}_{w}\nabla_{y}f(\{x\}, y^{*}(\{x\}))
    \end{equation}
    Note that we assume implicitlyt that the weight is an input.
    
    \vspace{10pt}

    (1) Terms of $\widetilde{\nabla}_w \nabla_{y}f(\{x\}, y^{*}(\{x\}))$ are given by
    
    \begin{equation}\label{eqn:poincarebackwarddwidyjf}
        \parderiv{}{w_i} \parderiv{}{y_j} f(\{x\}, y^*(\{x\})) = \frac{1}{|K|} \cdot \frac{2\arccosh(v^{(i)})}{\sqrt{(v^{(i)})^2 - 1}} \cdot T_{ij}
    \end{equation}
    
    (2) Terms of $\nabla_{yy} f(\{x\}, y^*(\{x\}))$ are given in Equation \ref{eqn:poincarebackwarddyidyj}
\end{thm}

\begin{proof}
    We note that Equation \ref{eqn:poincarebackwardpartialy} gives values for $\parderiv{}{y_j}f(\{x\}, y^*(\{x\}))$. These are given by
    
    \begin{equation*}
        \parderiv{}{y_j}f(\{x\}, y^*(\{x\})) = \frac{1}{K} \sum_{l = 1}^t w_l \cdot \frac{2\arccosh(v^{(l)})}{\sqrt{(v^{(l)})^2 - 1}} \cdot T_{lj}
    \end{equation*}
    
    It follows immediately that we have
    
    \begin{equation*}
        \parderiv{}{w_i} \parderiv{}{y_j} f(\{x\}, y^*(\{x\})) = \frac{1}{K} \cdot \frac{2\arccosh(v^{(i)})}{\sqrt{(v^{(i)})^2 - 1}} \cdot T_{ij}
    \end{equation*}
    
    Note that we have already shown the computation for the Hessian in \ref{thm:ballbackpropmain}
\end{proof}

\begin{thm}\label{thm:ballbackpropcurvature}
    Assume the same construction but instead let $f$ be a function of curvature $K$. The derivative with respect to the curvature is given by
    
    \begin{equation}\label{eqn:poincarebackwardcurvature}
        \widetilde{\nabla}_K y^*(\brac{x}) = -\nabla_{yy}^2f(\brac{x}, y^*(\brac{x}))^{-1} \widetilde{\nabla}_K \nabla_y f(\brac{x}, y^*(\brac{x}))
    \end{equation}
    
    where we again implicitly assume curvature is an input.
    
    (1) Terms of $\widetilde{\nabla}_K \nabla_y f(\brac{x}, y^*(\brac{x}))$ are given by
    
    \begin{equation}\label{eqn:poincarebackwarddKdyjf}
        \begin{gathered}
        \frac{d}{dK} \parderiv{}{y_j} f(\{x\}, y^*(\{x\})) = \sum_{i = 1}^t \frac{w_i \arccosh(v^{(i)}) T_{ij}}{K^2 \sqrt{(v^{(i)})^2 - 1}}\\ - \frac{w_i T_{ij}}{K} \paren{\frac{2}{((v^{(i)})^2 + 1)^{3/2}} - \frac{4v^{(i)}\arccosh(v^{(i)})}{((v^{(i)})^2 - 1)^2}} \frac{dv^{(i)}}{dK}
        - \frac{w_i}{K} \frac{2\arccosh((v^{(i)})^2)}{\sqrt{(v^{(i)})^2 - 1}} \frac{dT_{ij}}{dK}
        \end{gathered}
    \end{equation}
    
    where
    
    \begin{equation}\label{eqn:poincarebackwarddvidK}
        \frac{dv^{(i)}}{dK} = \frac{\snorm{2(v^{(i)})^2 - y_2^2}(\snorm{ x^{(i)}}_2^2 \snorm{y}_2^2 K^2 - 1)}{(1 + K\snorm{x^{(i)}}_2^2)^2(1 + K\snorm{y}_2^2)^2}
    \end{equation}
    
    and
    
    \begin{equation}\label{eqn:poincarebackwarddTijdK}
        \begin{gathered}
            \frac{dT_{ij}}{dK} = \frac{4(x_j^{(i)} - y_j)(1 - \snorm{x^{(i)}}_2^2 \snorm{y}_2^2 K^2)}{(1 + K\snorm{x^{(i)}}_2^2)^2(1 + K\snorm{y}_2^2)^2} + \frac{8Ky_i \snorm{x^{(i)} - y}_2^2}{(1 + K\norm{y}_2^2)^3}
        \end{gathered}
    \end{equation}
    
    (2) Terms of $\nabla_{yy}^2f(\brac{x}, y^*(\brac{x}))^{-1}$ are given in Equation \ref{eqn:poincarebackwarddyidyj}
\end{thm}

\begin{proof}
    We note that Equation \ref{eqn:poincarebackwardpartialy} gives values for $\parderiv{}{y_j}f(\{x\}, y^*(\{x\}))$. These are given by
    
    \begin{equation*}
        \parderiv{}{y_j}f(\{x\}, y^*(\{x\})) = \frac{1}{|K|} \sum_{i = 1}^t w_i \cdot \frac{2\arccosh(v^{(i)})}{\sqrt{(v^{(i)})^2 - 1}} \cdot T_{ij}
    \end{equation*}
    
    note that $|K| = -K$ since $K$ is always negative. By applying product rule, we see that Equation \ref{eqn:poincarebackwarddTijdK} is valid. To check the explicit derivatives, as a reminder we write the equations for $v^{(i)}$ and $T_{ij}$.
    
    \begin{equation*}
        v^{(i)} = 1 - \frac{2K \snorm{x^{(i)} - y}_2^2}{(1 + K\snorm{x^{(i)}}_2^2) (1 + K\snorm{y}_2^2)}
    \end{equation*}
    \begin{equation*}
        T_{ij} = \frac{4K(x_j^{(i)} - y_j)}{(1 + K\snorm{x^{(i)}}_2^2)(1 + K\snorm{y}_2^2)} + \frac{4K^2y_i\snorm{x^{(i)} - y}_2^2}{(1 + K\snorm{y}_2^2)^2}
    \end{equation*}

    We can calculate the derivative of of $v^{(i)}$ by applying quotient rule. In particular
    
    \begin{align*}
        \frac{dv^{(i)}}{dK} &= \frac{d}{dK} \paren{1 - \frac{2K \snorm{x^{(i)} - y}_2^2}{(1 + K\snorm{x^{(i)}}_2^2) (1 + K\snorm{y}_2^2)}}\\
        &= -\frac{d}{dK} \paren{\frac{2K \snorm{x^{(i)} - y}_2^2}{(1 + K\snorm{x^{(i)}}_2^2) (1 + K\snorm{y}_2^2)}}\\
        &= - \frac{\paren{(1 + K\snorm{x^{(i)}}_2^2) (1 + K\snorm{y}_2^2)} 2\snorm{v^{(i)} - y}_2^2 - 2K\snorm{v^{(i)} - y}_2^2\paren{\snorm{x^{(i)}} + \snorm{y}_2^2 + 2K \snorm{x^{(i)}}_2^2 \snorm{y}_2^2}}{(1 + K\snorm{x^{(i)}}_2^2)^2 (1 + K\snorm{y}_2^2)^2}\\
        &= \frac{\snorm{2(v^{(i)})^2 - y_2^2}(\snorm{ x^{(i)}}_2^2 \snorm{y}_2^2 K^2 - 1)}{(1 + K\snorm{x^{(i)}}_2^2)^2(1 + K\snorm{y}_2^2)^2}
    \end{align*}
    
    Similarly, we can calculate the derivative of $T_{ij}$ by applying similar rules. In particular we see that
    
    \begin{align*}
        &\frac{d}{dK}\paren{\frac{4K(x_j^{(i)} - y_j)}{(1 + K\snorm{x^{(i)}}_2^2)(1 + K\snorm{y}_2^2)}}\\
        &= \frac{4(x_j^{(i)} - y_j)(1 + K\snorm{x^{(i)}}_2^2)(1 + K\snorm{y}_2^2) - 4K(x_j^{(i)} - y_j)\paren{\snorm{x^{(i)}} + \snorm{y}_2^2 + 2K \snorm{x^{(i)}}_2^2 \snorm{y}_2^2})}{(1 + K\snorm{x^{(i)}}_2^2)^2(1 + K\snorm{y}_2^2)^2}\\
        &= \frac{4(x_j^{(i)} - y_j)(1 - \snorm{x^{(i)}}_2^2 \snorm{y}_2^2 K^2)}{(1 + K\snorm{x^{(i)}}_2^2)^2(1 + K\snorm{y}_2^2)^2}
    \end{align*}
    
    and
    
    \begin{align*}
        \frac{d}{dK} \paren{\frac{4K^2y_i\snorm{x^{(i)} - y}_2^2}{(1 + K\snorm{y}_2^2)^2}} &= \frac{8Ky_i \snorm{x^{(i)} - y}_2^2 (1 + K\snorm{y}_2^2) - 4K^2y_i\snorm{x^{(i)} - y}_2^2 \snorm{y}_2^2}{(1 + K\snorm{y}_2^2)^3}\\
        &= \frac{8Ky_i \snorm{x^{(i)} - y}_2^2}{(1 + K\norm{y}_2^2)^3}
    \end{align*}
    
    Putting it all together, we have that
    
    \begin{align*}
        \frac{dT_{ij}}{dK} &= \frac{d}{dK}\paren{\frac{4K(x_j^{(i)} - y_j)}{(1 + K\snorm{x^{(i)}}_2^2)(1 + K\snorm{y}_2^2)} + \frac{4K^2y_i\snorm{x^{(i)} - y}_2^2}{(1 + K\snorm{y}_2^2)^2}}\\
        &= \frac{4(x_j^{(i)} - y_j)(1 - \snorm{x^{(i)}}_2^2 \snorm{y}_2^2 K^2)}{(1 + K\snorm{x^{(i)}}_2^2)^2(1 + K\snorm{y}_2^2)^2} + \frac{8Ky_i \snorm{x^{(i)} - y}_2^2}{(1 + K\norm{y}_2^2)^3}
    \end{align*}
    
    which gives us the derivation for $\widetilde{\nabla}_K \nabla_y f(\brac{x}, y^*(\brac{x}))$. We have already shown derivations for the Hessian, so this gives us the desired formulation of the derivation $\widetilde{\nabla}_K y^*(\brac{x})$ with respect to curvature.
\end{proof}

\section{Riemannian Batch Normalization as a Generalization of Euclidean Bach Normalization}
\label{appendix:rbn}

In this section, we present the proof that our Riemannian batch normalization algorithm formulated in Algorithm \ref{alg:RiemannianBN} is a natural generalization of Euclidean batch normalization. In doing so, we also derive an explicit generalization which allows for vector-valued variance. To introduce this notion, we first define product manifolds. As opposed to previous methods, this allows us to introduce variance to complete the so-called ``shift-and-scale" algorithm.

\textbf{Product manifold}: We can define a product manifold $\M = \prod_{i = 1}^n \M_i$. If $\M_i$ is a $k_i$-dimensional manifold then $\dim \M = \sum_{i = 1}^n k_i$ and $T_m \M = \prod_{i = 1}^n T_{m_i} \M_i$ where $m = (m_i)_{i = 1}^n$. If these manifolds are Riemannian with metrics $g_i$, this inherits a metric $g = \begin{pmatrix} g_1 & \dots & \dots &\dots \\ \vdots & g_2 & \dots & \dots \\ \vdots & \vdots & \ddots & \\ \vdots & \vdots & & g_n\end{pmatrix}$ where the values are $0$ for elements not in the diagonal matrices.

\textbf{Fr\'echet Mean on Product Manifold:} First note that the definition of Fr\'echet mean and variance given in equations (\ref{frechetMean}), (\ref{frechetVar}) can be extended to the case of product manifolds $\M = \prod_{i = 1}^n \M_i$, where we will have $\mu_{fr} \in \M$, $\sigma_{fr}^2 \in \R^n$. For such product manifolds, we can define the Fr\'echet mean and variance element-wise as
\begin{equation}\label{frechetMeanProduct}
    \mu_{fr:prod}=(\mu_{fr:\M_i}))_{i \in [n]}
\end{equation}
\begin{equation}\label{frechetVarProduct}
    \sigma_{fr:prod}^{2}=(\sigma_{fr:\M_i}^2)_{i \in [n]}
\end{equation}

\begin{prop}\label{thm:appendixFrechetEquivProduct}
    The product Fr\'echet mean formula given in equation (\ref{frechetMeanProduct}) is equivalent with the Fr\'echet mean given in equation (\ref{frechetMean}) when applied on the product manifold.
\end{prop}

\begin{proof}
    Note that on the product manifold $d_\M^2(x, y) = \sum\limits_{i = 1}^m d_{\M_i}^2(x, y)$.
    The proposition follows immediately since we are optimizing the disjoint objectives $\sigma^2_{fr}(\M_i)$, and the values in our product are disjoint and unrelated.
\end{proof}

\begin{cor}\label{cor:appendixFrechetEquivEuclidean}
    Consider $\R^n$ as a product manifold $\R \times \R \dots \times \R$ $n$ times, then the values in equation (\ref{frechetMeanProduct}) and equation (\ref{frechetVarProduct}) correspond to the vector-valued Euclidean variance and mean.
\end{cor}

\begin{proof}
    This follows almost directly from \ref{prop:appendixFrechetEquiv}. In particular, we recall that the Euclidean vector-valued mean and variance of input points $x^{(1)}, \dots, x^{(t)}$ are defined by
    
    \begin{equation}\label{eqn:eucMeanVector}
        \mu_{euc} = \frac{1}{t} \sum_{i = 1}^t x^{(i)}
    \end{equation}
    \begin{equation}\label{eqn:eucVarianceVector}
        \sigma_{euc}^2 = \frac{1}{t} \sum_{i = 1}^t (x^{(i)})^2 - \paren{\frac{1}{t} \sum_{i = 1}^t x^{(i)}}^2
    \end{equation}
    
    Define $\mu_i$ to be the standard Euclidean mean and $\sigma_i^2$ to be the standard Euclidean variance for points $\{x_i^{(1)}, \dots, x_i^{(t)}\}$. We see that the mean coincides as
    
    \begin{align*}
        \mu_{euc} &= \frac{1}{t} \sum_{i = 1}^t x^{(i)}\\
        &=\paren{\frac{1}{t}\sum_{i = 1}^t x_1^{(i)}, \dots, \frac{1}{t}\sum_{i = 1}^t x_n^{(i)}}\\
        &= \paren{\mu_1, \dots, \mu_n}
    \end{align*}
    
    Similarly, we have that, for the variance
    
    \begin{align*}
        \sigma_{euc}^2 &= \frac{1}{t} \sum_{i = 1}^t (x^{(i)})^2 - \paren{\frac{1}{t} \sum_{i = 1}^t x^{(i)}}^2\\
        &= \paren{\frac{1}{t} \sum_{i = 1}^t (x_1^{(i)})^2 - \paren{\frac{1}{t} \sum_{i = 1}^t x_1^{(i)}}^2, \dots, \frac{1}{t} \sum_{i = 1}^t (x_n^{(i)})^2 - \paren{\frac{1}{t} \sum_{i = 1}^t x_n^{(i)}}^2}\\
        &= (\sigma_1^2, \dots, \sigma_n^2)
    \end{align*}
    
    which is what was desired.
\end{proof}

\begin{thm}\label{appendixScalingThm}
    The Riemannian batch normalization algorithm presented in Algorithm \ref{alg:RiemannianBN} is equivalent to the Euclidean batch normalization algorithm when $\M = \R^n$ for all $n$ during training time.
\end{thm}
\begin{proof}

We know from Corollary \ref{cor:appendixFrechetEquivEuclidean} that the Fr\'echet mean and variance computed in the first two steps correspond to Euclidean mean and variance.

The core of this proof lies in the fact that the exponential and logarithmic maps on $\R^n$ are trivial. This is because $T_x \R^n = \R^n$, and as a manifold $\R^n$ exhibits the same distance function as its tangent space. Consider the Euclidean batch normalization formula given below:
\begin{equation}\label{batchNormEquations}
    x_i'= \gamma\frac{x_i - \mu_\mathcal{B}}{\sigma_\mathcal{B}} + \beta
\end{equation}
where we remove the $\epsilon$ term in \cite{Ioffe2015BatchNA} for clarity. We know that in Euclidean space, $\exp_x \vect{v} = x + \vect{v}$ and $\log_x y = y - x$ and $PT_{x \to x'}(\vect{v}) = \vect{v}$. Then we can rewrite equation (\ref{batchNormEquations}) as:
\begin{equation}\label{batchNormEquivalence}
    x_i' = \exp_\beta\left(\gamma \frac{\log_{\mu_\mathcal{B}} x_i}{\sigma_\mathcal{B}}\right)= \exp_\beta\left(\frac{\gamma}{\sigma_\mathcal{B}} PT_{\mu_{\mathcal{B}} \to \beta} (\log_{\mu_\mathcal{B}} x_i)\right)
\end{equation}
which corresponds to the ``shifting-and-scaling" part of Algorithm \ref{alg:RiemannianBN}, which is the only action during train time.

\end{proof}

\begin{remark}
    We note that in our testing procedure (when acquiring the mean set statistics), we utilize the notion of momentum to define the test statistics instead of an iterative averaging method. To understand why this is necessary, we note that this is due to the lack of a non-closed formulation for the Fr\'echet mean. In particular if we have an iterative averaging procedure defined by
    
    \begin{equation*}
        \mu(\{x^{(1)}, \dots, x^{(n + 1)}\}) = f(\mu(\{x^{(1)}, \dots, x^{(n)}\}, x^{(n + 1)})
    \end{equation*}
    
    for some function $f$, then with this iterative procedure we can define the Fr\'echet mean by
    
    \begin{equation*}
        \mu(\{x^{(1)}, \dots, x^{(n)}\}) = f(\dots f(f(x^{(1)}, x^{(2)}), x^{(3)}), \dots, f^{(n)})
    \end{equation*}
    
    If one can solve this, then we would have a closed form for the Fr\'echet mean and a way to iterativly update, allowing us to fully generalize the Euclidean Batch Normalization algorithm (with both training and testing algorithms) to general Riemannian manifolds.
\end{remark}

\begin{remark}
    In this construction, we utilize product manifolds as a substitute for our Riemannian manifold $\R^n$. However, another natural formulation would be to normalize in the tangent space $T_m\M$. To see why this works, at least in the Euclidean space, note that $T_x\R^n = \R^n$ and so normalizing variance in the tangent space would be equivalent. However, this is due to the fact that $T_x\R^n$ has a coordinate basis which is invariant under parallel transport.
    
    However, on general Riemannian manifolds, this is not necessarily the case. For example, on the ball $\mathbb{D}_{-1}^n$, we normally define the basis as vectors of $\R^n$. But, we note that when we transport these vectors (say from two points on a non-diameter geodesic) then the bases vectors change. This means that it becomes difficult to establish a standard coordinate basis for variance normalization. This is a property called \textbf{holonomy} which informally examines how much information is lost under parallel transport because of curvature.
\end{remark}

\section{Additional Experimental Results and Details}




\subsection{Training and Architectural Details}

All experiments on link-prediction and batch normalization were run using modifications of the code provided by \citet{Chami2019HyperbolicGC}. Every experiment was run on a NVIDIA RTX 2080Ti GPU until improvement, as measured by the ROC AUC metric, had not occurred over the course of 100 epochs (this is also the default stopping criterion for the experiments in \citet{Chami2019HyperbolicGC}).

For clarity, we describe the original proposed hyperbolic graph convolution proposed by HGCN \cite{Chami2019HyperbolicGC} and our modification. The original hyperbolic convolution involves a hyperbolic linear transformation (down from the number of dataset features to 128), followed by a hyperbolic aggregation step, followed by a hyperbolic activation. Our modification uses instead a hyperbolic linear transformation (down from the number of input features to 128), followed by the differentiable Fr\'echet mean operation we developed, followed by a hyperbolic activation. 

For the batch normalization experiments, the baseline encoder was a hyperbolic neural network \cite{Ganea2018HyperbolicNN} with only two layers (the first going from the input feature dimension to  $128$, and the second going from $128$ to $128$). Our modification instead used two hyperbolic linear layers of the same dimensions with hyperbolic batch normalization layers following each layer.

\subsection{Pseudo-means Warp Geometry}
In our main paper, we frequently state that pseudo-Fr\'echet means warp geometry and are hence less desirable than the true Fr\'echet mean. Here we describe two common pseudo-means used in the literature and illustrate evidence for this fact.

One of these means is the tangent space aggregation proposed by \cite{Chami2019HyperbolicGC}. We saw from the main paper that this mean yielded worse performance on link prediction tasks when compared to the Fr\'echet mean. We can view this as evidence of the fact that taking the mean in the tangent space does not yield a representation that lets the model aggregate features optimally. Another general closed form alternatives to the Fr\'echet mean that has gained some traction is the Einstein midpoint; this has been used, due to its convenient closed form, by the Hyperbolic Attention Networks paper \cite{Glehre2018HyperbolicAN}. Although it has a closed form, this mean is not guaranteed to solve the naturally desirable Fr\'echet variance minimization.

Here we experimentally illustrate that both of these means do not in general attain the minimum Fr\'echet variance, and present the relative percentage by which both methods perform are worse (than the optimal variance obtained by the true mean). We conduct mean tests on ten randomly generated 16-dimensional on-manifold points for 100 trials. The points are generated in the Klein model of hyperbolic space, the Einstein midpoint is taken, then the points are translated into the hyperboloid model where we compute the mean using tangent space aggregation and our method. This is done for fair comparison, so that all methods deal with the same points. The results are shown in Table \ref{tab:warpedgeometry}. Notice that the tangent space aggregation method is very off, but this is somewhat expected since it trades global geometry for local geometry. The Einstein midpoint performs much better in this case, but can still be quite off as is demonstrated by its $>5\%$ relative error average and high $13\%$ relative error standard deviation.

\begin{table}[!htb]
\centering
\caption{\label{tab:warpedgeometry} Fr\'echet variance of various pseudo-means; we run our approach to become accurate within $\epsilon=10^{-12}$ of the true Fr\'echet mean, and then report how much worse the other means are in terms of attaining the minimum variance. The average over $100$ trials and corresponding standard deviation is reported. The primary baselines are the tangent space aggregation \cite{Chami2019HyperbolicGC} and the Einstein midpoint \cite{Glehre2018HyperbolicAN} methods.}. 
\vspace{0.2em}
\begin{tabular}{cccc}
    \toprule
    Mean & Relative Error (Fr\'echet Variance) & Distance in Norm from True Mean\\
    \midrule
    Tangent space aggregation \cite{Chami2019HyperbolicGC} &  $246.7\%$\tiny$\pm 4.5\%$ & $28207146.1$\tiny$\pm 8300242.0$ \\
    Einstein midpoint \cite{Glehre2018HyperbolicAN} & $6.8\%$\tiny$\pm 13.5\%$ & $30.3$\tiny$\pm 65.3$ \\
    Ours & $0.0\%$ & $10^{-12}$ \\
    \bottomrule
\end{tabular}
\end{table}

\subsection{Fr\'echet Mean Forward Computation: Reliance on Dimension and Number of Points}
\label{sec:frechetmeanforward_extra}

In this section we provide additional results that highlight our forward pass algorithm's performance relative to baselines. Specifically, we investigate how its performance changes when we increase the dimension of the underlying space and the number of points in the mean computation. For space dimensions of $\{10,20,50\}$ and for numbers of points equal to $\{10,100,1000\}$, we evaluate $10$ trials and report iteration counts and runtimes. Results for the Poincar\'e disk model are given in Table \ref{tab:frechetmeanextrapoin} (Hyperboloid model results are similar, but are not given here for brevity). Experiments were run on an Intel Skylake Core i7-6700HQ 2.6 GHz Quad core CPU. The grid search for the RGD baseline is performed in a manner similar to what was performed in the main paper\footnote{The grid search starts from $lr=0.2$ and goes to $lr=0.4$ in increments of $0.01$ for the Poincar\'e ball model (except dimension $50$, which starts from $lr=0.1$ and goes to $lr=0.3$).}. Note that some general trends can be observed. For instance, it appears that with more points, the convergence improves across all methods (i.e. fewer iterations are needed). Additionally, our method is fairly insensitive to both dimensions and number of points, ensuring fast computation of the Fr\'echet mean in a variety of settings, and does not break for higher dimensions like Karcher flow.

 
\begin{table}[!htb]
\centering
\caption{\label{tab:frechetmeanextrapoin} Empirical computation of the Fr\'echet mean for the Poincar\'e disk model; the average number of iterations required to become accurate within $\epsilon=10^{-12}$ of the true Fr\'echet mean are reported, together with runtime (in the format iterations $\vert$ runtime). $10$ trials are conducted, and standard deviation is reported. The primary baselines are the RGD \cite{Udriste1997} and Karcher Flow \cite{karcher1977riemannian} algorithms. An asterisk indicates that the algorithm did not converge in $50000$ iterations.}.
\vspace{2pt}
\small
\begin{tabular}{cc|cc|cc|cc}
    \toprule
    & & \multicolumn{2}{c|}{10 dim} & \multicolumn{2}{c|}{20 dim} & \multicolumn{2}{c}{50 dim} \\
    & & Iterations & Time (ms) & Iterations & Time (ms) & Iterations & Time (ms) \\
    \midrule
    \parbox[t]{2mm}{\multirow{4}{*}{\rotatebox[origin=c]{90}{\small{10 pts}}}} & RGD ($lr=0.01$) & $898.5$\tiny$\pm 13.9$ & \small$1365.4$\tiny$\pm 40.9$ & $681.1$\tiny$\pm 14.8$\small & $1090.9$\tiny$\pm 79.3$ & $425.5$\tiny$\pm 9.0$ & \small$557.5$\tiny$\pm 31.0$ \\
    & Karcher Flow & $25.4$\tiny$\pm 2.2$ & \small$30.4$\tiny$\pm 1.9$ & $344.3$\tiny$\pm 347.3$ & \small$340.6$\tiny$\pm 307.2$ & * & * \\
    & Ours & $\mathbf{11.9}$\tiny$\pm 0.7$ & \small$\mathbf{6.5}$\tiny$\pm 0.9$ & $\mathbf{13.4}$\tiny$\pm 0.7$ & \small$\mathbf{17.4}$\tiny$\pm 14.3$ & $\mathbf{15.3}$\tiny$\pm 0.5$ & \small$\mathbf{7.5}$\tiny$\pm 0.6$ \\
    \cline{2-8}
    \rule{0pt}{3ex} & RGD + Grid Search on $lr$ & $9.6$\tiny$\pm 0.5$ & \small$7151.7$\tiny$\pm 313.2$ & $10.6$\tiny$\pm 0.5$ & \small$7392.0$\tiny$\pm 674.5$ & $11.7$\tiny$\pm 0.5$ & \small$12523.8$\tiny$\pm 4330.9$\\
    \midrule
    \parbox[t]{2mm}{\multirow{4}{*}{\rotatebox[origin=c]{90}{\small{100 pts}}}} & RGD ($lr=0.01$) & $862.4$\tiny$\pm 9.9$ & \small$1349.1$\tiny$\pm 44.6$ & $635.1$\tiny$\pm 6.3$ & \small$1021.5$\tiny$\pm 57.1$ & $382.1$\tiny$\pm 1.3$ & \small$547.0$\tiny$\pm 10.2$ \\
    & Karcher Flow & $25.7$\tiny$\pm 0.9$ & \small$28.1$\tiny$\pm 2.3$ & $248.3$\tiny$\pm 41.8$ & \small$262.1$\tiny$\pm 42.4$ & * & * \\
    & Ours & $\mathbf{9.6}$\tiny$\pm 0.5$ & \small$\mathbf{6.1}$\tiny$\pm 0.8$ & $\mathbf{10.0}$\tiny$\pm 0.0$ & \small$\mathbf{17.2}$\tiny$\pm 6.3$ & $\mathbf{9.9}$\tiny$\pm 0.3$ & \small$\mathbf{6.1}$\tiny$\pm 1.0$ \\
    \cline{2-8}
    \rule{0pt}{3ex} & RGD + Grid Search on $lr$ & $7.8$\tiny$\pm 0.4$ & \small$6802.9$\tiny$\pm 298.1$ & $8.8$\tiny$\pm 0.4$ & \small$7435.0$\tiny$\pm 481.7$ & $9.4$\tiny$\pm 0.5$ & \small$33899.7$\tiny$\pm 17189.6$ \\
    \midrule
    \parbox[t]{2mm}{\multirow{4}{*}{\rotatebox[origin=c]{90}{\small{1000 pts}}}} & RGD ($lr=0.01$) & $857.6$\tiny$\pm 7.5$ & \small$1991.5$\tiny$\pm 257.9$ & $631.4$\tiny$\pm 4.0$ & \small$1649.3$\tiny$\pm 207.6$ & $378.2$\tiny$\pm 2.0$ & \small$1309.8$\tiny$\pm 624.8$ \\
    & Karcher Flow & $25.9$\tiny$\pm 0.3$ & \small$43.1$\tiny$\pm 4.4$ & $228.7$\tiny$\pm 7.4$ & \small$397.7$\tiny$\pm 51.2$ & * & * \\
    & Ours & $\mathbf{9.0}$\tiny$\pm 0.0$ & \small$\mathbf{7.4}$\tiny$\pm 1.9$ & $\mathbf{9.0}$\tiny$\pm 0.0$ & \small$\mathbf{13.7}$\tiny$\pm 11.4$ & $\mathbf{8.0}$\tiny$\pm 0.0$ & \small$\mathbf{8.2}$\tiny$\pm 1.2$ \\
    \cline{2-8}
    \rule{0pt}{3ex} & RGD + Grid Search on $lr$ & $7.0$\tiny$\pm 0.0$ & \small$9634.3$\tiny$\pm 864.5$ & $7.9$\tiny$\pm 0.3$ & \small$11818.3$\tiny$\pm 903.1$ & $8.8$\tiny$\pm 0.4$ & \small$53194.4$\tiny$\pm 7788.3$ \\
    \bottomrule
\end{tabular}
\end{table}

\end{document}